\documentclass[aos,preprint]{imsart}

\usepackage{amssymb}
\usepackage{graphicx}
\usepackage{caption}
\usepackage{subfigure}
\usepackage[usenames,dvipsnames]{color}
\usepackage{xcolor}
\usepackage{epstopdf}
\usepackage{hyperref}
\RequirePackage[OT1]{fontenc}
\RequirePackage{amsthm,amsmath}
\usepackage{stmaryrd}
\RequirePackage[numbers]{natbib}
\RequirePackage[colorlinks,citecolor=blue,urlcolor=blue]{hyperref}
\arxiv{arXiv:0000.0000}
\usepackage{url}
\usepackage{xr}
\graphicspath{{./pics/}}
\startlocaldefs
\newtheorem{lem}{Lemma}
\newtheorem{prop}{Proposition}
\newtheorem{them}{Theorem}
\newtheorem{assum}{Assumption}
\newtheorem{corol}{Corollary}

\newcommand{\argmin}{\mathop{\rm argmin}\limits}

\newcommand{\boldtheta}{{\boldsymbol{\theta}}}
\newcommand{\bolddelta}{{\boldsymbol{\delta}}}
\newcommand{\boldthetaP}{{\boldsymbol{\theta}}^{(p)}}
\newcommand{\boldthetaQ}{{\boldsymbol{\theta}}^{(q)}}
\newcommand{\boldthetaPtop}{{\boldsymbol{\theta}}^{(p)\top}}

\newcommand{\boldH}{{\boldsymbol{H}}}

\newcommand{\boldTheta}{{\boldsymbol{\Theta}}}
\newcommand{\boldf}{{\boldsymbol{f}}}
\newcommand{\boldu}{{\boldsymbol{u}}}

\newcommand{\boldSigma}{{\boldsymbol{\Sigma}}}

\newcommand{\boldDelta}{{\boldsymbol{\Delta}}}

\newcommand{\KL}{\mathrm{KL}}

\newcommand{\boldx}{{\boldsymbol{x}}}
\newcommand{\bolde}{\boldsymbol{e}}
\newcommand{\boldxp}{{\boldsymbol{x}}_{p}}
\newcommand{\boldxq}{{\boldsymbol{x}}_{q}}
\newcommand{\boldz}{{\boldsymbol{z}}}
\newcommand{\boldg}{{\boldsymbol{g}}}
\newcommand{\boldw}{{\boldsymbol{w}}}

\newcommand{\boldzero}{{\boldsymbol{0}}}
\newcommand{\thetahat}{{\hat{\boldsymbol{\theta}}}}

\newcommand{\zhat}{{\hat{\boldsymbol{z}}}}

\newcommand{\boldy}{{\boldsymbol{y}}}

\newcommand{\boldpsi}{{\boldsymbol{\psi}}}

\newcommand{\distP}{P}
\newcommand{\distQ}{Q}
\newcommand{\iid}{\stackrel{\mathrm{i.i.d.}}{\sim}}
\newcommand{\dx}{\mathrm{d}\boldx}

\newcommand{\normDtwo}{\|\nabla_{\boldtheta}^2 \ell(\boldtheta)\|}
\newcommand{\DtwoLSE}{\nabla_{\boldg}^2 \mathrm{LSE}\left(g_1, g_2, \cdots, g_{n_q}\right)}

\newcommand{\vertiii}[1]{{\left\vert\kern-0.25ex\left\vert\kern-0.25ex\left\vert #1 
    \right\vert\kern-0.25ex\right\vert\kern-0.25ex\right\vert}}

\newenvironment{changemargin}[2]{%
\begin{list}{}{%
	\setlength{\topsep}{0pt}%
	\setlength{\leftmargin}{#1}%
	\setlength{\rightmargin}{#2}%
	\setlength{\listparindent}{\parindent}%
	\setlength{\itemindent}{\parindent}%
	\setlength{\parsep}{\parskip}%
}%
\item[]}{\end{list}}
\endlocaldefs

\begin{document}

\begin{frontmatter}
\title{Support Consistency of Direct Sparse-Change Learning in Markov Networks}
\runtitle{On Learning Sparse Changes of Markov Networks}
\date{}
\begin{aug}
	\author{\fnms{Song} \snm{Liu}\thanksref{m1}\ead[label=e1]{liu@ism.ac.jp}},
	\author{\fnms{Taiji} \snm{Suzuki}\thanksref{m2}\ead[label=e2]{s-taiji@is.titech.ac.jp}}
	\author{\fnms{Raissa} \snm{Relator}\thanksref{m3} \ead[label=e5]{rai.relator@aist.go.jp}}
	\author{\fnms{Jun} \snm{Sese}\thanksref{m3} \ead[label=e6]{sese.jun@aist.go.jp}}\\
	\author{\fnms{Masashi} \snm{Sugiyama}\thanksref{m4} \ead[label=e3]{sugi@k.u-tokyo.ac.jp}}
	\and
	\author{\fnms{Kenji} \snm{Fukumizu}\thanksref{m1}
			\ead[label=e1]{fukumizu@ism.ac.jp}}
	
	\runauthor{S. Liu et al.}
	
	\affiliation{The Institute of Statistical Mathematics\thanksmark{m1}, Tokyo Institute of Technology\thanksmark{m2}, National Institute of Advanced Industrial Science and Technology\thanksmark{m3} and University of Tokyo\thanksmark{m4}}
	
	\address{10-3 Midori-cho, Tachikawa, Tokyo\\ 190-8562, Japan\\
		\printead{e1}\\}
	\address{ 2-12-1 O-okayama, Meguro, Tokyo\\ 152-8552, Japan\\
		\printead{e2}\\}
	\address{7-3-1 Hongo, Bunkyo-ku, Tokyo\\ 113-0033, Japan\\
		\printead{e3}\\}	
	\address{2-4-7 Aomi, Koto-ku, Tokyo\\ 135-0064, Japan\\
		\printead{e5,e6}\\}	
\end{aug}

\begin{abstract}
	We study the problem of learning sparse structure changes between two
	Markov networks $P$ and $Q$. Rather than fitting two Markov networks
	separately to two sets of data and figuring out their differences, a recent work proposed
	to learn changes \emph{directly} via estimating the ratio between two
	Markov network models.  In this paper, we give sufficient
	conditions for \emph{successful change detection} with respect to the
	sample size $n_p, n_q$, the dimension of data $m$, and the number of
	changed edges $d$.  When using an unbounded density ratio model we prove that the true sparse
	changes can be consistently identified for 
	$n_p = \Omega(d^2 \log \frac{m^2+m}{2})$ and $n_q = \Omega({n_p^2})$, with an exponentially decaying
	upper-bound on learning error.  Such sample complexity can be improved to $\min(n_p, n_q) = \Omega(d^2 \log \frac{m^2+m}{2})$ when the boundedness of the density ratio model is assumed.  Our theoretical guarantee can be
	applied to a wide range of discrete/continuous Markov networks.
\end{abstract}


\end{frontmatter}

\section{Introduction}
\label{sec.intro}
Learning changes in interactions between random variables plays an important role in many real-world applications. For example, genes may regulate each other in different ways when external conditions are changed. The number of daily flu-like symptom reports in nearby hospitals may become correlated when a major epidemic disease breaks out. EEG signals from different regions of the brain may be synchronized/desynchronized when the patient is performing different activities. Identifying such changes in interactions helps us expand our knowledge on these real-world phenomena. 

In this paper, we consider the problem of learning changes between two undirected graphical models. Such a model, also known as a Markov network (MN) \citep{PGM_Koller}, expresses interactions via the conditional independence between random variables.
Hammersley-Clifford theorem \cite{HC_them} states that the joint distribution of an MN can be factorized over subsets of interacted random variables and 
general MNs may have factors over arbitrary numbers of random variables. For simplicity, we focus on a special case, namely \emph{pairwise MNs}, whose joint distribution can  be factorized over only single or pairwise random variables. 

The problem of learning structure of MN itself has been thoroughly investigated in the last decade. The graphical lasso method \citep{Banerjee_Model_Selection,Friedman_GLasso} learns a sparse precision (inverse covariance) matrix from data by using the $\ell_1$-norm, while the neighborhood regression methods \citep{Lee_EfficientL1Learning,MeinshausenBuhlmann,Ravikumar_2010} solve a node-wise lasso program to identify the neighborhood of each single node. 

One naive approach to learning changes in MNs is to apply these methods to two MNs separately and compare the learned models. 
However, such a two-step approach does not work well when the MNs themselves are dense
(this can happen even when the change in MNs is sparse).
A recent study \citep{Gaussian_Change} adopts a neighbourhood selection procedure to learn sparse changes between Gaussian MNs via a fused-lasso type regularizer \citep{Tibshirani_FusedLasso}. However, no theoretical guarantee was given on identifying changes.
Furthermore, extension of the above mentioned methods to general non-Gaussian MNs is hard
due to the computational intractability of the normalization term.

To cope with these problems, an novel algorithm has been proposed recently
\citep{liu2014ChangeDetection}.
Its basic idea is to model the changes between two MNs $P$ and $Q$ 
as the ratio between two MN density functions $p(\boldx)$ and $q(\boldx)$,
and the ratio $p(\boldx)/q(\boldx)$ is directly estimated in one-shot
without estimating $p(\boldx)$ and $q(\boldx)$ themselves \citep{Density_Ratio_Book}.
Since parameters in the density ratio model represent
the parametric difference between $P$ and $Q$,
sparsity constrains can be directly imposed for sparse change learning.
Thus, the density-ratio approach can work well
even when each MN is dense as long as the change is sparse.
Furthermore, the normalization term in the density-ratio approach can be
approximately computed by the straightforward sample average and
thus there is no computational bottleneck in using non-Gaussian MNs.
Experimentally, the density-ratio approach was demonstrated to perform well. 
However, its theoretical properties have not been explored yet.

The ability of recovering a sparsity pattern via a sparse learning algorithm 
has been studied under the name of \emph{support consistency} or sparsistency \citep{WainwrightL1Sharp},
 that is, the support of the estimated parameter converges to the true support. Previous works for 
\emph{successful structure recovery} are available for $\ell_1$-regularized maximum (pseudo-)likelihood estimators \citep{Ravikumar_2010,YangGMGLM}.
However, density ratio estimator in \cite{liu2014ChangeDetection} brought us a new question: 
what is the sparsistency of identifying correct sparse changes \emph{without} learning individual MNs? Such a concern is
very practical since in applications such as learning changes in gene expression between stimuli conditions, we only care changes rather than individual structures before or after changes. We illustrate such an example in the problem of gene regulatory networks in Section \ref{sec.gene}.

In this paper, we theoretically investigate the success of the density-ratio approach and provide sufficient conditions for \emph{successful change detection}
with respect to the number of samples $n_p$, $n_q$, data dimension $m$,
and the number of changed edges $d$.
More specifically, we prove that
if $n_p = \Omega(d^2 \log \frac{m^2+m}{2})$ and $n_q = \Omega(n_p^2)$, 
changes between two MNs can be consistently learned under mild assumptions, 
regardless the sparsity of individual MNs. Such sample complexity can be further improved to $\min(n_p, n_q) = \Omega(d^2 \log \frac{m^2+m}{2})$ when the boundedness of the density ratio model is assumed.
Technically, our contribution can be regarded as an extention of 
support consistency of lasso-type programs \citep{WainwrightL1Sharp} 
to the \emph{ratio} of MNs. 
The convergence rate does not rely on the individual sparsity of each MN, thus structures like hub-nodes can exist. Such hub structure is common in many applications, such as gene expression data where one gene regulates many other genes. 
Our theorem holds for the most general log-linear MN models, and does \emph{not} assume any special type of individual MNs (such as Gaussian or Ising). 

Note that the theoretical results presented in this paper are fundamentally different from previous works on learning a ``jumping MN'' \citep{XingJump}, where the focuses are learning the partition boundaries between jumps, and the successful recovery of graphical structure within each partition, rather than learning sparse changes between partitions. 

In previous works \cite{Ravikumar_2010,YangGMGLM}, the (upper/lower) boundedness of the Fisher information matrix, or log-partition function derivatives of a density model are often assumed. In this work, similar assumptions are imposed on the true density ratio model. Moreover, we show that such assumptions have profound links with the smoothness of our model, which implies the magnitude of change should not be too drastic for keeping the density ratio model well-behaved.
These assumptions are also automatically satisfied under some special cases.

The target of \cite{liu2014ChangeDetection} coincides with another recently proposed method  where a \emph{differential network} is learned directly using a different technique \cite{zhao2014direct}.
Without learning a precision matrix for each MN, this approach estimates a differential network utilizing a special equality obtained for Gaussian MNs. 
However, such an objective function does not generalize to  ordinary pairwise MNs. 
The theorems obtained in this paper and the ones in \cite{zhao2014direct} both rely on one similar assumption: The changes are sparse. However, theorems in this paper manage to achieve the same sample complexity of recovering the correct structure changes without explicitly assuming Gaussianity over datasets. 

This paper is organized as follows: First we introduce the problem formulation of learning changes between two MNs in Section \ref{sec.background}. Second, we review the density ratio estimation method proposed in \cite{liu2014ChangeDetection}. Then, as the main focus of this paper, we analyze the sufficient conditions for successful change detection, i.e., the support consistency of such algorithm in Section \ref{sec.main} and \ref{sec.proof.them.1}. Moreover, in Section \ref{sec.analysis.assums}, we study the key assumptions in this paper, and discuss their consequences. Through experiments in Section \ref{sec.exp}, we demonstrate the validity of our theorems and compare the performance of the density ratio approach with a state of the art method. Finally, in Section \ref{sec.gene}, we show the density ratio method successfully identifies key changes in a gene network between two stimuli conditions.

\section{Direct Change Learning between Markov Networks}
\label{sec.background}
In this section, we review a direct structural change detection method
\citep{liu2014ChangeDetection}.

\subsection{Problem Formulation}
\label{sec.prob.form}
Consider two sets of independent samples drawn separately from two probability distributions $P$ and $Q$ on $\mathbb{R}^m$:
\begin{align*}
\{\boldx_p^{(i)}\}_{i=1}^{n_p} \iid P \text{ and }  \{\boldx_q^{(i)}\}_{i=1}^{n_q} \iid Q.
\end{align*}
We assume that $P$ and $Q$ belong to the family of \emph{Markov networks} (MNs)
consisting of univariate and bivariate factors,
i.e., 
their respective probability densities $p$ and $q$ are expressed as
\begin{align}
\label{eq.density.model}
p(\boldx;\boldthetaP) =\frac{1}{Z(\boldthetaP)}\exp\left(
\sum_{u,v = 1, u\ge v}^{m} \boldthetaP_{u,v}{}^\top \boldpsi_{u,v}(x_u,x_v) \right),
\end{align}
where
$\boldx = (x_{1}, \dots, x_{m})^\top$ is the $m$-dimensional random variable, $\top$ denotes the transpose,
$\boldthetaP_{u,v}$ is the $b$-dimensional parameter vector for the elements $x_{u}$ and $x_{v}$, and
\begin{align*}
\boldthetaP = (\boldthetaPtop_{1,1},\ldots, \boldthetaPtop_{m,1},\boldthetaPtop_{2,2},\ldots,\boldthetaPtop_{m,2},\ldots,\boldthetaPtop_{m,m})^\top
\end{align*}
is the entire parameter vector.
$\boldpsi_{u,v}(x_{u},x_{v})$ is a bivariate vector-valued basis function,
and $Z(\boldthetaP)$ is the normalization factor defined as
\begin{align*}
Z(\boldthetaP)  = \int \exp\left(\sum_{u,v = 1, u\ge v}^{m}  \boldthetaP_{u,v}{}^\top \boldpsi_{u,v}(x_{u},x_{v})\right)\dx.
\end{align*}
$q(\boldx; \boldthetaQ)$ is defined in the same way.
Using these notations, we can define two well-known MNs as examples:
\paragraph{Ising-model (see e.g.\cite{PGM_Koller})}
One of the earliest and widely known graphical models is the Ising model, where $\psi_{u,v}(x_u,x_v) = x_ux_v$, and $ x_u, x_v \in \{-1,1\}$. For all pairs $(u,v) \in E$, where $E$ is the edge set of the graphical model, $\theta_{u,v}$ is a scalar and has non-zero value. 
\paragraph{Gaussian MN} Gaussian MN is a representative of continuous MN. $\psi_{u,v}(x_u,x_v) = x_ux_v$ and $x_u, x_v \in \mathbb{R}$. For all pairs $(u,v) \in E$ or $u=v$, $\theta_{u,v}$ is a scalar and has non-zero value.

The research problem now becomes clear: Given two parametric models $p(\boldx;\boldthetaP)$ and $q(\boldx;\boldthetaQ)$, we hope to discover \emph{changes in parameters} from $P$ to $Q$, i.e., $\boldthetaP-\boldthetaQ$.

\subsection{Density Ratio Formulation for Structural Change Detection}
The key idea in \citep{liu2014ChangeDetection} is to consider the ratio of $p$ and $q$:

\begin{align*}
\frac{p(\boldx; \boldthetaP)}{q(\boldx; \boldthetaQ)}
\propto \exp \left( \sum_{u,v = 1, u\ge v}^m (\boldthetaP_{u,v}-\boldthetaQ_{u,v})^\top \boldpsi_{u,v}(x_{u},x_{v})\right),
\end{align*}
where $\boldthetaP_{u,v}-\boldthetaQ_{u,v}$ encodes the difference
between $\distP$ and $\distQ$ for factor $\boldpsi_{u,v}(x_{u},x_{v})$, i.e., 
$\boldthetaP_{u,v} - \boldthetaQ_{u,v}$ is zero
if there is no change in the factor $\boldpsi_{u,v}(x_{u},x_{v})$.




Once the ratio of $p$ and $q$ is considered, 
each parameter $\boldthetaP_{u,v}$ and $\boldthetaQ_{u,v}$ does not have to be estimated,
but only their difference $\boldtheta_{u,v}=\boldthetaP_{u,v} - \boldthetaQ_{u,v}$
is sufficient to be estimated for change detection.
Thus, in this density-ratio formulation,
$p$ and $q$ are no longer modeled separately,
but it models the changes from $p$ to $q$ \emph{directly} as
\begin{align}
\label{ratio.model.def}
r(\boldx;\boldtheta) =
\frac{1}{N(\boldtheta)} \exp\left(\sum_{u,v = 1, u\ge v}^m \boldtheta_{u,v}^\top \boldpsi_{u,v}(x_{u},x_{v})\right),
\end{align}
where $N(\boldtheta)$ is the normalization term.
This direct formulation also halves the number of parameters
from both $\boldthetaP$ and $\boldthetaQ$ to only $\boldtheta$.

The normalization term $N(\boldtheta)$ is 
chosen to fulfill
$\int q(\boldx)r(\boldx;\boldtheta) \dx = 1$:
\begin{align*}
N(\boldtheta) = \int q(\boldx) \exp\left(\sum_{u,v = 1, u\ge v}^m \boldtheta_{u,v}^\top \boldpsi_{u,v}(x_{u},x_{v}) \right)\dx,  
\end{align*}
which is the expectation over $q(\boldx)$.
This expectation form of the normalization term is another notable advantage
of the density-ratio formulation because it can be easily 
approximated by the sample average over $\{\boldxq^{(i)}\}_{i=1}^{n_q}\iid q(\boldx)$:
\begin{align*}
&\hat{N}(\boldtheta; \boldx_q^{(1)}, \dots, \boldx_q^{(n_q)}) := \frac{1}{n_q}\sum_{i=1}^{n_q}
\exp\left(\sum_{u,v = 1, u\ge v}^m \boldtheta_{u,v}^\top \boldpsi_{u,v}(x_{q,u}^{(i)}, x_{q,v}^{(i)}) \right).
\end{align*}
Thus, one can always use this empirical normalization term 
for any (non-Gaussian) models $p(\boldx; \boldthetaP)$ and $q(\boldx; \boldthetaQ)$.



An important observations can be made from this formulation:
Although two MNs may have sophisticated models individually, their changes might be ``simple'' since many terms may be canceled while taking the ratio, i.e. $\boldtheta^{(p)}_{u,v} -  \boldtheta^{(q)}_{u,v}$ might be zero. Thus, if we use $\psi_{u,v}(x_u x_v) = x_u x_v$ in our ratio model, it does not mean we assume two individual MNs are Gaussian or Ising, it simply means we assume the changes of interactions are linear while other non-linear interactions remain unchanged. This formulation allows us to consider highly complicated MNs as long as their changes are ``simple''. We will give a concrete example later. 

Throughout the rest of the paper, we simplify the notation from $\boldpsi_{u,v}$ to $\boldpsi$ by assuming the feature functions are the same for all pairs of random variables. However, our analysis still holds if this assumption is violated. Next, we study the density ratio formulation in the case of Gaussian MNs.

\paragraph{Gaussian MN}: Given two $m$-dimensional zero-mean Gaussian MNs $p(\boldx;\boldTheta^{(p)}) $ and $q(\boldx;\boldTheta^{(q)})$ parameterized by the precision matrix $\boldTheta^{(p)}$ and $\boldTheta^{(q)}$ respectively, it is reasonable to parametrize a density ratio model 
\begin{align}
\label{eq.gaussian.model}
r(\boldx; \boldDelta) = \frac{1}{N(\boldDelta)} \exp\left( - \frac{1}{2}\boldx^\top \boldDelta \boldx\right),
\end{align}
where $\boldDelta$ is a symmetric real-valued matrix and 
\begin{align*}
N(\boldDelta)
= \int q(\boldx; \boldTheta^{(q)}) \exp\left( - \frac{1}{2} \boldx^\top \boldDelta \boldx\right) \dx 
= \frac{\det\left(\boldTheta^{(q)}\right)^{1/2}}{\det\left(\boldDelta + \boldTheta^{(q)}\right)^{1/2}}.
\end{align*}
However, this formulation brings a problem: it still contains an unknown parameter $\boldTheta^{(q)}$, meaning that we will have to learn $\boldTheta^{(q)}$ first before we can model the difference between two MNs.
To solve this problem, one may use the empirical version of the normalization term instead
\begin{align}
\label{eq.gaussian.norm}
\hat{N}(\boldDelta) = \frac{1}{n_q} \sum_{i=1}^{n_q} \exp\left( - \frac{1}{2}{\boldx_q^{(i)}}^\top \boldDelta \boldx_q^{(i)} \right).
\end{align}
Interestingly, by using this model, the Gaussianity assumption has been loosened: the normalization term is obtained by an empirical average and did not use the analytical form offered by the Gaussianity of $q(\boldx;\boldTheta^{(q)})$. Thus, it can actually model the density ratio for any $p$ and $q$ as long as their changes are limited to the quadratic components.

\subsection{Direct Density-Ratio Estimation} 
\label{estimate.ratio.sec}
Density ratio estimation has been recently introduced
to the machine learning community 
and is proven to be useful in a wide range of applications
\citep{Density_Ratio_Book}.
In \citep{liu2014ChangeDetection}, a density ratio estimator called 
the \emph{Kullback-Leibler importance estimation procedure} (KLIEP) 
for log-linear models \citep{Covariate_Shift,Log_Linear_KLIEP} was employed
in learning structural changes.


For a density ratio model $r(\boldx; \boldtheta)$,
the KLIEP method minimizes the Kullback-Leibler divergence
from $p(\boldx)$ to $\hat{p}(\boldx;\boldtheta) = q(\boldx) r(\boldx;\boldtheta)$:
\begin{align}
\KL[p\|\hat{p}_\boldtheta] 
= \int p(\boldx) \log\frac{p(\boldx)}{q(\boldx)r(\boldx;\boldtheta)} \dx
=\text{Const.} - \int p(\boldx) \log r(\boldx; \boldtheta) \dx.
\label{eq.obj}
\end{align}
Note that the density-ratio model \eqref{ratio.model.def} automatically
satisfies the non-negativity and normalization constraints:
\begin{align*}
r(\boldx;\boldtheta) \ge 0
~~\mbox{and}~~
\int q(\boldx) r(\boldx; \boldtheta) \dx = 1.
\end{align*}
Here we define \[\hat{r}(\boldx; \boldtheta) = \frac{
	\exp \left({\sum_{u,v = 1, u\ge v}^m \boldtheta_{u,v}^\top \boldpsi(x_{q,u},x_{q,v})}\right)}
{\hat{N}(\boldtheta; \boldx_q^{(1)}, \dots, \boldx_q^{(n_q)}) }\] as the \emph{empirical density ratio model}. In practice, one minimizes
the negative empirical approximation of the rightmost term in Eq.\eqref{eq.obj}:
\begin{align*}
\ell_{\mathrm{KLIEP}}(\boldtheta) = 
-\frac{1}{n_p}\sum_{i=1}^{n_p} \log &\hat{r}(\boldxp^{(i)}; \boldtheta)
= - \frac{1}{n_p}\sum_{i=1}^{n_p} \sum_{u,v = 1, u\ge v}^m \boldtheta_{u,v}^\top \boldpsi(x_{p,u}^{(i)},x_{p,v}^{(i)}) \\
&+\log \left(\frac{1}{n_q}\sum_{i=1}^{n_q} \exp\left(\sum_{u,v = 1, u\ge v}^m \boldtheta_{u,v}^\top \boldpsi(x_{q,u}^{(i)},x_{q,v}^{(i)})\right)\right),
\end{align*}

Because $\ell_{\mathrm{KLIEP}}(\boldtheta)$ is convex with respect to $\boldtheta$,
its global minimizer can be numerically found by standard optimization techniques
such as gradient descent or quasi-Newton methods.
The gradient of $\ell_{\mathrm{KLIEP}}$ with respect to $\boldtheta_{u,v}$ is given by
\begin{align}
\label{gradient.def}
\nabla_{\boldtheta_{u,v}} \ell_{\text{KLIEP}}(\boldtheta)
&=
-\frac{1}{n_p}\sum_{i=1}^{n_p} \boldpsi (x_{p,u}^{(i)},x_{p,v}^{(i)}) + \frac{1}{n_q} \sum_{i=1}^{n_q} \hat{r}(\boldx^{(i)}; \boldtheta) \boldpsi(x_{q,u'}^{(i)},x_{q,v'}^{(i)}),
\end{align}
that can be computed in a straightforward manner for \emph{any} feature vector $\boldpsi(x_{u},x_{v})$.
\paragraph{Importance Sampling} From the gradient of KLIEP \eqref{gradient.def}, we can observe a clear link between KLIEP and Importance Sampling (see e.g., \cite{RobertMCStat2005}). The second term on the right-hand side is an ``importance sampled'' approximation of $\mathbb{E}_{p}\left[\boldpsi(x_u, x_v)\right]$ using our density ratio model while the first term is a straightforward sample average. 
The population version of \eqref{gradient.def} equals zero if and only if $p(\boldx)  = \hat{r}(\boldx;\boldtheta) q(\boldx)$. 
Therefore, one should aware that the assignment of $p$ and $q$ may affect the performance of such an approximation as importance sampling can be easily affected by the choice of the instrumental distribution (in this case, $q$). To reduce the estimation variance, $q$ is usually picked as the one with a thicker tail \cite{WassermanAllStat2010}. 
This observation reveals a fundamental asymmetry of KLIEP which will be discussed in Section \ref{sec.discussion}.

\paragraph{Gaussian MN} By using the density ratio model of \eqref{eq.gaussian.model} and the normalization term \eqref{eq.gaussian.norm}, we can write the objective function and its gradient as
\begin{align}
\label{eq.gaussian.ll}
\ell_\mathrm{KLIEP}(\boldDelta) = \frac{1}{2n_p} \sum_{i=1}^{n_p} {\boldx_p^{(i)}}^\top \boldDelta {\boldx_p^{(i)}} + \log \frac{1}{n_q} \sum_{i=1}^{n_q} \exp\left(-\frac{1}{2}{\boldx_q^{(i)}}^\top \boldDelta {\boldx_q^{(i)}}\right),
\end{align}
and 
\begin{align*}
\nabla_{\Delta_{u,v} }&\ell_\mathrm{KLIEP}(\boldDelta) =\\  &\frac{1}{2n_p} \sum_{i=1}^{n_q} x^{(i)}_{p,u} \Delta_{u,v} x^{(i)}_{p,v} - \frac{\frac{1}{n_q} \sum_{i=1}^{n_q}\exp\left(-\frac{1}{2}{\boldx_q^{(i)}}^\top \boldDelta {\boldx_q^{(i)}}\right) x^{(i)}_{q,u} \Delta_{u,v} x^{(i)}_{q,v}  }{\frac{2}{n_q}\sum_{j=1}^{n_q}\exp\left(-\frac{1}{2}{\boldx_q^{(j)}}^\top \boldDelta {\boldx_q^{(j)}}\right)}.
\end{align*}



\subsection{Sparsity-Inducing Norm}
\label{sec.sparse}
To find a sparse change between $P$ and $Q$,
one may regularize the KLIEP solution
with a sparsity-inducing norm $\sum_{u\ge v} \| \boldtheta_{u,v} \|$, i.e.,
the \emph{group-lasso} penalty \citep{YuanLi2006GroupLasso} where we use $\|\cdot\|$ to denote the $\ell_2$ norm.
Note that 
the separate density estimation approaches sparsify both $\boldtheta_p$ and $\boldtheta_q$
so that the difference $\boldtheta_p-\boldtheta_q$ is also sparsified.
On the other hand, the density-ratio approach \citep{liu2014ChangeDetection}
directly sparsifies the difference $\boldtheta_p-\boldtheta_q$,
and thus intuitively this method can still work well even if $\boldtheta_p$ and $\boldtheta_q$ are dense
as long as $\boldtheta_p-\boldtheta_q$ is sparse.


Now we have reached our final objective:
\begin{align}
\label{eg.obj.final}
\thetahat = \argmin_{\boldtheta} \ell_{\text{KLIEP}}(\boldtheta) + \lambda_{n_p} \sum_{u,v = 1, u \ge v}^m\| \boldtheta_{u,v} \|.
\end{align}

\section{Support Consistency of Direct Sparse-Change Detection}
\label{sec.main}
The above density-ratio approach to change detection
was demonstrated to be promising in empirical studies \citep{liu2014ChangeDetection}.
However, its theoretical properties have not yet been investigated. 
In this section, we give theoretical guarantees of the convex program \eqref{eg.obj.final}
on sparse structural change learning.
More specifically, we give \emph{sufficient conditions} for detecting correct changes
in terms of the sample size $n_p$ and $n_q$, data dimensions $m$, and the number of changed edges $d$,
followed by the discussion of the insights we can gain from such theoretical analysis.

\subsection{Notations}

%
In the previous section, a sub-vector of $\boldtheta$ indexed by $(u,v)$ corresponds to a specific edge of an MN. From now on, we use new indices with respect to the ``oracle'' sparsity pattern of the true parameter for notational simplicity. 
We introduce the ``true parameter'' notation $\boldtheta^*, p(\boldx)=q(\boldx)r(\boldx;\boldtheta^*),$
and the pairwise index set $E = \{(u,v) | u\ge v\}$.
Two sets of \textit{sub-vector indices} regarding to $\boldtheta^*$ and $E$ are  defined as  $S = \{t'\in E ~|~ \|\boldtheta^*_{t'}\| \neq 0\}, S^c = \{t'' \in E ~|~ \|\boldtheta^*_{t''}\| = 0\}.$ We rewrite the objective \eqref{eg.obj.final} as
\begin{align}
\label{eq.obj.alter}
\thetahat = \argmin_{\boldtheta} \ell_{\mathrm{KLIEP}}(\boldtheta) &+ \lambda_{n_p} \sum_{t'\in S} \| \boldtheta_{t'} \| + \lambda_{n_p} \sum_{t''\in S^c}\|\boldtheta_{t''}\|.
\end{align}
Similarly we can define 
$\hat{S} = \{t' \in E ~|~ \|\hat{\boldtheta}_{t'}\| \neq 0\}$ and $\hat{S^c}$ accordingly.
Sample Fisher information matrix $\mathcal{I} \in \mathbb{R}^{\frac{b(m^2+m)}{2} \times \frac{b(m^2+m)}{2}}$ denotes the Hessian of the log-likelihood: $\mathcal{I} = \nabla^2 \ell_{\text{KLIEP}} (\boldtheta^*)= \nabla^2 \log\hat{N}(\boldtheta^*)$ where we simplify $\hat{N}(\boldtheta;\boldx_q^{(1)}, \dots, \boldx_q^{(n_q)})$ as $\hat{N}(\boldtheta) $. $\mathcal{I}_{AB}$ is a sub-matrix of $\mathcal{I}$ indexed by two sets of indices $A, B \subseteq E$ on rows and columns. 

We also concatenate $\boldpsi \in \mathbb{R}^2 \mapsto \mathbb{R}^{b}$ to get a ``linearized'' version of the feature function $\boldf: \mathbb{R}^{m} \mapsto \mathbb{R}^{\frac{b(m^2+m)}{2}}$ as
\begin{align*}
\boldf(\boldx) = (\boldpsi^\top(x_1, x_1),\ldots, \boldpsi^\top(x_m, x_1),\boldpsi^\top(x_2, x_2),\ldots,
\boldpsi^\top(x_m, x_2),\ldots,\boldpsi^\top(x_m, x_m))^\top
\end{align*}
and $\boldf_A(\boldx)$ is the partial output of $\boldf(\boldx)$ indexed by a set of indices $A, A\subseteq E$.

\paragraph{Gaussian MN}
Here we derive the Fisher information matrix for the Gaussian MN ratio model. Define an auxiliary matrix $\boldH(\boldDelta) \in \mathbb{R}^{n_q \times n_q}$:
\begin{align*}
\boldH(\boldDelta) := \frac{1}{\widehat{N}^2(\boldDelta)} \left(\widehat{N}(\boldDelta) I_{n_q} - \bolde^\top(\boldDelta)\bolde(\boldDelta)\right),
\end{align*}
where $\bolde(\boldDelta) := \left[\exp\left( - \frac{1}{2}{\boldx_q^{(1)}}^\top \boldDelta \boldx_q^{(1)}\right),\dots, \exp\left( - \frac{1}{2}{\boldx_q^{(n_q)}}^\top \boldDelta \boldx_q^{(n_q)}\right)\right]$.
The Fisher information matrix $\mathcal{I} = \nabla^2_{\boldDelta}\ell_{\mathrm{KLIEP}}(\boldDelta^*)$ of the likelihood function using the Gaussian density ratio model described in \eqref{eq.gaussian.ll} has the form
\begin{align*}
\mathcal{I}_{(u,v),(u',v')} := \sum_{i=1}^{n_q} \sum_{j=1}^{n_q} x^{(i)}_{q,u} x^{(i)}_{q,v} x^{(j)}_{q,u'} x^{(j)}_{q,v'} H_{i,j}(\boldDelta^*), ~\mathcal{I} \in \mathbb{R}^{m^2\times m^2}.
\end{align*}
%

\subsection{Assumptions}
There is an important guideline for imposing assumptions in this paper: we try \emph{not} to put any explicit constrains on the types of individual MN $P$ or $Q$ nor their structures, but only on the changes between them. This is crucial since KLIEP is a \emph{direct} and \emph{flexible} change learning method and have no restrictions on the types of individual MNs on which it works.
Therefore, we hope to obtain the most generic theorem for this method.

Similarly to previous researches on sparsity recovery analysis \citep{WainwrightL1Sharp,Ravikumar_2010}, the first two assumptions are made on the Fisher information matrix. 

\begin{assum}[Dependency Assumption]
	\label{assum.depen}
	The sample Fisher information \textbf{submatrix}  $\mathcal{I}_{{SS}}$ has bounded eigenvalues:
	\begin{align*}
	\Lambda_\mathrm{min}(\mathcal{I}_{{SS}}) \ge \lambda_\mathrm{min} > 0,
	\end{align*}
	with probability 1, where $\Lambda_{\mathrm{min}}$ is the minimum-eigenvalue operator of a symmetric matrix
\end{assum}
This assumption on the \emph{submatrix} of $\mathcal{I}$ is to ensure that the model is identifiable (see \ref{sec.proof.lemma.1} in Appendix for details). Note ``$\lambda$'' denotes either eigenvalue or regularization parameter depending on its subscript.

\begin{assum}[Incoherence Assumption]
	\label{assum.incoherence}
	\begin{align*}
	\max_{t'' \in S^c}\|\mathcal{I}_{t''S} \mathcal{I}_{SS}^{-1}\|_1 \le 1-\alpha, 0<\alpha \le 1.
	\end{align*} 
	with probability 1, where $\|Y\|_1 = \sum_{i,j} \|Y_{i,j}\|_1$.
\end{assum}This assumption says the unchanged edges cannot exert overly strong effects on changed edges and is a common assumption can be found in previous literatures on support consistency analysis such as \cite{WainwrightL1Sharp,Ravikumar_2010}.

\begin{assum}[Smoothness Assumption on Likelihood Ratio]
	\label{assum.smooth}
	The log-likelihood ratio $\ell_\mathrm{KLIEP}(\boldtheta)$ is smooth around its optimal value, i.e., it has bounded derivatives
	\begin{align}
	\label{eq.assum.max_second}
	&\max_{\bolddelta, \|\bolddelta\|\leq \|\boldtheta^*\|}\left\| \nabla^2 \ell_\mathrm{KLIEP}(\boldtheta^*+\bolddelta)\right\|
	 = \max_{\bolddelta, \|\bolddelta\|\leq \|\boldtheta^*\|} \left\|\nabla^2 \log \hat{N}(\boldtheta^*+\bolddelta)\right\| \leq \lambda_\mathrm{max} < \infty,\\
	&\max_{t\in S \cup S^c} \max_{\bolddelta, \|\bolddelta\|\leq \|\boldtheta^*\|}  \vertiii{\nabla_{\boldtheta_t}\nabla^2 \ell_\mathrm{KLIEP}(\boldtheta^* + \bolddelta)} \\
	= &\max_{t\in S \cup S^c} \max_{\bolddelta, \|\bolddelta\|\leq \|\boldtheta^*\|}  \vertiii{\nabla_{\boldtheta_t}\nabla^2 \log \hat{N}(\boldtheta^* + \bolddelta)} \leq \lambda_{3,\mathrm{max}}<\infty ,\notag
	\end{align}
	with probability $1$. 
\end{assum}
 $\left\|\cdot\right\|$,  $\vertiii{\cdot}$ are the spectral norms of a matrix and a tensor respectively (See e.g., \cite{tomioka2014spectral} for the definition of spectral norm of a tensor).
Note that \eqref{eq.assum.max_second} also implies the bounded largest eigenvalue of $\mathcal{I}$. Assumption \ref{assum.smooth} can be regarded as an analogy of assumptions on the log-normalization function in \cite{YangGMGLM}. As we set no explicit restrictions on the type of distribution $P$ and $Q$, this assumption guarantees the log-likelihood function is well-behaved. 


Now, we make the following assumptions on the density ratio:
\begin{assum}[The Correct Model Assumption]
	\label{assum.correct}
	The density ratio model is correct, i.e. there exists $\boldtheta^*$ such that
	\begin{align*}
	p(\boldx) = r(\boldx;\boldtheta^*)q(\boldx).
	\end{align*}
\end{assum}

Assumptions \ref{assum.depen}, \ref{assum.incoherence}, and \ref{assum.smooth} are in fact related to distribution $Q$. However, the density ratio estimation objective is an M-estimator summed up over samples from $P$. Assumption \ref{assum.correct} provides a transform between $P$ and $Q$ and allows us to perform analysis on such an M-estimator in an ``importance sampling'' fashion.

Next, we impose assumptions on the ``smoothness'' of the density ratio model. Generally speaking, if we expect good performance from the density ratio estimator, the density ratio model should be ``well-behaved''. The following assumption quantifies such an intuition. 

\begin{assum}[Smooth Density Ratio Model Assumption]
	\label{assum.smooth.ratiomodel.nod}
	For any vector $\bolddelta \in \mathbb{R}^{\text{dim}(\boldtheta^*)}$ such that $\|\bolddelta\|\leq \|\boldtheta^*\|$ and every $a\in \mathbb{R}$, the following inequality holds:
	\begin{align*}
	\mathbb{E}_q \left[\exp\left( a\left( r(\boldx, \boldtheta^* + \bolddelta) - 1 \right)\right) \right]  \le \exp\left(10a^2\right).
	\end{align*}
\end{assum}

We list a few consequences of the Assumption \ref{assum.smooth.ratiomodel.nod}.

\begin{prop}
	\label{eq.assum.well.behave.nod}
	For all $\epsilon > 0$ and	for any vector $\bolddelta \in \mathbb{R}^{\text{dim}(\boldtheta^*)}$ such that $\|\bolddelta\|\leq \|\boldtheta^*\|$,
$
	P\left( r(\boldx, \boldtheta^* + \bolddelta) - 1 \ge \epsilon \right) \le 2\exp\left( - \frac{\epsilon^2}{40}\right).
$
\end{prop}

Using Assumption \ref{assum.smooth.ratiomodel.nod}, we get Proposition \ref{eq.assum.well.behave.nod} that provides a tail probability bound of the density ratio model on $Q$, which is further used to obtain an exponentially decaying upper-bound of empirical approximation error of the log-normalization term (see Proposition \ref{prop.norm.bound} in Appendix for details). 



\begin{prop}
	\label{prop.bounded.variance}
	For any vector $\bolddelta \in \mathbb{R}^{\text{dim}(\boldtheta^*)}$ such that $\|\bolddelta\|\leq \|\boldtheta^*\|$,
$
	\mathrm{Var}_q \left[ r(\boldx;\boldtheta^* + \bolddelta) -1 \right] \le 20.
$
\end{prop}
\begin{proof}
	Noting $\mathbb{E}_q\left[r(\boldx;\boldtheta^*+\bolddelta) -1 \right] = 0$, the above inequalities is the consequence of sub-Gaussianity. 
\end{proof}

Since the density ratio can be thought as the magnitude of change between two MNs, Proposition \ref{prop.bounded.variance} tells the fact that the change should not be too drastic in order to keep our ratio-model well-behaved.


We are now ready to state the main theorem.

\subsection{Sufficient Conditions for Successful Change Detection}

The following theorem establishes sufficient conditions of change detection in terms of parameter sparsity. 
Its proof is provided in Section \ref{sec.proof.maintext}. First, let us define $g(m) = \frac{\log(m^2+m)}{(\log\frac{m^2+m}{2})^2}$ (see Figure \ref{fig.gm} in Appendix for its plot) which is smaller than 1 when $m>4$. 

\begin{them}
	\label{them.the.main.theorem}
	Suppose that Assumptions \ref{assum.depen}, \ref{assum.incoherence}, \ref{assum.smooth}, \ref{assum.correct}, and \ref{assum.smooth.ratiomodel.nod} as well as $\min_{t\in S} \|\boldtheta^*_t\| \geq \frac{10}{\lambda_\mathrm{min}} \sqrt{d}\lambda_{n_p}$ are satisfied, where $d$ is the number of changed edges defined as $d = |S|$, i.e., the cardinality of the set of non-zero parameter groups.
	Suppose also that the regularization parameter is chosen so that
	\begin{align}
	\label{eq.lambda.condition}
	\frac{8(2-\alpha)}{\alpha} \sqrt\frac{{M_1 \log \frac{m^2+m}{2}}}{n_p} \le \lambda_{n_p} \le \frac{4 (2-\alpha)M_1}{\alpha}  \min\left(\frac{\|\boldtheta^*\|}{\sqrt{b}}, 1\right),
	\end{align}
	 where $M_1=\lambda_{\text{max}}b+2$, $n_q \ge M_2 n_p^2 g(m)$ and $M_2$ is a positive constant.
	Then there exist some constants $L_1$,  $K_1$, and $K_2$ such that if $n_p\geq L_1 d^2\log  \frac{m^2+m}{2}$, with the probability at least
	\begin{align}
	\label{eq.them1.decay}
		1- \exp\left( - K_1 \lambda_{n_p}^2n_p \right) - 4\exp\left( -K_2 dn_q \lambda_{n_p}^4 \right),
	\end{align}
	the following properties hold:
	\begin{itemize}
		\item Unique Solution: The solution of \eqref{eq.obj.alter} is unique.
		\item Successful Change Detection: $\hat{S} = S$ and $\hat{S}^c = S^c$, where $\hat{S}$ and $\hat{S^c}$ are estimated sparse/non-sparse indices.
	\end{itemize}
\end{them}

 First, it is interesting to analyze the sample complexity of $n_q$, which is a novel element in this research. Intuitively, one should obtain a sufficient number of samples from $Q$ to accurately approximate the normalization term.   Theorem \ref{them.the.main.theorem} states $n_q$ should grow at least quadratically with respect to $n_p$, which is undesirable if $n_p$ is large. In the next corollary, we discuss a relaxed coupling between $n_p$ and $n_q$ with some extra but mild cost.


Second, Assumption \ref{assum.smooth.ratiomodel.nod} together with Proposition \ref{prop.bounded.variance} shows the variation allowed for the density ratio model is irrelevant to the number of changed edges $d$.
This implies that, if $d$ is large, we are only able to detect weak changes that do not cause huge fluctuations in the density ratio model, which is rather restrictive and  unrealistic in some occasions, since the magnitude of change usually increases when the number of changed edges $d$ increases.
Below, we consider another more relaxed scenario,
where the assumption on the smoothness of the density ratio model is allowed to grow with $d$. 

\begin{assum}
	\label{assum.smooth.ratiomodel}
	For any vector $\bolddelta \in \mathbb{R}^{\text{dim}(\boldtheta^*)}$ such that $\|\bolddelta\|\leq \|\boldtheta^*\|$ and every $a\in \mathbb{R}$, the following inequality holds:
	\begin{align*}
	\mathbb{E}_q \left[\exp\left( a\left( r(\boldx, \boldtheta^* + \bolddelta) - 1 \right)\right) \right]  \le \exp\left( 10da^2\right),
	\end{align*}
	where $d$ is the number of changed edges.
\end{assum}

\begin{prop}
	\label{eq.assum.well.behave}
	For some small constants $\epsilon$
	and any vector $\bolddelta \in \mathbb{R}^{\text{dim}(\boldtheta^*)}$ such that $\|\bolddelta\|\leq \|\boldtheta^*\|$,
	then $P\left( r(\boldx, \boldtheta^* + \bolddelta) - 1 \ge \epsilon \right) \le 2\exp\left( - \frac{\epsilon^2}{40d}\right).$
\end{prop}

\begin{prop}
	\label{prop.bounded.variance.with.d}
	For any vector $\bolddelta \in \mathbb{R}^{\text{dim}(\boldtheta^*)}$ such that $\|\bolddelta\|\leq \|\boldtheta^*\|$,
$
	\mathrm{Var}_q \left[ r(\boldx;\boldtheta^* + \bolddelta) -1 \right] \le 20d.
$
\end{prop}

From Proposition \ref{prop.bounded.variance.with.d} we can see the magnitude of changes between MNs are allowed to grow at most linearly with $d$. Now we see how much this will bring changes to our sufficient conditions:

\begin{corol}
	\label{corol.relaxed}
	Suppose that Assumptions \ref{assum.depen}, \ref{assum.incoherence}, \ref{assum.smooth}, \ref{assum.correct}, and \ref{assum.smooth.ratiomodel} are satisfied, $\min_{t\in S} \|\boldtheta_t^*\|$ satisfies the condition in Theorem \ref{them.the.main.theorem}, and
	the regularization parameter is chosen so that 
	\begin{align*}	
	\frac{2-\alpha}{\alpha} \sqrt\frac{M_1{\log \frac{m^2+m}{2}}}{n_p^{\frac{3}{4}}} \le \lambda_{n_p} \le 
	\frac{4 (2-\alpha)M_1}{\alpha} \min\left(\frac{\|\boldtheta^*\|}{\sqrt{b}}, \frac{1}{n_p^{1/8}}\right),
	\end{align*}
	where $M_1=\lambda_{\mathrm{max}}b+2$, $n_q \ge  M_2 d n_p g(m)  ,$ and $M_2$ is some positive constant. Then there exist some constant $L_1$ such that if $n_p\geq L_1 d^\frac{8}{3} \left({\log \frac{m^2+m}{2}}\right)^{\frac{4}{3}}$, KLIEP has the same properties as those stated in Theorem \ref{them.the.main.theorem}.
\end{corol}
See Appendix \ref{sec.proof.corol} for the proof. Corollary \ref{corol.relaxed} states that it is possible to 
drop the growth rate of $n_q$ on $n_p$ from 2 to 1 with the cost that $n_p$ has to grow with $d^{\frac{8}{3}}$ (rather than just $d^2$ in the previous case). This is an encouraging result, since with slight changes on growth rate with respect to $d$ and $\log(\frac{m^2+m}{2})$, we are able to consider a milder coupling between $n_p$ and $n_q$.

Moreover, under the weaker Assumption \ref{assum.smooth.ratiomodel}, $n_q$ now grows linearly with $d$. It shows the prices we need to pay when consider the magnitude of changes increasing with $d$.

So far, we have only considered the scaling quadruple $(n_p, n_q, d, m)$. However, it is also interesting to consider that the scalability of our theorem relative to $b$, the dimension of the pairwise feature vector. 
This is a realistic scenario: It may be difficult to know the true underlying model of MN in practice, and thus we may adopt a model that contains many features to be ``flexible enough'' to describe the interactions among data. In the following corollary, we restate Theorem \ref{them.the.main.theorem} with $b$ and a new scalar $s$, which is the maximum number of non-zero elements in a pairwise feature vector $\boldpsi$. We assume that the positions of non-zero elements are independent of each sample $\boldx$.  

\begin{corol}
	\label{corol.compex.model}
	Suppose that Assumptions \ref{assum.depen}, \ref{assum.incoherence}, \ref{assum.smooth}, \ref{assum.correct}, and \ref{assum.smooth.ratiomodel.nod} are satisfied, $\min_{t\in S} \|\boldtheta_t^*\|$ satisfies the condition in Theorem \ref{them.the.main.theorem}, and the regularization parameter is chosen so that 
	\begin{align*}
		\frac{8(2-\alpha)}{\alpha} \sqrt\frac{{M_1 s\log \frac{m^2+m}{2}}}{n_p} \le \lambda_{n_p} \le
		\frac{4 (2-\alpha)M_1}{\alpha} \min\left(\frac{\|\boldtheta^*\|}{\sqrt{b}}, 1\right),
	\end{align*}
	\[
	\]
	where $M_1=\lambda_{\text{max}}b+2$, $n_q \ge M_2s n_p^2 g'(m) $ and $M_2$ is some positive constant, and $g'(m) = \frac{\log\left((m^2+m){b \choose s}\right)}{(\log\frac{m^2+m}{2})^2}$. Then there exist some constant $L_1$ such that if $n_p\geq L_1 s d^2  \log \frac{m^2+m}{2}$, KLIEP has the same properties as those stated in Theorem \ref{them.the.main.theorem}. 
\end{corol}
See Appendix \ref{sec.proof.corol} for the proof. From Corollary \ref{corol.compex.model}, we can see that required $n_p$ and $n_q$ for change detection grows only linearly with respect to $s$, and $n_q$ grows mildly with respect to $b \choose s$. Therefore, it is possible for one to consider a highly flexible model in practice. 


\subsection{Discussions}

From the above theorem, one may gather some interesting insights into change detection
based on density ratio estimation.

First, the required number of samples depends solely on $d$ and $m$ and is irrelevant to the number of edges of each MN. In contrast, separate graphical structural learning methods require more samples when each MN gets denser in terms of the number of edges or neighborhood \citep{MeinshausenBuhlmann,Ravikumar_2010,RaskuL1GMRF}. This establishes the superiority of the density-ratio approach in sparse change detection between dense MNs. In other words, in order to detect sparse changes, the density-ratio approach does not require the individual MN to be sparse. 

Second, the growth of $n_q$ is also lower-bounded and grows quadratically with respect to $n_p$. This result illustrates the consequence of introducing a sample approximated normalization term $\hat{N}(\boldtheta)$. An insufficient number of samples from $Q$ would lead to poor approximation of the normalization term, and makes change detection more difficult. 
Fortunately, such growth rate can be further relaxed, and with slightly increased sample complexity of $n_p$.

Finally, our theorem also points out the limits of the density-ratio approach. 
Our analysis shows that the density ratio model may not deviate too much from its mean $1$ over distribution $Q$. A previous study on another density ratio estimator also has a similar observation \citep{NC:Yamada+etal:2013}.
Since the density ratio indicates how much $P$ differs from $Q$, this analysis generally implies that to make KLIEP work, the discrepancy between $P$ and $Q$ should be mild. This is a reasonable assumption since we have already assumed that the changes in the MN structure are sparse. In high dimensional setting, it implies $P$ and $Q$ are similar. Similar assumption can be found in \cite{zhao2014direct} where the $\ell$-1 norm of the differences between two precision matrices is bounded.
\section{Proof of Support Consistency}
\label{sec.proof.them.1}

\subsection{The Proof Outline of the Main Theorem}
\label{sec.proof.maintext}
The procedure of the main proof partially follows the steps of previous support consistency proofs using the \emph{primal-dual witness} method \citep{WainwrightL1Sharp}, however, the problem settings are quite different:
First, $\ell_{\text{KLIEP}}$ is a \emph{likelihood ratio} between two densities which means that two sets of samples are involved in this proof and we have to consider the sparsity recovery conditions not only on one dataset, but with respect to two different MNs. Second, we did not explicitly limit the types of distribution for $P$ and $Q$, and the parameter of each factor $\boldtheta_t, t\in S\cup S^c$ is a vector rather than a scalar, which gives enough freedom for modelling highly complicated distributions. To the best of our knowledge, this is the first sparsity recovery analysis on learning changes between two type-free MNs. From now on, $\ell_\mathrm{KLIEP}$ is shortened as $\ell$.

First, define a dual variable $\hat{\boldz}$ associated with $\hat{\boldtheta}$ using the following equality:
	\begin{align}
	\label{eq.duality}
	\nabla \ell(\thetahat) + \lambda_{n_p} \hat{\boldz} = \boldzero.
	\end{align}
and if $\hat{\boldz}_t$ is the subgradient of $\|\hat{\boldtheta}_t\|$, i.e., $\hat{\boldz}_t \in \nabla_{\boldtheta_t} \|\hat{\boldtheta}_t\|, t\in S \cup S^c$, \eqref{eq.duality} is the optimality condition of \eqref{eq.obj.alter} and $\hat{\boldtheta}$ is an optimal solution to \eqref{eq.obj.alter}. Moreover, the next Lemma tells the relationship between dual variable $\hat{\boldz}$ and sparsity patterns of any other optimal solution of \eqref{eq.obj.alter}.
\begin{lem}
	\label{lem.1.maintext}
	If there exists an optimal $\hat{\boldtheta}$ of \eqref{eq.obj.alter} with associated $\hat{\boldz}$ in \eqref{eq.duality} such that $\|\hat{\boldz}_{t''}\| < 1$, for all $t''\in S^c$. Then any optimal $\tilde{\boldtheta}$ of \eqref{eq.obj.alter} should have $\tilde{\boldtheta}_{t''} = \boldzero$ for all $t'' \in S^c$.
\end{lem}
See Appendix \ref{sec.proof.lemma.1} for the proof.

Now we illustrate the proof procedure of Theorem \ref{them.the.main.theorem}:
\begin{itemize}
	\item Solve the constrained optimization problem
	\begin{align}
	\label{eq.constrained.obj}
	\thetahat_{S} = \argmin_{{\boldtheta}_{S}} \ell\left(
	\begin{bmatrix}
	\boldtheta_S\\
	\boldzero
	\end{bmatrix}
	\right) + \lambda_{n_p} \sum_{t'\in S}\| \boldtheta_{t'} \|;
	\end{align}
	\item For all $t'\in S$, set $\hat{\boldz}_{t'}= \nabla \|\hat{\boldtheta}_{t'}\|$, and let $\thetahat = [\hat{\boldtheta}_{S}, \boldzero]$;
	\item Obtain $\hat{\boldz}_{t''}$ for all $t'' \in S^c$ using equality \eqref{eq.duality};
	\item Show $\max_{t''\in S^c}\|\boldz_{t''}\| < 1$ with high probability under certain conditions. According to Lemma \ref{lem.1.maintext}, we conclude that for any optimal $\tilde{\boldtheta}$ from \eqref{eq.obj.alter}, the correct sparsity pattern is recovered.
\end{itemize}

Bounding $\max_{t'' \in S^c} \|\boldz_{t''}\|$ requires obtaining $\boldz_{t''}$ from \eqref{eq.duality}. More specifically, from \eqref{eq.duality} we have:
\begin{align*}
\nabla \ell(\hat{\boldtheta}) + \lambda_{n_p}\hat{\boldz} = \boldzero 
\Rightarrow \nabla \ell(\hat{\boldtheta}) + \lambda_{n_p}\hat{\boldz} -\nabla \ell(\boldtheta^*)= -\nabla \ell(\boldtheta^*).
\end{align*}
Applying Mean-value Theorem, 
\begin{align}
\label{eq.qrw}
\underbrace{\nabla^2 \ell(\boldtheta^*)}_{\mathcal{I}}[\thetahat - \boldtheta^*]^\top + \lambda_{n_p}\hat{\boldz}& = \underbrace{-\nabla \ell(\boldtheta^*)}_{\boldw} + \underbrace{[\nabla^2 \ell(\boldtheta^*) - \nabla^2 \ell(\bar{\boldtheta})][\thetahat - \boldtheta^*]^\top}_{\boldg},
\end{align}
where $\bar{\boldtheta}$ is between $\boldtheta^*$ and $\hat{\boldtheta}$ in a coordinate fashion.
We can then rewrite \eqref{eq.qrw} in block-wise fashion:
\begin{align}
\label{eq.qwrsc}
\mathcal{I}_{S,S}[\thetahat_S - \boldtheta_S^*] + \lambda_{n_p}\zhat_{S}& = \boldw_{S} + \boldg_{S} \notag\\
\mathcal{I}_{t'', S}[\thetahat_S - \boldtheta_S^*] + \lambda_{n_p}\zhat_{t''}& = \boldw_{t''} + \boldg_{t''}, ~~ t'' \in S^c.
\end{align}
Substitute $\thetahat_S - \boldtheta^*_S = \mathcal{I}_{SS}^{-1}[\boldw_{S} + \boldg_{S} - \lambda_{n_p}\zhat_{S}]$ into \eqref{eq.qwrsc}, we have
\begin{align*}
\mathcal{I}_{t''S}\mathcal{I}_{SS}^{-1}[\boldw_{S} + \boldg_{S} - \lambda_{n_p}\zhat_{S}] + \lambda_{n_p}\zhat_{t''}& = \boldw_{t''} + \boldg_{t''}.
\end{align*}
Rearrange terms, we have
\begin{align*}
\lambda_{n_p}\zhat_{t''}& = \boldw_{t''} + \boldg_{t''} - \mathcal{I}_{t''S}\mathcal{I}_{SS}^{-1}[\boldw_{S} + \boldg_{S} - \lambda_{n_p}\zhat_{S}].
\end{align*}
According to triangle inequality,
\begin{align*}
\lambda_{n_p}\max_{t''\in S^c}\|\zhat_{t''}\| \le &\max_{t'' \in S^c} \|\boldw_{t''}\| + \max_{t'' \in S^c} \|\boldg_{t''}\| + \\
&\max_{t'' \in S^c} \|\mathcal{I}_{t''S}\mathcal{I}_{SS}^{-1}\|_1\left(\max_{t'\in S}\|\boldw_{t'}\|
 + \max_{t'\in S}\|\boldg_{t'}\| + \lambda_{n_p}\right).
\end{align*}
By assumption, $\max_{t''\in S^c} \|\mathcal{I}_{t''S}\mathcal{I}_{SS}^{-1}\|_1 \le (1- \alpha)$, and we obtain
\begin{align*}
\max_{t''\in S^c}\|\zhat_{t''}\| & \le \frac{(2-\alpha)}{\lambda_{n_p}}\left(\max_{t\in S\cup S^c}\|\boldw_t\|+\max_{t\in S \cup S^c}\|\boldg_t\| \right) + (1-\alpha).
\end{align*}
Now we need to show the boundedness of $\boldw$ and $\boldg$.

The boundedness of $\boldw$, which is the gradient of log-likelihood function on $\boldtheta^*$, is guaranteed by the following lemma:
\begin{lem}
	\label{lem.2.maintext}
	There exist constants $c = \lambda_{\text{max}}b+2,c'$ and $c''$ and if the regularization parameter $\lambda_{n_p}$ satisfies
	\[
	\frac{8 (2-\alpha) }{\alpha} \sqrt\frac{{c\log \frac{m^2+m}{2}}}{n_p} \le \lambda_{n_p} \le \frac{4 (2-\alpha)c}{\alpha} \min\left(\frac{\|\boldtheta^*\|}{\sqrt{b}}, 1\right),
	\]
	then 
	\begin{align}
	\label{eq.lemma2.prob}
	& P\left( \max_{t\in S^c \cup S} \|\boldw_t\| \geq 
	\frac{\alpha}{4(2-\alpha)}\lambda_{n_p}  \right) \leq \notag\\
	&\exp\left( -c'n_p \right) + 4\exp\left(- c''n_q \left(\frac{ \log \frac{m^2+m}{2}}{n_p}\right)^2 +b \log (m^2+m)\right) .
	\end{align}
\end{lem}
Quite different from similar lemmas in previous works (such as Lemma 2 in \citep{Ravikumar_2010}), Lemma \ref{lem.2.maintext} is not a simple concentration of sample mean converging to its population mean, since the gradient of the likelihood contains two sets of data from different distributions $P$ and $Q$. Moreover, the sub-Gaussianity or the boundedness of the $\boldpsi$ of $P$ or $Q$ were not assumed, so the concentration inequality cannot be applied here. Instead, via the smoothness behavior of the likelihood ratio function (Assumption \ref{assum.smooth}), we are able to derive such a boundedness of $\boldw$ without using any explicit properties of two distributions.

The consequence of such differences is important: this analysis allows us to consider a very wide range of distributions which may not be well-behaved (e.g. heavy-tailed), as long as the change between two distributions are minor. After all, all assumptions are imposed on the density ratio model $r$ only, rather than $P$ or $Q$. 
This analysis preserves the flexibility of the density ratio estimation methodology.
The proof of Lemma \ref{lem.2.maintext} can be found in Appendix \ref{sec.proof.lemma.2}.

The next lemma bounds the difference between the estimated parameter and the true parameter over the non-sparse indices, which is further used to bound $\boldg$ and derive the sample complexity.
\begin{lem}
	\label{lem.3.maintext}
	If 	$d\lambda_{n_p} \le \frac{\lambda_\mathrm{min}^2}{20\lambda_{3,\mathrm{max}}}$ and $\max_{t\in S^c \cup S} \|\boldw_t\| \le \frac{\lambda_{n_p}}{4}$ then $\|\boldtheta_S^* - \hat{\boldtheta}_S\| \le \frac{10}{\lambda_\mathrm{min}} \sqrt{d}\lambda_{n_p}$,
\end{lem}
The boundedness of $\boldg$ is finally given by 
\begin{lem}
	\label{lem.4.maintext}
	If $	\lambda_{n_p} d \le \frac{\lambda_\mathrm{min}^2}{100\lambda_{3,\mathrm{max}}}\frac{\alpha}{4(2-\alpha)},$ and $\max_{t \in S \cup S^c} \|\boldw_t\|\le \frac{\lambda_{n_p}}{4}$, then $\max_{t \in S \cup S^c}\|\boldg_t\| \leq \frac{\alpha\lambda_{n_p} }{4(2-\alpha)}$.
\end{lem}
See Appendix \ref{sec.proof.lemma.3} and \ref{sec.proof.lemma.4} for proofs. Using Lemma \ref{lem.2.maintext}, \ref{lem.3.maintext} and \ref{lem.4.maintext},  we have
$
\max_{t''\in S^c}\|\zhat_{t''}\| \le 1-\frac{\alpha}{2} < 1.
$

To show the correct non-zero pattern recovery, it suffices to show \[\max_{t\in S^c \cup S} \|\hat{\boldtheta}_t-\boldtheta_t^*\| < \frac{1}{2}\min_{t\in S} \|\boldtheta_t^*\|.\] Since Lemma 3 shows 
$\max_{t\in S\cup S^c} \|\hat{\boldtheta}_t-\boldtheta_t^*\|\leq \|\hat{\boldtheta} - \boldtheta^*\| < \frac{10}{\lambda_\text{min}} \sqrt{d}\lambda_{n_p},
$
we just need $\min_{t\in S} \|\boldtheta_t^*\| > \frac{20}{\lambda_\text{min}} \sqrt{d}\lambda_{n_p}$ to ensure such recovery. 

\subsection{Sample Complexity}
\label{sec.sample.complexity}
The sample complexity for $n_p$ and $n_q$ are derived from the conditions of the Lemmas. 
To make Lemma \ref{lem.2.maintext} holds, we may set 
$
\lambda_{n_p} = C\sqrt\frac{{\log \frac{m^2+m}{2}}}{n_p}, 
$
where $C$ is chosen so that the lower bound of $\lambda_{n_p}$ in the Lemma \ref{lem.2.maintext} is satisfied. Since the upper-bound is a constant while such setting of $\lambda_{n_p}$ is always decaying as $n_p$ grows, it is automatically satisfied at some point. 

Moreover, $\lambda_{n_p}$ should also satisfy the upper-bound condition in Lemma \ref{lem.4.maintext}:
$
		\lambda_{n_p} d \le \frac{\lambda_\mathrm{min}^2}{100\lambda_{3,\mathrm{max}}}\frac{\alpha}{4(2-\alpha)}, 
$ 
and this inequality can be satisfied when $n_p = \Omega(d^2\log \frac{m^2+m}{2})$.

The upper-bound of $\lambda_{n_p}$ is in Lemma \ref{lem.4.maintext} is tighter than it is in Lemma \ref{lem.3.maintext}, so the condition of Lemma \ref{lem.3.maintext} is automatically satisfied. However, one still needs to make sure that the tail probability term that involves $n_q$ in Lemma \ref{lem.2.maintext} decays, i.e.,
\begin{align*}
	4\exp\left(- c''n_q \left(\frac{ \log \frac{m^2+m}{2}}{n_p}\right)^2 +b \log (m^2+m)\right) \rightarrow 0, n_p \rightarrow \infty.
\end{align*}
This can be guaranteed by setting $n_q = \Omega(n_p^2 g(m))$.

%

\section{Analysis of Assumptions}
\label{sec.analysis.assums}


In this section, we investigate the conditions under which the maximum (minimum) eigenvalues of likelihood ratio derivatives are bounded. We show that under mild regularity conditions of sample statistics of distribution $Q$ and tightened smoothness conditions of the density ratio model, Assumption \ref{assum.depen} and \ref{assum.smooth} holds automatically. 

\subsection{Bounded Density Ratio Model}
Since the derivatives of the log-likelihood ratio expresses the curvature of our objective function, we expect the smoothness of the density ratio model $r(\boldx;\boldtheta)$ may play an important role in such analysis.
To begin with, consider a simple bounded-ratio model by replacing the smoothness Assumption \ref{assum.smooth.ratiomodel.nod} (or \ref{assum.smooth.ratiomodel}) with a tightened Assumption \ref{assum.bounded.ratio}:
\begin{assum}[Smooth Density Ratio Model Assumption]
	\label{assum.bounded.ratio}
	For any vector $\bolddelta \in \mathbb{R}^{\text{dim}(\boldtheta^*)}$ such that $\|\bolddelta\|\leq \|\boldtheta^*\|$, the following inequality holds:
	\begin{align*}
	0 < C_\mathrm{min} \le r(\boldx, \boldtheta^* + \bolddelta) \le C_\mathrm{max} < \infty.
	\end{align*}
	As consequences, $\frac{1}{C_\mathrm{ratio}}\le \hat{r}(\boldx;\boldtheta^*+\bolddelta) \le C_\mathrm{ratio}$ and $\|\boldf_t(\boldx)\| \le C_{\boldf_t,\mathrm{max}}$, where  $C_{\boldf_t,\mathrm{max}}$ and $C_\mathrm{ratio}$ are all constants.
\end{assum}
%

Since Assumption \ref{assum.bounded.ratio} is stronger than Assumption \ref{assum.smooth.ratiomodel.nod} or \ref{assum.smooth.ratiomodel} (for appropriately chosen $C_\mathrm{max}$ and $C_\mathrm{min}$), the proof of Theorem \ref{them.the.main.theorem} still holds if one uses above assumptions to substitute Assumption \ref{assum.smooth.ratiomodel.nod} or \ref{assum.smooth.ratiomodel}. However, as we will demonstrate later, an improved sample complexity for $\min(n_p, n_q)$ can be derived. 

As Assumption \ref{assum.depen} and \ref{assum.smooth} are constructed using samples from distribution $Q$, it is also natural to assume some basic regularity conditions on sample statistics:

\begin{assum}[Bounded Moments]
	\label{assum.bounded.moments.q}
	The feature transform $\boldf(\boldx) \in \mathbb{R}^{b(m^2+m)/2}$, where $\boldx$ is drawn from $Q$, has upper-bounded moments with probability one: i.e.
	\begin{align*}
		\max_{t\in S \cup S^c}\hat{\mathbb{E}}_q \left[\|\boldf_t(\boldx)\|\right] \le D_\mathrm{max, 1} < \infty, \\
		\left\|\hat{\mathbb{E}}_q\left[\boldf(\boldx)\boldf(\boldx)^\top\right] \right\| \text{ and } ~ \left\|\widehat{\mathrm{Cov}} _q\left[\boldf(\boldx)\right]\right\| \le D_\mathrm{max,2} < \infty\\
	\end{align*}
	and 
	\begin{align}
		\label{eq.lambda2.lower}
		\Lambda_\mathrm{min}\left\{\widehat{\mathrm{Cov}}_q\left[\boldf_S(\boldx)\right]\right\} &\ge D_\mathrm{min,2} > 0,
	\end{align}
	is bounded with probability $1- \delta_{n_q}$.
\end{assum}
$\widehat{\mathrm{Cov}}$ is the sample covariance estimator and $\hat{\mathbb{E}}_q\left[ \boldg(\boldx) \right] = \frac{1}{n_q} \sum_{i=1}^{n_q} \boldg(\boldx^{(i)})$ is the empirical expectation over samples drawn from $Q$. From now on, we remove subscript $p$ or $q$ from a random sample $\boldx$ when summing up, as long as the indices give enough context for telling in which distribution the sample is drawn. For example, $\sum_{i=1}^{n_q} \boldx^{(i)}$ is a summation over samples drawn from distribution $Q$.

Of course, one may impose similar bounded moments constrains on the corresponding population quantities, then the above assumption automatically holds with high probability under certain regularity conditions. To avoid lengthy proofs, we stick to assumptions using sample quantities in this paper.

The following propositions show Assumptions \ref{assum.bounded.ratio} and \ref{assum.bounded.moments.q} guarantee the boundedness of the derivatives of the likelihood ratio function.

\begin{prop}[Bounded Hessian]
	\label{prop.bounded.2nd}
	For any vector $\bolddelta \in \mathbb{R}^{\text{dim}(\boldtheta^*)}$ such that $\|\bolddelta\|\leq \|\boldtheta^*\|$, if Assumptions \ref{assum.bounded.ratio} and \ref{assum.bounded.moments.q} hold then $\|\nabla_{\boldtheta}^2 \ell(\boldtheta^* + \bolddelta)\| \le 2C_{\mathrm{ratio}} D_{\mathrm{max,2}}$ with probability $1-\delta_{n_q}$.
\end{prop}

\begin{prop}[Bounded 3rd-order Derivative]
	\label{prop.bounded.3rd}
	For any vector $\bolddelta \in \mathbb{R}^{\text{dim}(\boldtheta^*)}$ such that $\|\bolddelta\|\leq \|\boldtheta^*\|$, if Assumption \ref{assum.bounded.ratio} and \ref{assum.bounded.moments.q} holds, then
	\[
    \max_{t\in S \cup S^c}\vertiii{\nabla_{\boldtheta_t}\nabla^2 \ell(\boldtheta^* + \bolddelta)} \le
	6C^2_{\mathrm{ratio}} D_\mathrm{max,1}D_\mathrm{max,2}
	\]
	with probability $1-\delta_{n_q}$.
\end{prop}

\begin{prop}[Eigenvalue Lower Bound of Invertible Hessian Submatrix]
	\label{prop.bounded.min.2nd}
	For any vector $\bolddelta \in \mathbb{R}^{\text{dim}(\boldtheta^*)}$ such that $\|\bolddelta\|\leq \|\boldtheta^*\|$, if Assumption \ref{assum.bounded.ratio} and \ref{assum.bounded.moments.q} holds, then
	\begin{align*}
	\Lambda_\mathrm{min}\left[\nabla^2_{\boldtheta_S} \ell(\boldtheta^* + \bolddelta)\right]
	\ge \frac{D_\mathrm{min,2}}{C^2_\mathrm{ratio}}
	\end{align*} 
	with probability $1-\delta_{n_q}$.
\end{prop}
Proofs of the above propositions are listed in \ref{sec.proof.prop.2nd}, \ref{sec.proof.prop.3nd} and \ref{sec.proof.prop.min.2nd} of Appendix.

\subsection{Sufficient Conditions under Bounded Density Ratio}
Now, we give a variation of Theorem \ref{them.the.main.theorem} based on the totally bounded density ratio model. 
Consider the objective
\begin{align*}
	\label{eq.obj.alter.npnq}
	\thetahat = \argmin_{\boldtheta} \ell(\boldtheta) &+ \lambda_{n_p, n_q} \sum_{t'\in S} \| \boldtheta_{t'} \| + \lambda_{n_p,n_q} \sum_{t''\in S^c}\|\boldtheta_{t''}\|,
\end{align*}
which is identical to \eqref{eq.obj.alter} but the regularization parameter is now determined with respect to both $n_q$ and $n_q$.
\begin{corol}
	\label{corol.bounded.ratio}
	Suppose that Assumptions  \ref{assum.incoherence}, \ref{assum.correct}, \ref{assum.bounded.ratio} and \ref{assum.bounded.moments.q} as well as $\min_{t\in S} \|\boldtheta^*_t\| \geq \frac{10}{\lambda_\text{min}} \sqrt{d}\lambda_{n_p, n_q}$ are satisfied.
	Suppose also that the regularization parameter is chosen so that
	\begin{align*}
	\frac{24(2-\alpha)}{\alpha} \sqrt\frac{{M \log \frac{m^2+m}{2}}}{\min(n_p,n_q)} &\le \lambda_{n_p,n_q}, 
	\end{align*}
	where $M$ is a positive constant.
	Then there exist some constants $L$ and $K$ such that if $\min(n_p,n_q)\geq L d^2\log  \frac{m^2+m}{2}$, with the probability at least $1- 4\exp\left( - K \lambda_{n_p, n_q}^2 \min(n_p, n_q) \right)  - \delta_{n_q}$, KLIEP has the same properties as those stated in Theorem \ref{them.the.main.theorem}.
\end{corol}
The proof of this corollary is done by replacing Lemma \ref{lem.2.maintext} with Lemma \ref{lemma5.maintext}:  
	\begin{lem}
		\label{lemma5.maintext}
		If $\lambda_{n_p, n_q} \ge \frac{24(2-\alpha)}{\alpha}\cdot\sqrt{\frac{c\log \frac{(m^2+m)}{2}}{\min(n_p,n_q)}}$, then 
		\[	P\left(\max_{t \in S\cup S^c}\|\boldw_t\| \ge \frac{\alpha \lambda_{n_p, n_q}}{4(2-\alpha)} \right)\le  4\exp\left(-c''\min(n_p,n_q) \right),\]
		where $c$ and $c''$ are some constants.
	\end{lem}
See Appendix \ref{sec.proof.lemma.5} for the proof.
Note that we have ditched Assumption \ref{assum.depen} and \ref{assum.smooth} since we have already shown that they are automatically satisfied with probability 1-$\delta_{n_q}$ given Assumption \ref{assum.bounded.ratio} and \ref{assum.bounded.moments.q}.

\subsection{Smoothness Assumption Relaxed}
\label{sec.relaxed.assumptions}
In the previous derivation, the assumption of boundedness of density ratio model guarantees its empirical counterpart $\hat{r}(\boldx;\boldtheta)$ is always upper-bounded by $C_\mathrm{ratio}$ and is lower bounded by $\frac{1}{C_\mathrm{ratio}}$, but this was somewhat restrictive. 
In this section, we discuss a relaxation of Assumption \ref{assum.bounded.ratio} as follows:
\begin{assum}[Smooth Density Ratio Model Assumption]
	\label{assum.bounded.ratio.3}
	For any vector $\bolddelta \in \mathbb{R}^{\text{dim}(\boldtheta^*)}$ such that $\|\bolddelta\|\leq \|\boldtheta^*\|$, the following inequality holds:
	\begin{align}
	0 < r(\boldx, \boldtheta^* + \bolddelta) \le C_\mathrm{max}&, \notag
	\\
	\label{eq.assum.bounded.expectation}
	\mathbb{E}_q \left[\inf_{\bolddelta \in \mathbb{R}^{\mathrm{dim}(\boldtheta^*)}: \|\bolddelta\|
		\leq \|\boldtheta^*\|}   r(\boldx, \boldtheta^* + \bolddelta) \right]  &\ge 1 - c, 0< c < 1,
	\end{align}
	where $c$ is a constant.
\end{assum}
Now the strictly positive lower-bound of the density ratio model is removed, and we add a new uniform lower-bound on the expectation of density ratio model around the true model. Such condition allows us to control the tail of the empirical density ratio model  so that the ratio model of a specific sample $\boldx$ from $Q$, i.e., $\inf_{\bolddelta \in \mathbb{R}^{\mathrm{dim}(\boldtheta^*)}: \|\bolddelta\|
	\leq \|\boldtheta^*\|} \hat{r}(\boldx;\boldtheta+\bolddelta)$, would not deviate ``too much'' from $1$.

Note that since smoothness Assumption \ref{assum.bounded.ratio.3} is still stronger than \ref{assum.smooth.ratiomodel.nod} or \ref{assum.smooth.ratiomodel}, the proof of Theorem \ref{them.the.main.theorem} can be used without modification to show the support consistency when this assumption is substituted. 
However, the proof of Corollary \ref{corol.bounded.ratio} cannot be used when Assumption \ref{assum.bounded.ratio.3} is imposed, since the proof requires the boundedness of $\|\boldf_t\|$ which is not implied by this assumption.
From now on, we show that such relaxed regime together with moment-bounding Assumption \ref{assum.bounded.moments.q} also allow us to bound eigenvalues of derivatives of the likelihood function.


First, we give an example showing that if the expectation of the density ratio derivative is bounded over $Q$, the above assumption holds. 
\begin{prop}
\label{prop.bounded.min}
	If
	$
		\|\boldtheta^*\|\mathbb{E}_{q}\left[ \sup_{\bolddelta \in \mathbb{R}^{\mathrm{dim}(\boldtheta^*)}: \|\bolddelta\|
			\leq \|\boldtheta^*\|}\|\nabla r(x, \boldtheta^*+ \bolddelta)\|\right] \le c, 0<c<1
	$
	then \eqref{eq.assum.bounded.expectation} holds. 
\end{prop}

The proof is listed in \ref{sec.proof.bounded.min} of Appendix. This proposition intuitively shows that as long as the density ratio model is smooth in the first-order, and the changes in parameter is not too drastic, our new assumption holds. 

\paragraph{Bounding $\|\nabla_{\boldtheta}^2 \ell(\boldtheta)\|$ and $\vertiii{\nabla_{\boldtheta_t}\nabla^2 \ell(\boldtheta)}$} 
It can be seen from Propositions \ref{prop.bounded.2nd} and \ref{prop.bounded.3rd} that upper-bounding the second or the third order derivative relies on the upper-bound of the empirical density ratio model $\hat{r} \le C_\mathrm{ratio}$. However, under this new assumption, the empirical density ratio model is no longer explicitly bounded. Now we derive the upper-boundedness of $\hat{r}$ using the new assumption.
\begin{prop}[Uniformly Upper-bounded $\hat{r}$]
	\label{prop.bounded.ratio.model.uniform}
If Assumption \ref{assum.bounded.ratio.3} holds
	\begin{align*}
		\sup_{\bolddelta \in \mathbb{R}^{\mathrm{dim}(\boldtheta^*)}: \|\bolddelta\|
				\leq \|\boldtheta^*\|} \hat{r}(\boldx;\boldtheta^*+\bolddelta) \le C'_\mathrm{ratio} < \infty, 
	\end{align*}
	holds with probability at least $1-\exp\left(-\frac{2n_q\epsilon^2}{C_\mathrm{max}^2}\right)$.
\end{prop}

The proof is in \ref{sec.proof.model.uniform} in Appendix. Therefore, replace $C_\mathrm{ratio}$ in \eqref{eq.laplacian.trick}, \eqref{eq.bounding.3rdorder2} and \eqref{eq.d3.nodi} with $C'_\mathrm{ratio}$, and we have the Assumption \ref{assum.smooth} holds with high probability. 
%


\paragraph{Bounding Minimum Eigenvalue of $\nabla_{\boldtheta_S}^2 \ell(\boldtheta^*)$}
Under Assumption \ref{assum.bounded.ratio.3}, the lower-bound of empirical density ratio model $\min_j \hat{r}(\boldx^{(j)};\boldtheta^*)$ is no longer valid since the density ratio can approach to 0. Here we illustrate another proof showing the boundedness of the minimum eigenvalue using concentration inequalities. 
	\begin{prop}
		\label{prop.bounded.normalization.ratio}
		If Assumption \ref{assum.bounded.ratio.3} holds, then
		\begin{align}
		\label{eq.bounded.ratio.model.over.q}
		1-\epsilon \le \frac{1}{n_q}\sum_{i=1}^{n_q} r(\boldx^{(i)};\boldtheta^*) = \frac{\hat{N}(\boldtheta^*)}{N(\boldtheta^*)} \le 1+\epsilon
		\end{align}
		holds with probability at least $1-2\exp \left( -\frac{2n_q \epsilon^2}{C_\mathrm{max}^2}\right)$.
	\end{prop}

	\begin{prop}
	\label{prop.bounded.min.lambda}
		If Assumption \ref{assum.bounded.ratio.3} holds and the importance-sampled covariance using \textbf{true} density ratio model $r(\boldx;\boldtheta^*)$ satisfies
		\begin{align*}
			\Lambda_{\mathrm{min}}\left(\widehat{\mathrm{Cov}}_{qr_{\boldtheta^*}}\left[\boldf_S(\boldx)\right]\right) \ge D'_\mathrm{min,2},
		\end{align*}
		with probability $\delta_{n_q}$,
		then
$
			\Lambda_{\mathrm{min}}\left(\nabla^2_{\boldtheta_S} \ell(\boldtheta^*)\right)\ge D'_\mathrm{min,2}/(1+\epsilon)^2, \forall \epsilon < \infty
$ 
holds with probability at least $1-2\exp \left( -\frac{2n_q \epsilon^2}{C_\mathrm{max}^2}\right) - \delta_{n_q}$.
	\end{prop}

	Proofs are listed in \ref{sec.prop.normalization.ratio} and \ref{sec.prop.min.lambda}, Appendix. In fact, following the derivations used in above proofs, we can re-write derivatives of $\ell(\boldtheta)$ as higher-order sample statistics importance-sampled by samples from $Q$ with the empirical density ratio model $\hat{r}(\boldx;\boldtheta^*)$. See Proposition \ref{prop.is.moments} in Appendix for a precise statement.
	
\subsection{Bounded Density Ratio Assumption: How Strong Is It?}
\label{sec.bounded.ratio.discussion}
In this section, we have considered a few stronger alternative assumptions to Assumptions \ref{assum.smooth.ratiomodel.nod} and \ref{assum.smooth.ratiomodel} in order to derive the boundedness of derivatives of the likelihood function, which are crucial to the proof of Theorem \ref{them.the.main.theorem}. However, it is natural to ask, how strong these assumptions are?

The main advantage of the density ratio based change detection described in \cite{liu2014ChangeDetection} is that such a method does not limit itself to certain distributions. Therefore, limiting the differences between two distributions help us avoid making assumptions on individual MNs.

In fact, the totally bounded density ratio assumption (Assumption \ref{assum.bounded.ratio}) is very well justified through our ``interest'': learning the changes between patterns (MNs). The power of  density ratio is the magnitude of the changes between two density functions. If the change itself is ``insanely'' big, such a change detection task would not make any sense in the first place. Unfortunately, such a restriction will rule out some common distributions for change detection, such as Gaussian-distribution whose density ratio value is 
not necessarily upper-bounded. Nonetheless, it does not forbid us to consider \emph{truncated Gaussian distributions} where we focus on a ``confined area'' as our interested region of learning changes.

To loosen this restriction, we utilize another fact that density ratio, like density functions, are naturally lower-bounded by 0. 
Therefore, Section \ref{sec.relaxed.assumptions} is dedicated to the case where the density ratio can decay unbounded toward 0 and thus it allows us to consider the sufficient statistics $\boldf_t$ with unbounded $\ell_2$ norm. 
Illustrative figures of the applicability of our smoothness assumptions are given in Section \ref{fig.applicatility.smooth.ass}, Appendix.

\section{Synthetic Experiments}
\label{sec.exp}
In this section, we validate our theorem and compare KLIEP with a state of the art method on synthetic datasets. For a practical usage of KLIEP, see Section \ref{sec.gene} for details. The MATLAB code skeleton that is used for our experiments can be found at \url{http://www.ism.ac.jp/~liu/software.html}.

If all the \emph{sufficient conditions} in Theorem \ref{them.the.main.theorem} are satisfied, the solution of our optimization problem in \eqref{eq.obj.alter} should successfully recover the sparsity pattern in $\boldtheta^*$ with high probability. Therefore, we can validate our theorem by examining the probability of successful detection of changed edges, i.e., the proportion of the simulation where the method exactly recovers the support of the changed edges. 
We set the regularization parameter as a scaling variable: $\lambda_{n_p} = C\sqrt{\frac{\log m}{n_p}}$, where $C$ is a chosen constant, so the right side inequality of \eqref{eq.lambda.condition} may be satisfied at some point as $n_p$ grows. As $\log \frac{m^2+m}{2}$ is upper-bounded by $2\log m$ if $m>1$, the left side of \eqref{eq.lambda.condition} is also satisfied if $C$ is appropriately chosen.
Note that this is \emph{not} how the hyper-parameter is chosen in practice.

Now using the same reasoning illustrated in Section \ref{sec.sample.complexity}, we can deduce that 
when fixing $d$, the number of samples $n_p$ required for detecting the correct sparse changes grows linearly with $\log m$, so the success rate versus $n_p / \log m$ plot should align well for MNs with different number of nodes (dimensions) $m$.

Moreover, our theorem does not have ``a preference'' on any specific graph structure (such as trees or stars), nor the connectivity of each \emph{individual} MN. Therefore, as long as the number of changed edges $d$ is the same, the success rate plot should have similar behaviours for MNs with different structures. This is a unique feature of \emph{direct change detection} comparing to methods involving learning two separate MNs. See \ref{sec.exp.settings} in Appendix for detailed experimental settings.
\begin{figure*}[t]
	\subfigure[4-neighbour Lattice, Gaussian, without scaling with $\log m$. ]{
		\label{fig.exp.lattice}\includegraphics[width=.45\textwidth]{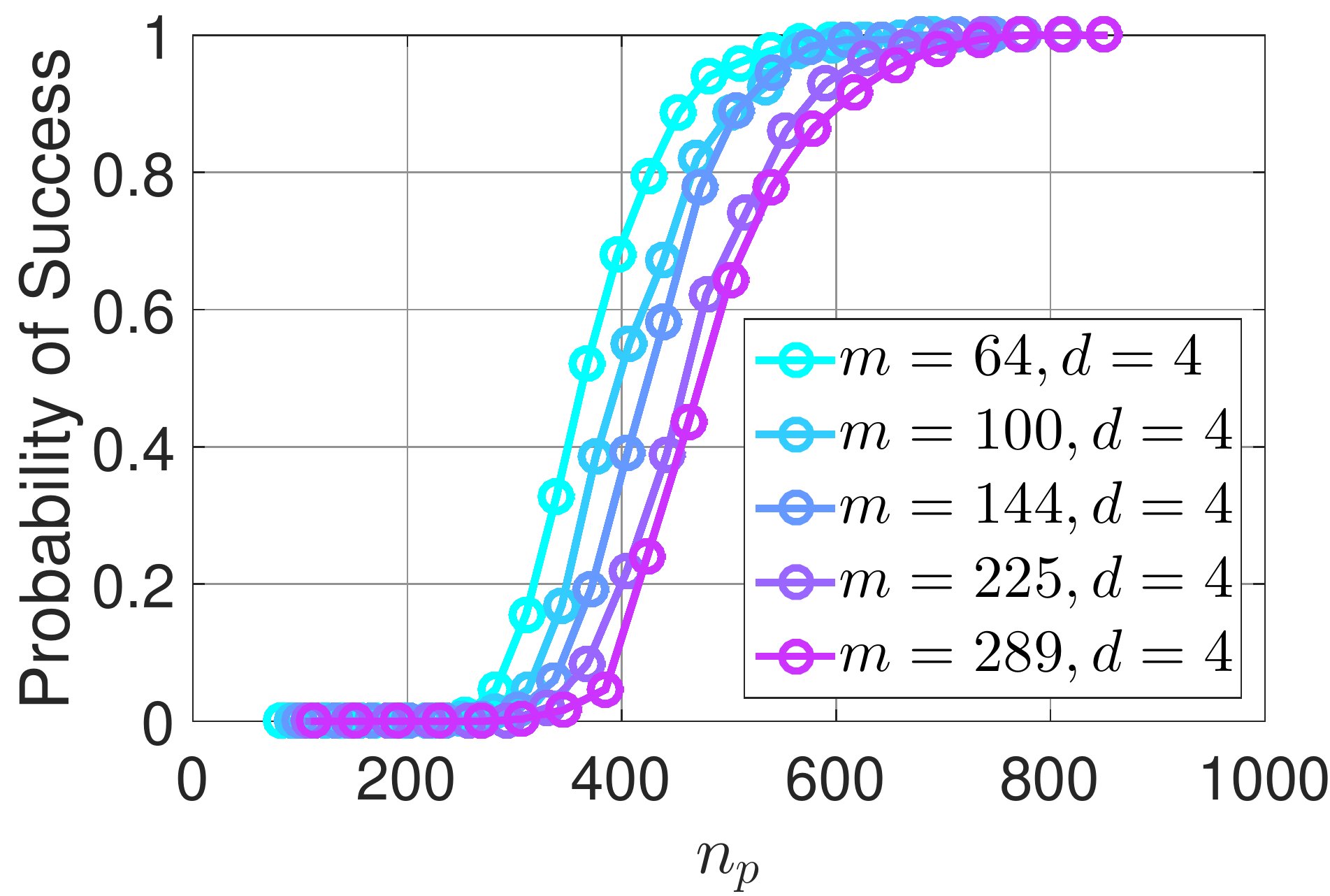}}
	\subfigure[Lattice, Gaussian, rescaled by $\log m.$]{\includegraphics[width=.45\textwidth]{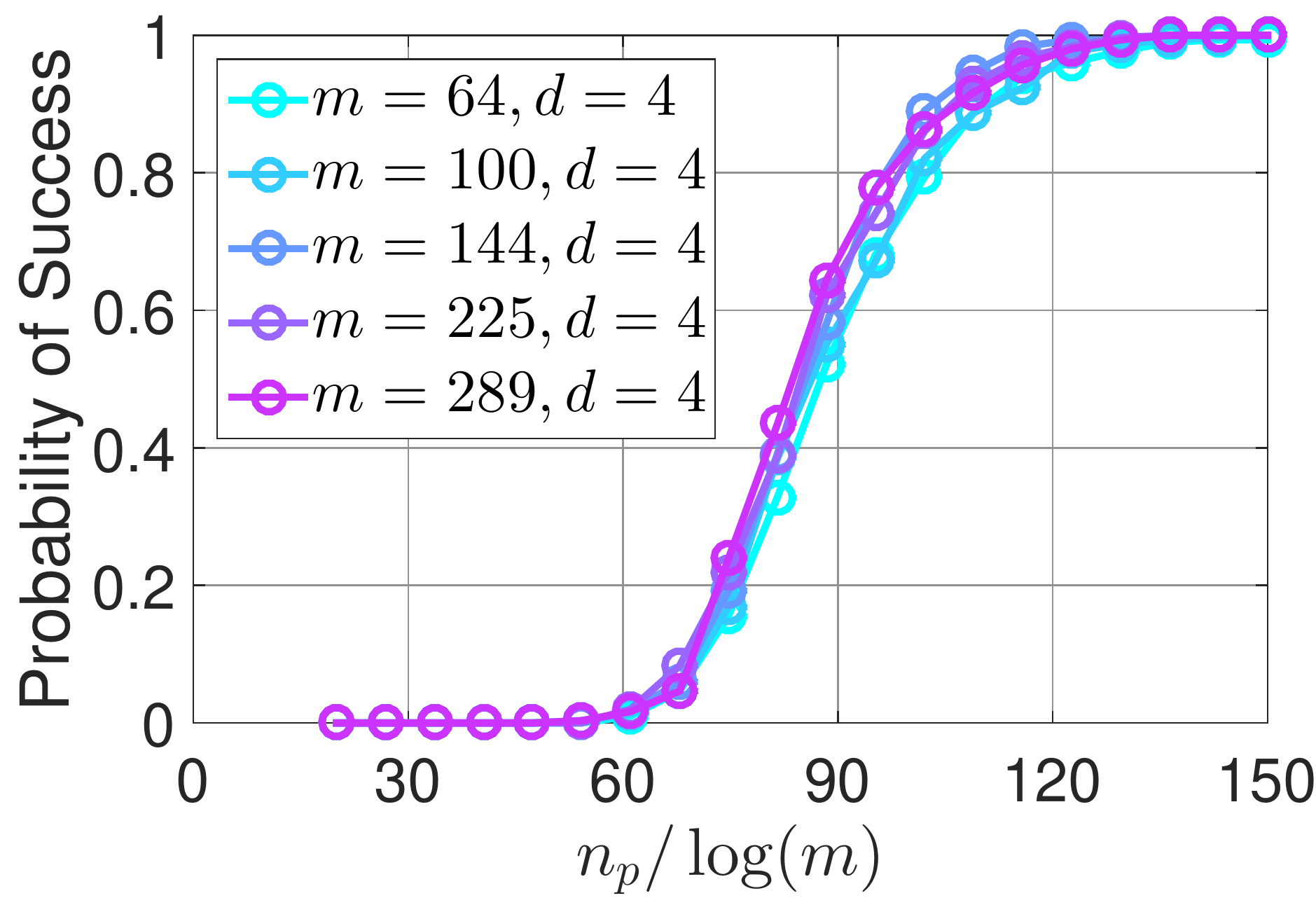}}
	\subfigure[Random structure ($\sim 5$\% connectivity), Gaussian, without scaling with $\log m$.]{\includegraphics[width=.45\textwidth]{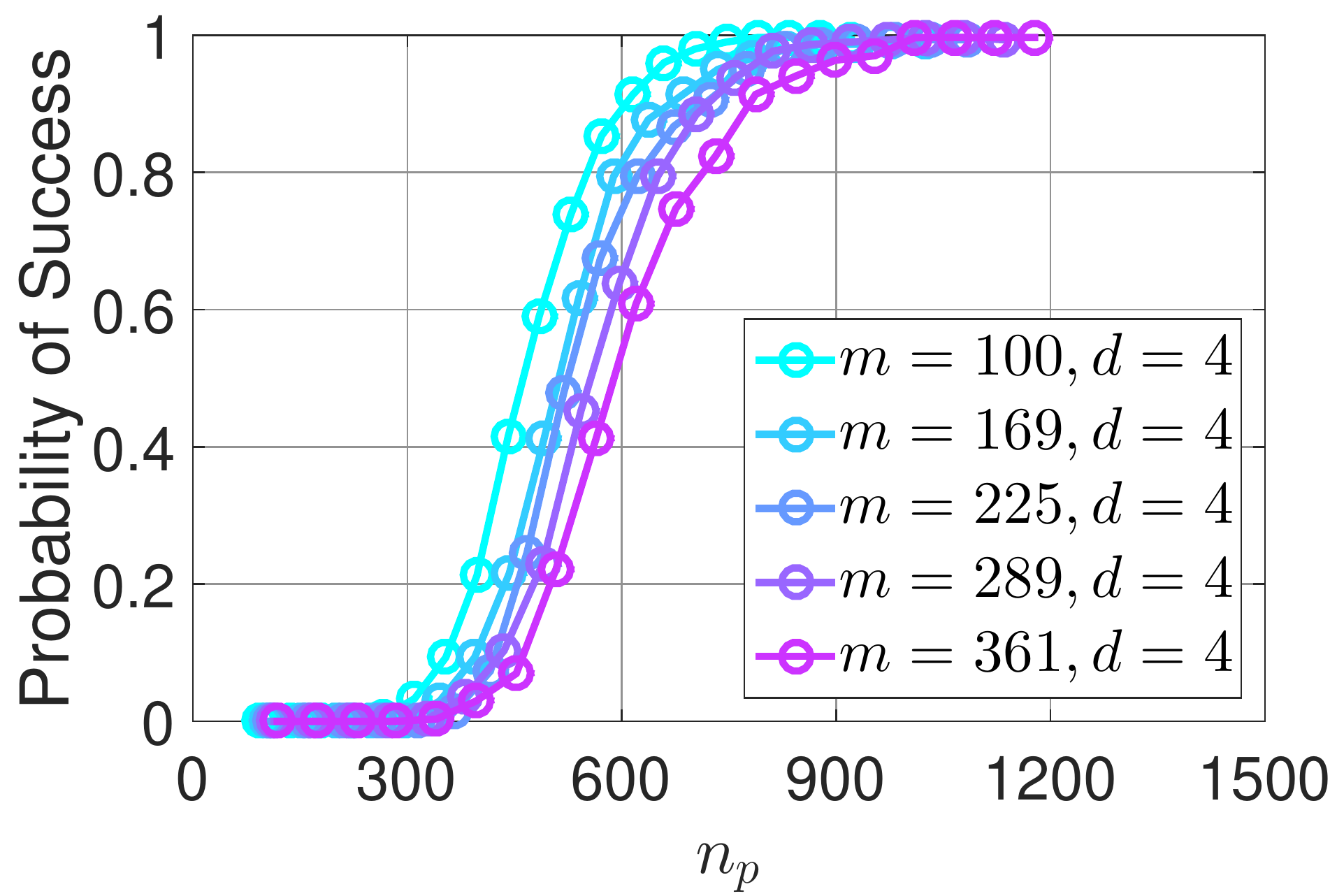}}
	\subfigure[Random structure, Gaussian,  rescaled by $\log m$.]{\includegraphics[width=.45\textwidth]{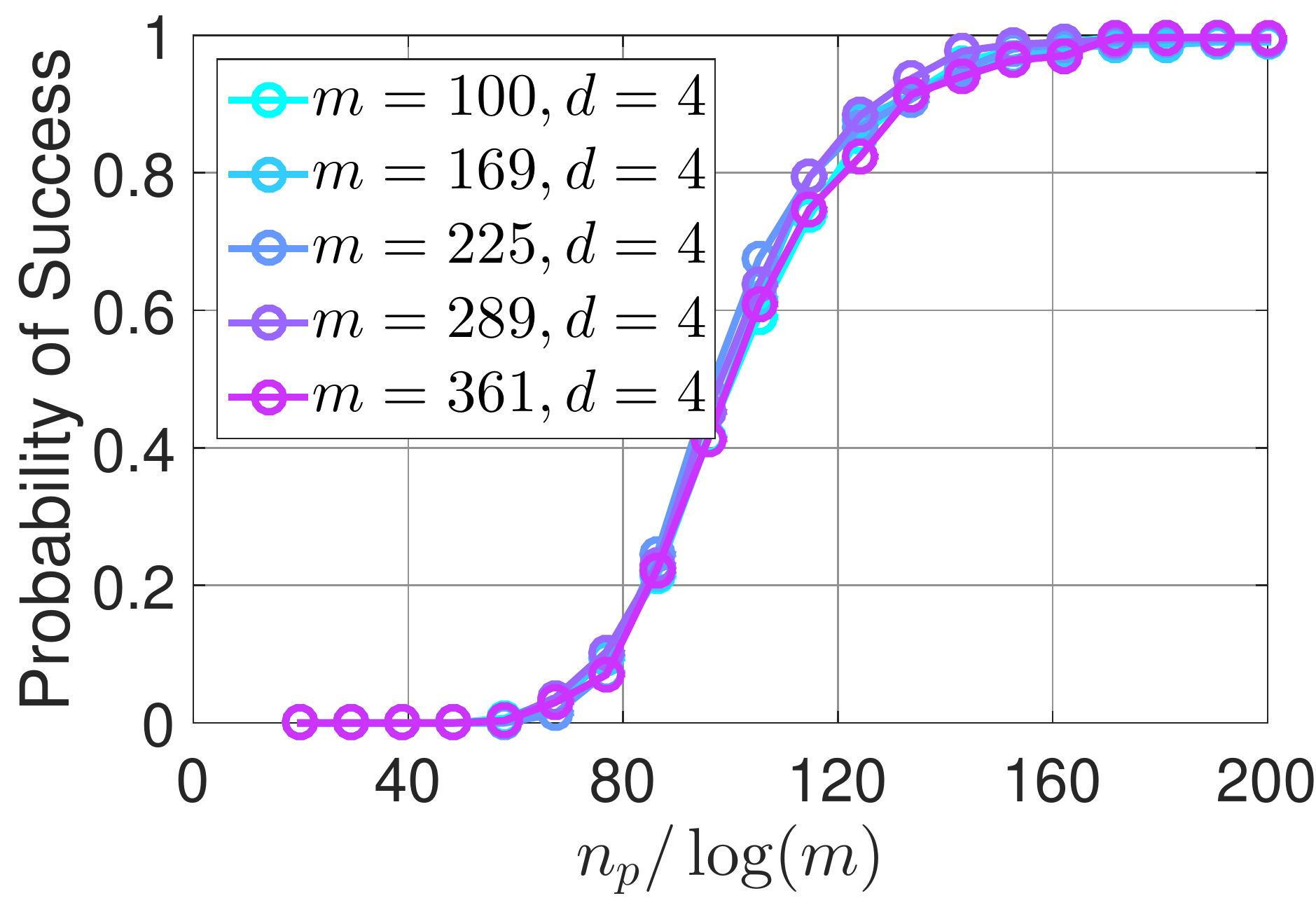}}
	\caption{The relationship between $n_p$ and $\log m$, while $n_q=1000$ is kept fixed. Success rates are computed over 300 runs, same below.}
	\label{fig.q.1000}
\end{figure*}
\subsection{$n_p$ versus $\log m$}
\label{sec.exp.gaussian}
We now illustrate above effects via experiments. Since the density ratio estimator involves two sets of data with size $n_p$ and $n_q$, to avoid complication, we first set $n_q$ to a sufficiently large value ($n_q = 1000$), and examine the relationship between $m$ and $n_p$ with $d=4$ fixed. The results in Figure \ref{fig.q.1000} show, all success rate curves align well over MNs of different sizes for both ``lattice'' or randomly shaped structures.
\begin{figure*}[t]
	\subfigure[Random structure, $n_q=1000$.]{\label{fig.1000.uns}\includegraphics[width=.45\textwidth]{rand1000_norm_0}}
	\subfigure[Random structure $n_q=500$]{\label{fig.500}\includegraphics[width=.45\textwidth]{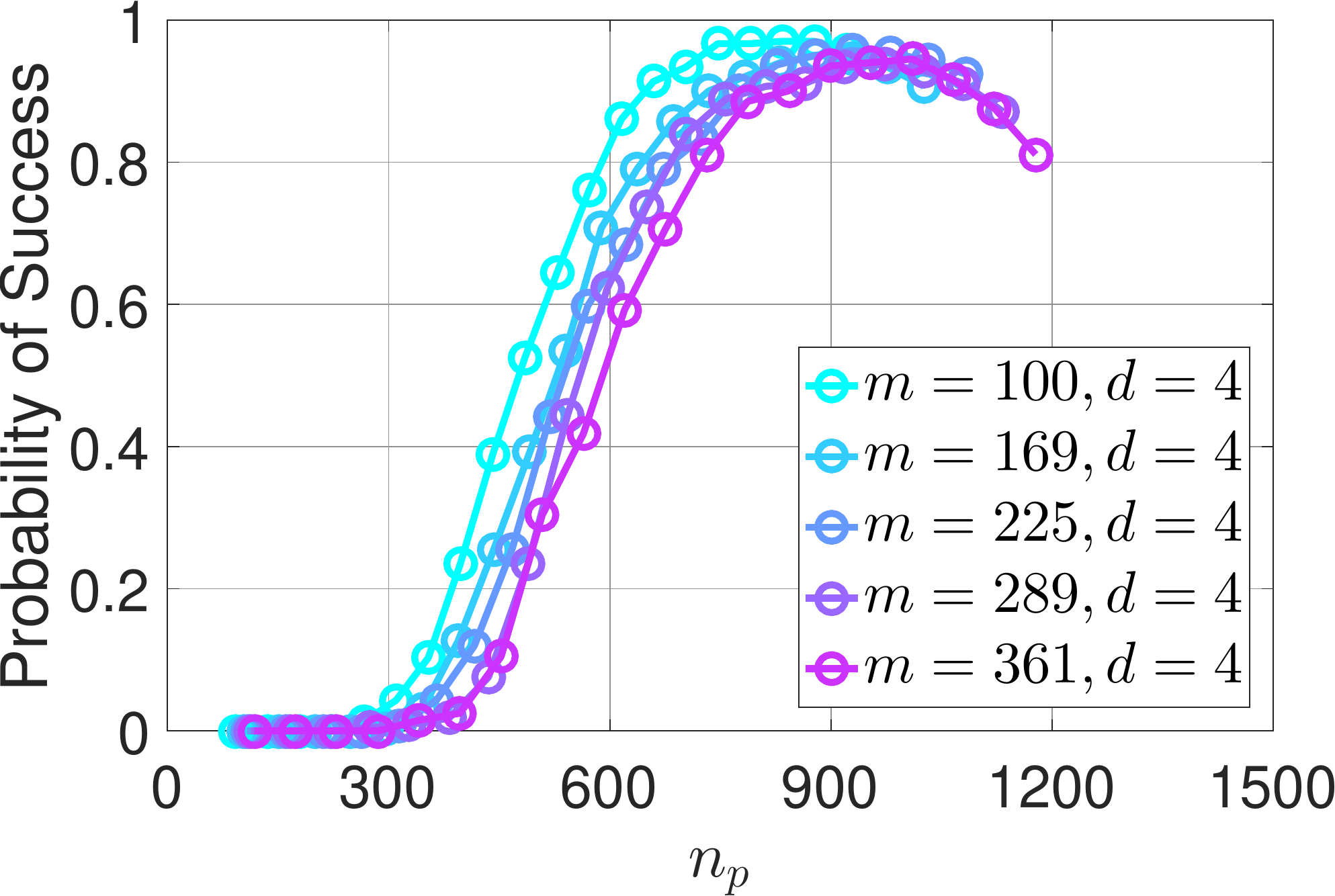}}
	\subfigure[Lattice, $n_q=0.01n_p^2$]{\label{fig.nq2}\includegraphics[width=.45\textwidth]{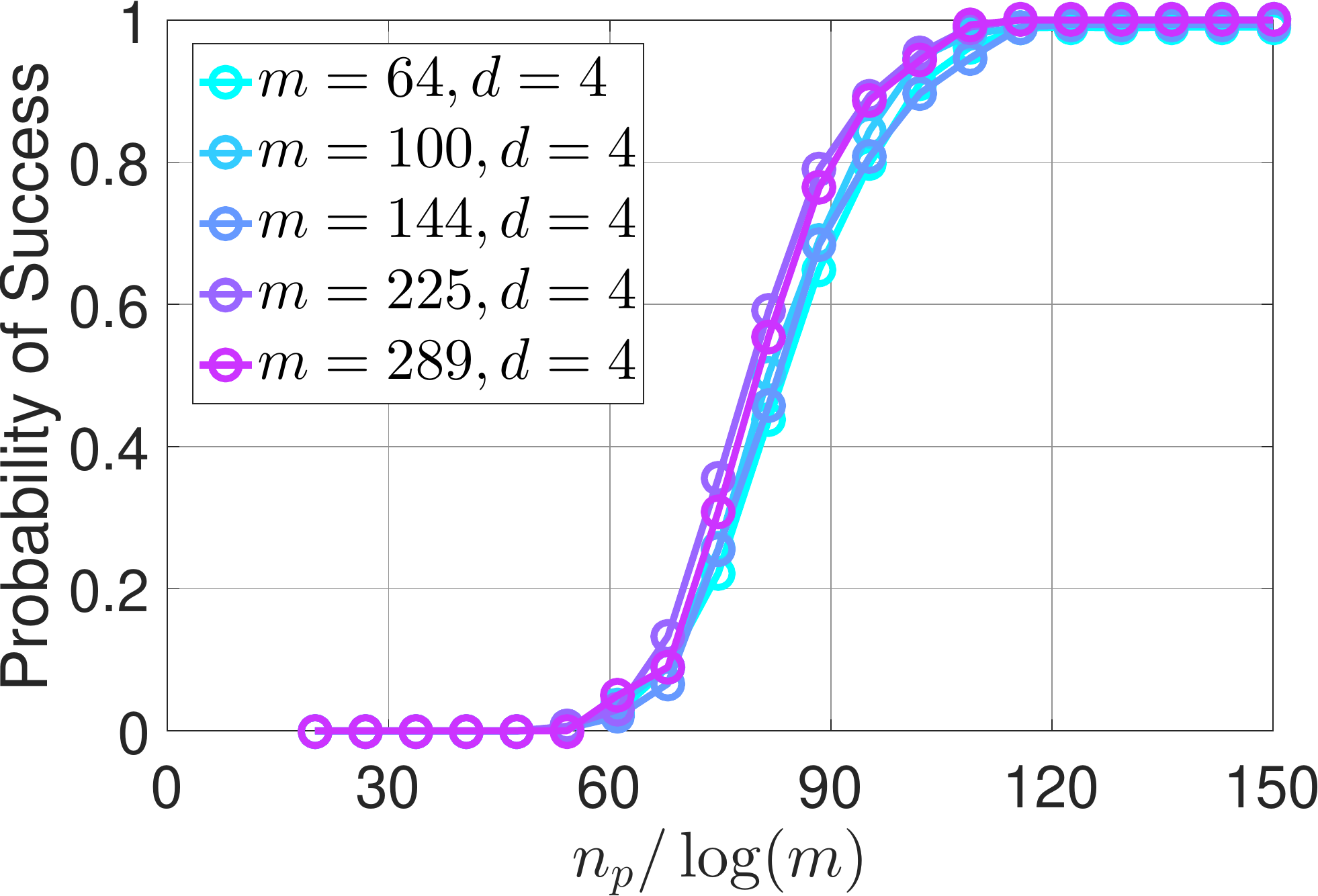}}
	\subfigure[Lattice, Truncated Gaussian, $n_p = n_q$]{
		\label{fig.truncated.gaussian}
		\includegraphics[width=.45\textwidth]{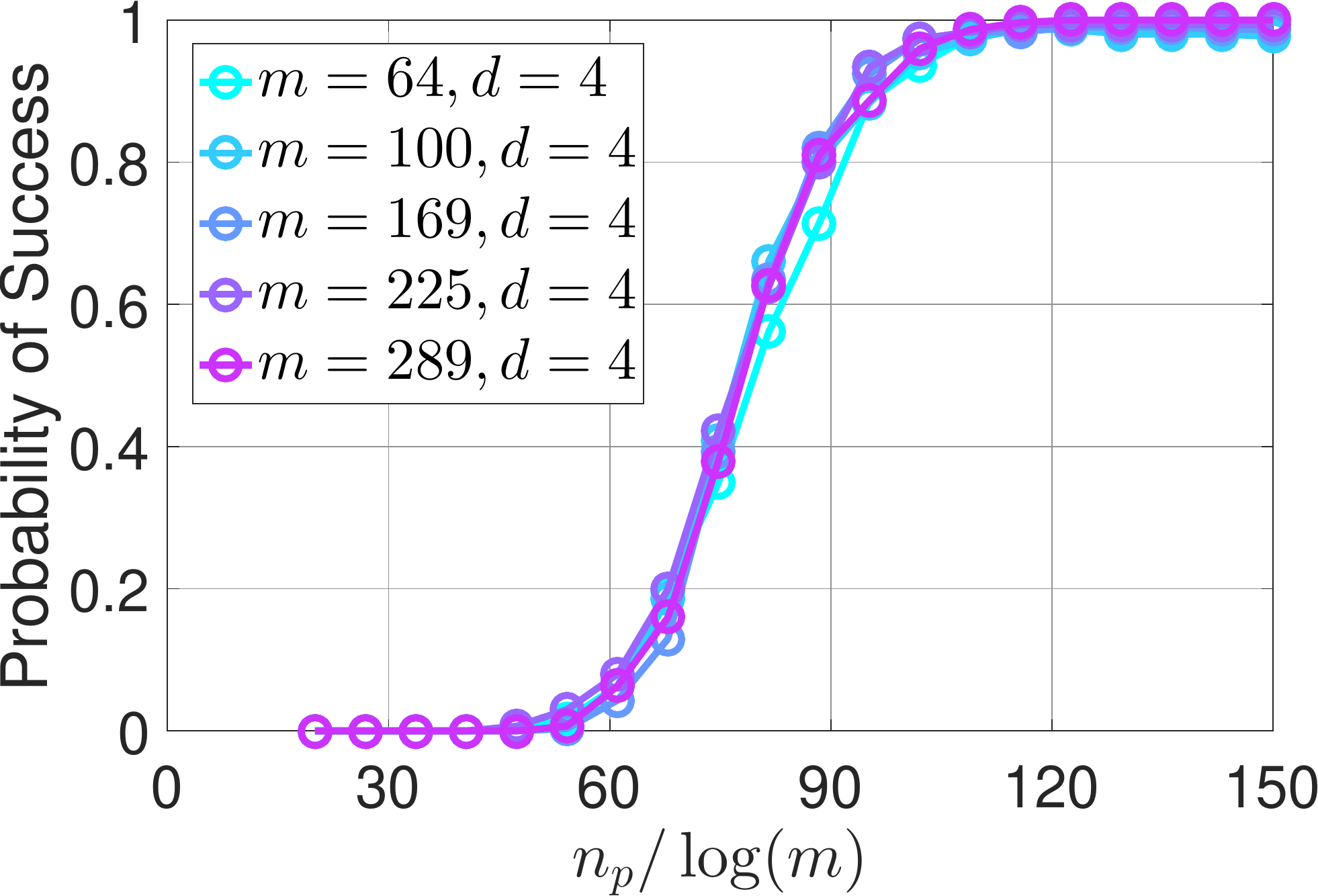}}
	\caption{The relationship between $n_p$ and $\log m$, when varying $n_q$. (d) Gaussian is truncated within a ball centred at the origin with the radius of $15$.}
\end{figure*}
\subsection{Changing $n_q$}\label{sec.exp.changing.nq} Our theorem also states that $n_q$ should also satisfy a certain relationship with $n_p$. In this experiment, we vary $n_q$ to observe the change of success rate pattern using the ``random'' and ``lattice'' dataset in the previous experiment. As we can see from Figure \ref{fig.500}, when $n_q = 500$, we \emph{cannot} reach 100\% success rate even for an ever growing $n_p$ and the probability of success even decays in the final stage.
This can be explained by \eqref{eq.them1.decay} in Theorem \ref{them.the.main.theorem}. If $n_q$ is large enough, the second term in \eqref{eq.them1.decay} can be safely ignored. However, as $\lambda_{n_p}$ decays when $n_p$ grows, a small $n_q$ may not be able to suppress the second term and the overall probability of success starts to decay eventually.
By setting $n_q = 0.01n_p^2$ we obtain a perfectly aligned result (Figure \ref{fig.nq2}), as our theorem indicated. It also shows that though $n_q$ is required to grow quadratically with $n_p$, it  can be rescaled by a small constant (in this case, 0.01). 
Moreover, Corollary \ref{corol.bounded.ratio} points out if the density ratio model is bounded, we may relax the coupling condition between $n_p$ and $n_q$. To verify this, we truncate a Gaussian distribution by rejecting samples fall out of a ball centered at origin with radius 15,
then let $n_q = n_p$. From Figure \ref{fig.truncated.gaussian} we can see a similar patter of success rates alignment.

\begin{figure*}[t]
	\begin{center}
		\subfigure[$d = 2,4,6,8,16,$ and $ 32, n_q = 0.01n^2_p$]{
			\label{fig.dd.1}
			\includegraphics[width=.45\textwidth]{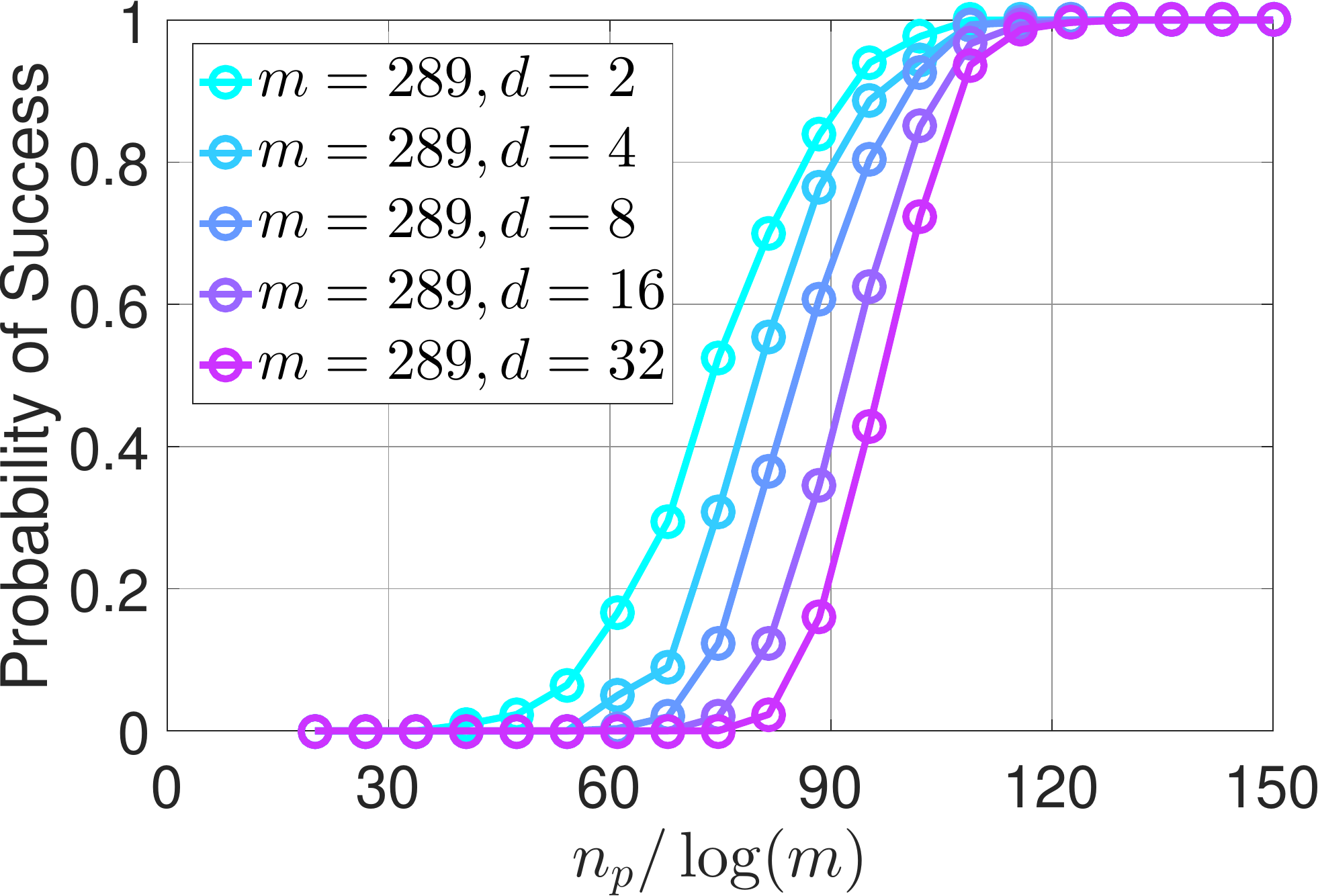}}
		\subfigure[``8-shaped Distribution'', $n_q = 5n_p$, without rescaling with $\log m$ ]{\label{fig.latt.nongaussian.diamond}\includegraphics[width=.45\textwidth]{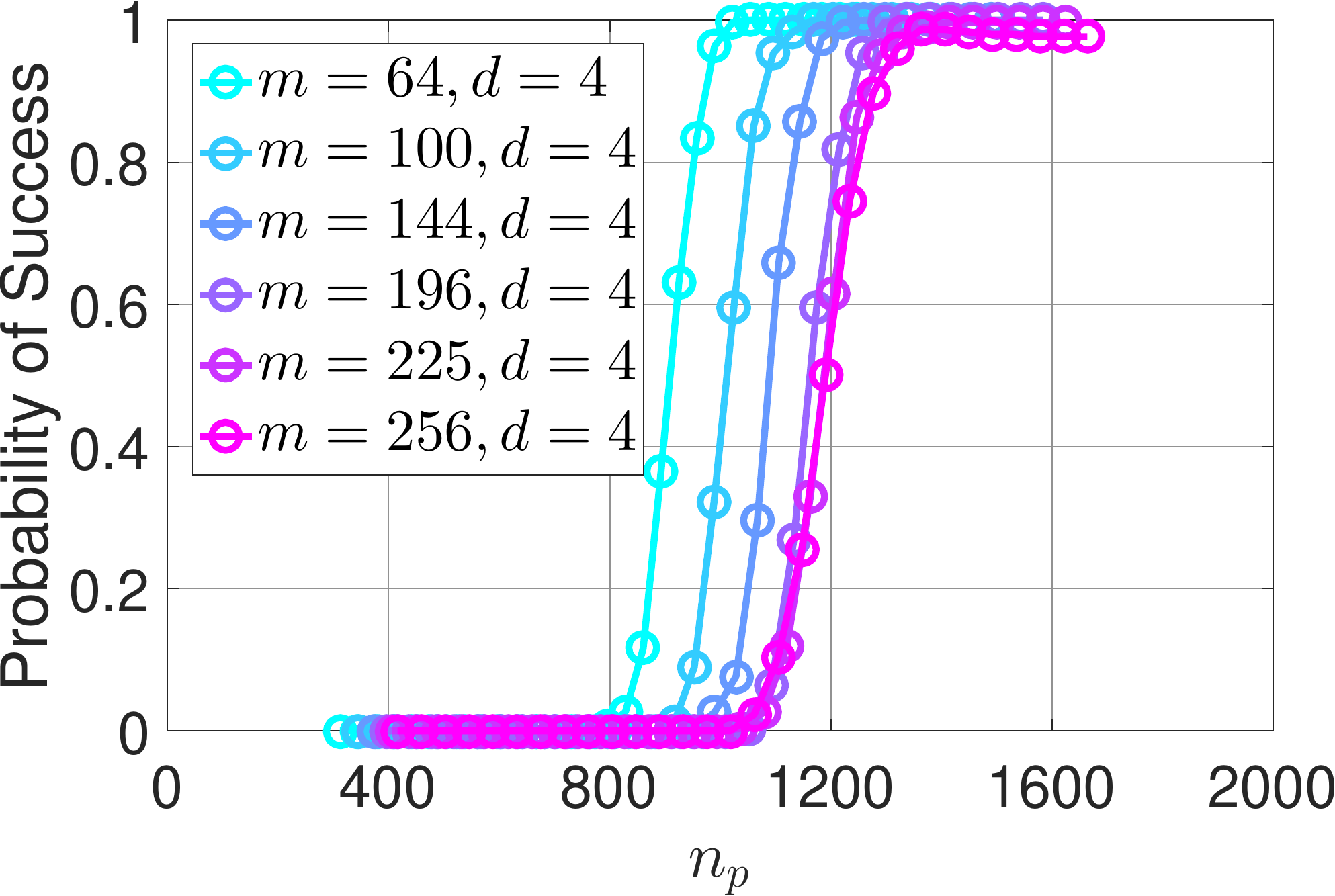}}
		\subfigure[``8-shaped'', $n_q = 5n_p$ ]{\label{fig.latt.nongaussian.8}\includegraphics[width=.45\textwidth]{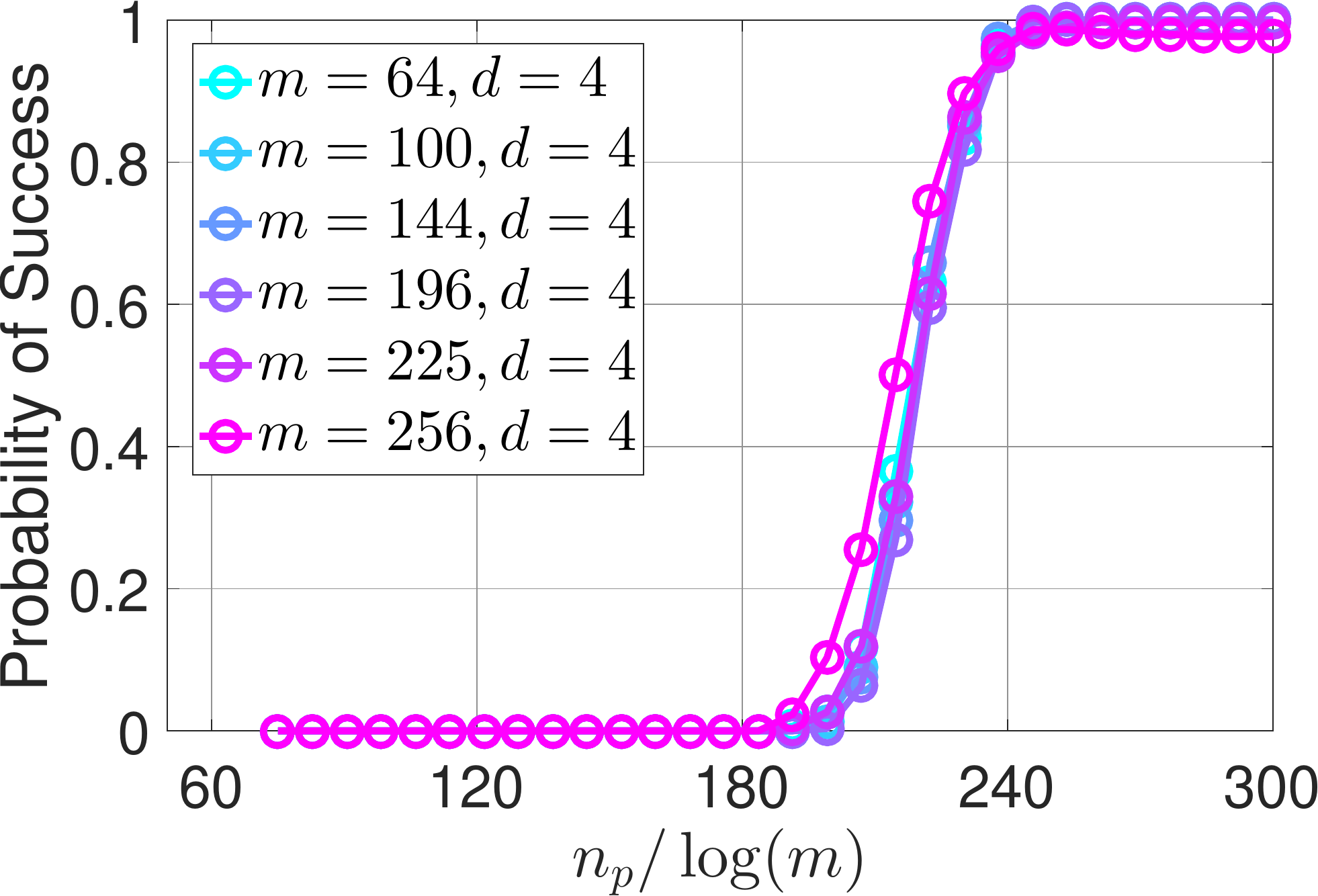}}
		\subfigure[ROC comparison]{
			\includegraphics[width=.45\textwidth]{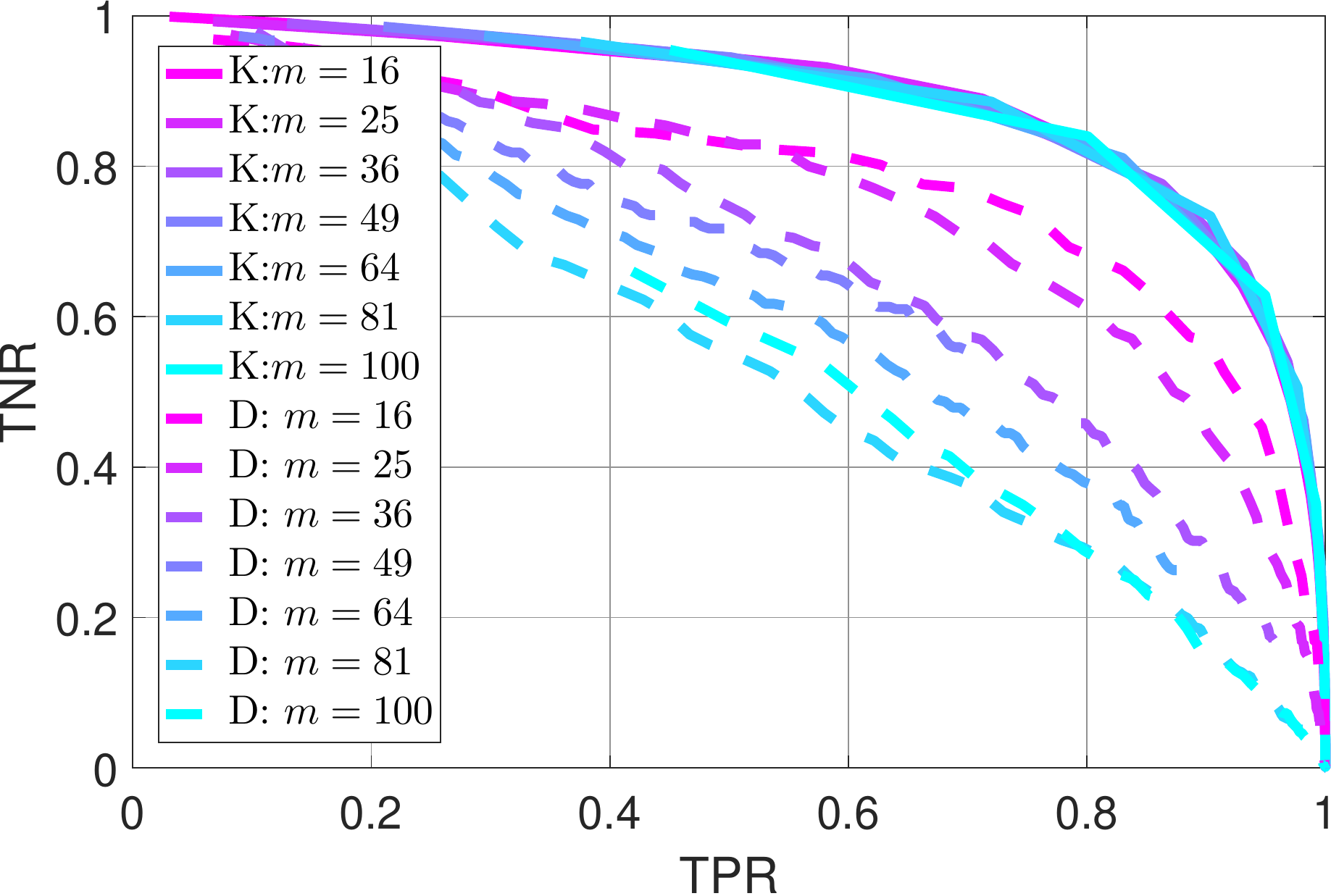}
			\label{fig.roc}
		}
	\end{center}
	\caption{The relationship between $n_p$ and $\log m$, when varying $d$ (a) or under a non-Gaussian distribution (b and c). (d) is the ROC comparison between (K)LIEP and (D)ifferential learning.}
\end{figure*}
\subsection{Changing $d$}\label{sec.exp.changing.d}
As Theorem \ref{them.the.main.theorem} indicates, $n_p$ should grow at least quadratically with $d$, the number of \emph{changed edges}. However, in our experiments, it shows such condition is overly conservative. Figure \ref{fig.dd.1} shows the success rate depends on $d$ only very mildly (see $n_p/\log m$ needed for success rates passing 80\%), which is a good news. This indicates that the bound can be tightened under certain regimes.
\subsection{Non-Gaussian distribution} \label{sec.exp.nongau}
We next perform experiments on the truncated ``8''-shaped distribution (see Figure \ref{fig.contour} in Appendix for details). The MNs are constructed as lattices, and the samples are generated via slice sampling \citep{slicesampling}. 
Figure \ref{fig.latt.nongaussian.8} shows, for the lattice grids with dimensions $m=60 \sim 256$, the curves of success rates are well-aligned with the setting $n_q = 5n_p$. 
\subsection{Comparison with Differential Network Leanring}
In this section, we conduct experimental comparison between KLIEP and the differential network learning proposed in \cite{zhao2014direct}. 
where the following constrained objective is minimized:
\begin{align*}
\hat{\boldDelta} = \argmin_{\boldDelta} \| \boldDelta \|_1 ~~ \text{subject to }
\| \hat{\boldSigma}^{(p)} \boldDelta \hat{\boldSigma}^{(q)} + \hat{\boldSigma}^{(p)} - \hat{\boldSigma}^{(q)}\|_{\infty} <\epsilon,
\end{align*}
where $\hat{\boldSigma}^{(p)}$ and $\hat{\boldSigma}^{(q)}$ are the sample covariance matrices, $\epsilon$ is a small constant and $\hat{\boldDelta}$ is the estimated differential network. To obtain a sparse solution, \citeauthor{zhao2014direct} thresholds the solution at a certain level $\tau$, i.e. for all $u,v$ that $\Theta_{u,v}<\tau$ is thresholded to $0$. 

We compare the performance between KLIEP and differential network learning by Receiver Operating Characteristic (ROC) plot of True Positive and True Negative rate described in \cite{zhao2014direct}. See Appendix \ref{sec.exp.diff} for more detailed settings.
Experimental data are constructed using the synthetic Gaussian 4-neighbor lattice MN grid described in Section \ref{sec.exp.gaussian}. However, instead of fixing some key variables, we adopt a more practical setting: $n_p = n_q = 50$ and $d = \sqrt{m}$ (i.e. the number of changed edges increases with $m$). The ROC curves of KLIEP (solid) and the differential method (dashed) are reported in Figure \ref{fig.roc}. 

As it can be seen from Figure \ref{fig.roc}, KLIEP performs consistently better than the differential network learning method. 
Particularly, when $m$ and $d$ increases, the performance almost remains the same for KLIEP while for the differential method, it decays significantly. 

Figure \ref{fig.roc} is plotted up to $m=100$, due to the fact that the differential method fails to return an output within a reasonable amount of time ($10$ hours) when $m$ is scaled up to $121$, using either authors' or our implementation. 
This is consistent with the authors' claim in \cite{zhao2014direct}, where the differential method requires more than $14$ hours to obtain the result for a MN sized $120$. 
As for KLIEP, we can compute a full ROC curve within $3$ hours even for $m=625$ when $n_p=n_q=50$. 

\section{Gene Expression Analysis}
\label{sec.gene}
\begin{figure*}[t]
	\includegraphics[width=.9\textwidth]{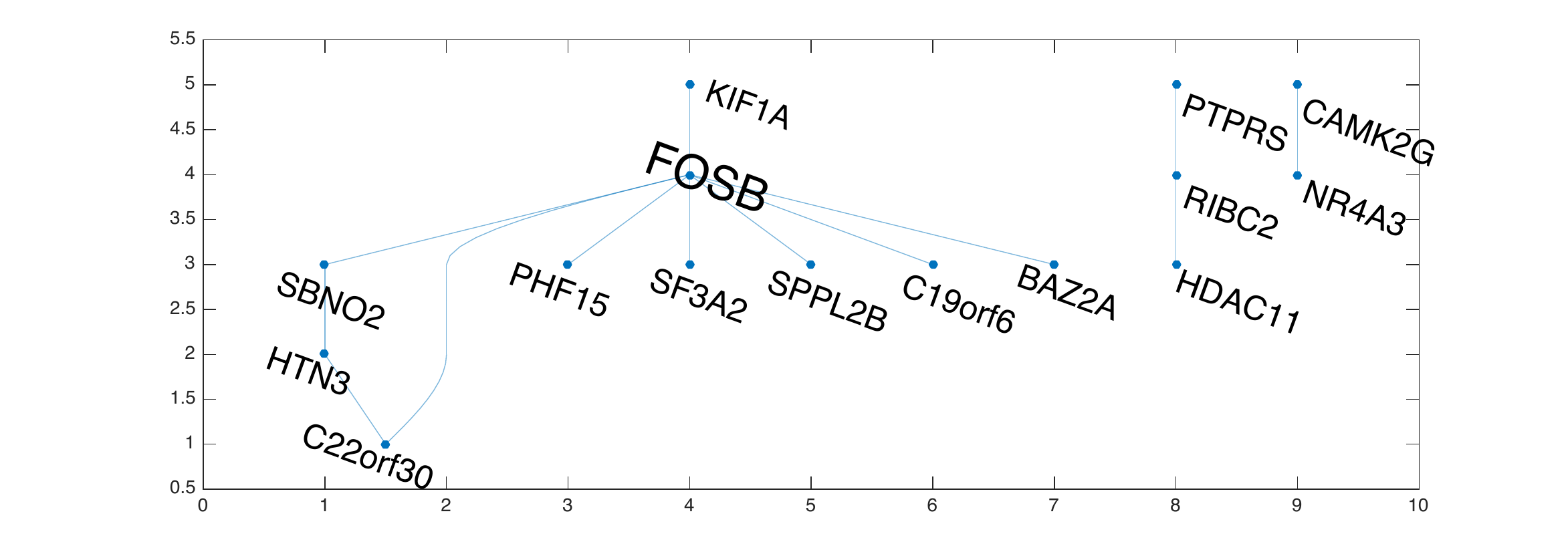}
	\caption{Gene expression change graph highlighting a hub node}
	\label{fig.gene1}
\end{figure*}
We applied KLIEP to gene expression profiles
for estimating changes in gene networks activated by
two different stimuli: epidermal growth factor (EGF) and
heregulin (HRG).
EGF is known to induce proliferation in MCF7 human breast cancer cells, while
HRG induces differentiation.
We used the gene expression data from \cite{nagashima2007quantitative}.
The expression profiles were
from cells stimulated with 
two controls, resulting in 29 EGF and 28 HRG sample conditions ($n_p = 29, n_q = 28$).
We extracted 1,835 genes ($m = 1835$) from the gene set
in \cite{nagashima2007quantitative} by selecting genes with high expression variance (at least
three times the mean of all variances). The values were log2-transformed and
normalized using the 2\% trimmed mean before finally getting the
respective ratios with the controls. The change graph is obtained by reducing the regularization parameter until $|\hat{S}| > 10$.

Figure \ref{fig.gene1} shows a learned change network.
Each node represents a gene, and each edge indicates
that the regulation between them is different
from EGF stimuli to HRG stimuli.
The leftmost large component includes 10 genes and 10 interactions,
and a hub node of the component is the FOSB gene,
which is a member of the Fos family of transcription factors,
regulating expressions of other genes.
This indicates that KLIEP successfully found that
FOSB regulates other genes without any prior knowledge,
and suggests that the regulation has been changed between stimuli.
Moreover, FOSB is known as a regulator of
cell proliferation and differentiation \cite{nagashima2007quantitative},
showing that the detected change network agrees with
biological knowledge.
The result confirms that KLIEP can
detect known biologically significant changes from
the expression data. We also swapped $P$ and $Q$, and confirmed that detected changes are similar.

To evaluate the reliability of this experiment, we also conducted bootstrap experiments which are detailed in Section \ref{sec.bootstrap} in the Appendix.
\section{Discussion}
\label{sec.discussion}
Let's further discuss some properties of KLIEP.
\subsection{Asymmetry of KLIEP}
It can be noticed that KLIEP is asymmetric, i.e. the structural change learning performance may differ when swapping $P$ and $Q$. This is caused by directly parametrizing the density ratio which is naturally asymmetric. Such asymmetry is also reflected in Theorem \ref{them.the.main.theorem}, where the sample complexity is not the same as $n_p = \Omega(d\log m(m+1)/2)$ and $n_q= \Omega(n_p^2)$, though they are ``symmetrized'' in Corollary \ref{corol.bounded.ratio} as $\min(n_p, n_q) = \Omega(d\log(m(m+1)/2))$ under the stronger bounded ratio assumption (Assumption \ref{assum.bounded.ratio}). These analysis shed some light on choosing $P$ and $Q$ in practice when two datasets are given: $Q$ should be ``wide'' and more ``spread-out'' compared to $P$, so that the boundedness of the density ratio can be guaranteed, since a ``sharp'' $Q$ would lead to very sharp and unbounded density ratio. This insight implies that KLIEP is in fact a \emph{directional} method, and achieves better performance when the change itself shows a tendency of evolving from one ``general'' state to another more ``specialized'' state.


Symmetric measures of differences between densities are also available without using the density ratio, such as $\ell$-2 distance \cite{sugiyama2013density}. In general, such difference cannot be parameterized using the difference between parameters of individual MNs, thus cannot serve as a difference measure for direct structure change learning. 
An alternative is the differential network learning \cite{zhao2014direct} where the objective function is symmetric for both $P$ and $Q$.  However such an objective is only sensible when both $P$ and $Q$ are Gaussian and it cannot be easily generalized to non-Gaussian cases.

Moreover, inspired by Corollary \ref{corol.bounded.ratio}, we may consider a ``symmetrized'' version of KLIEP by learning $p/q$ and $q/p$ independently, after which we mark changing edges by taking the \textit{union} of two sparsity patterns in both models. For the support recovery probability, we simply need a union bound applied on the current results. Such a union-support algorithm is similar to the node-wise regression that is discussed in \cite{Ravikumar_2010}.

\subsection{Comparison with Differential Network Learning \cite{zhao2014direct}}
Other than the asymmetry issues mentioned above, the theoretical analysis in this work and the one in \cite{zhao2014direct} share some key similarities and both have good guarantees in high dimensional setting. First, they both set assumptions on the true difference/changes between two MNs. In \cite{zhao2014direct}, such a constrain is explicitly expressed: The true differential network $\|\boldDelta^*\|_1$ is bounded by a constant that does not grow with $m$ (Condition 1, \cite{zhao2014direct}). In our analysis, the assumption is implicitly made via limiting the ``smoothness'' of the density ratio (Assumptions \ref{assum.smooth.ratiomodel.nod}, \ref{assum.smooth.ratiomodel}, \ref{assum.bounded.ratio}, \ref{assum.bounded.ratio.3}). 
Second, differential network learning prohibits strong ``connections'' among covariates. It assumes the magnitude of off-diagonal values in the covariance matrices decays as the number of changed edges increases (Condition 2, \cite{zhao2014direct}), while in our work, we assume feature vectors on unchanged edges should not have strong correlations with those on changed edges (Assumption \ref{assum.incoherence}). Both assumptions are imposed to satisfy the incoherence condition \cite{donoho2001uncertainty} needed for Lasso-type model selection. The convergence results are also similar: The required sample size scales with $\log m$ for differential network learning (Theorem 2, \cite{zhao2014direct}) and KLIEP ($\log (m^2+m) \approx 2\log m$); The difference between the true parameter and the estimated one vanishes at the speed of $\sqrt{\frac{\log m}{\mathrm{min}(n_p,n_q)}}$ in the $\ell$-2 norm. However, we manage to achieve this rate without any assumptions of Gaussianity, while differential network learning described in \cite{zhao2014direct} cannot be directly applied on non-Gaussian MNs.

\subsection{Joint Structural Change Learning}
Another emerging trend in graphical model structural learning is to learn multiple \emph{similar} MNs simultaneously \cite{Danaher2013jointGlasso}. For example, one may use the following fused-lasso penalized objective function to learn $K$-Gaussian MNs at the same time:
\begin{align}
\label{eq.mle.fused}
\min_{\boldTheta^{(1)},\dots,\boldTheta^{(K)}} \sum_{a=1}^{K}\ell_\mathrm{MLE}(\boldTheta^{(a)}) + \lambda_1\sum_{\substack{a < b\\ a,b \in \{1,\dots,K\}}}\|\boldTheta^{(a)}-\boldTheta^{(b)}\|_1
+ \lambda_2\sum_{a \in \{1, \dots, K\}} \|\boldTheta^{(a)}\|_1,
\end{align}
where $\ell_\mathrm{MLE}(\boldTheta)$ is the negative Gaussian MN log-likelihood parametrized by precision matrix $\boldTheta$ and the assumption is that all $\boldTheta^{(a)}$ share a similar structure. Though the final outputs are estimated sparse precision matrices, one may still obtain a differential graph by taking the differences. 
We refer the readers to \cite{liu2014ChangeDetection} for more empirical comparisons between KLIEP and Fused-lasso differential network learning. 

Following the same spirit, the KLIEP based change detection can also be utilized to learn \emph{changes} from $K$-MNs at the same time by assuming all MNs sharing a similar structure, and one of the options is 
\begin{align}
\label{eq.KLIEP.fused}
\min_{\boldtheta^{a, b}, a,b \in \{1\dots K\}, a\neq b} \sum_{\substack{a,b \in \{1\dots K\}\\a \neq b}}^{K}\ell^{a, b}_\mathrm{KLIEP}(\boldtheta^{a, b}) + &\lambda \sum_{ \substack{a,b,c \in \{1\dots K\}}}\sum_{u,v = 1, u\le v}^m\|\boldtheta_{u,v}^{a, b}+\boldtheta_{u,v}^{b, c}\|_2
\end{align}
where $\ell^{a, b}_\mathrm{KLIEP}(\boldtheta^{a, b})$ is the KLIEP log-likelihood of density ratio $p^{a}/p^{b}$, $a,b,c$ are the triples from $1 \ldots K$ and by definition we have \[\boldtheta_{u,v}^{a, b}+\boldtheta_{u,v}^{b, c} \equiv \boldtheta_{u,v}^{a} - \boldtheta_{u,v}^{b}+ \boldtheta_{u,v}^{b} -  \boldtheta_{u,v}^{c} \equiv \boldtheta_{u,v}^{a} -  \boldtheta_{u,v}^{c},\] and let $\boldtheta_{u,v}^{a,a} \equiv6 \boldzero$.

The objective described in \eqref{eq.mle.fused} still requires a tractable likelihood and cannot be easily generalized to non-Gaussian models. It also assumes each individual MN is sparse. However, the variation of KLIEP given in \eqref{eq.KLIEP.fused} does not impose any assumptions on each individual MNs and can be computed even in non-Gaussian cases. 
We will explore this algorithm and other possible alternatives in future works.

\subsection{Uncertainty of Estimation}
In this work, we have only focused on the successful change detection (Theorem \ref{them.the.main.theorem}) which states $P\left(S=\hat{S} \text{ \& } S^c = \hat{S}^c\right)$ converges to one eventually. However, under very a low sample regime, the learned change structure may have high uncertainty, i.e., the results may be sensitive to some minor modifications or the randomness of the dataset. This would degrade the reliability of the obtained results.
To evaluate such an uncertainty, we performed bootstrap experiments on the gene dataset used in Section \ref{sec.gene}, and the results are presented in Section \ref{sec.bootstrap} in the Appendix. Such a method is useful for practitioners to measure the reliability of an estimated network. However, the rigours quantification of such an uncertainty is still an open question and would be an important future direction to pursue. 


\bibliography{main}

\newpage

\begin{changemargin}{-1cm}{-1cm}
	\appendix
	
\section{Image Change Detection}
\begin{figure*}[t]
	\subfigure[4:30PM, 7th Mar, 2016]{
		\label{fig.img.a2}
		\includegraphics[width=.4\textwidth]{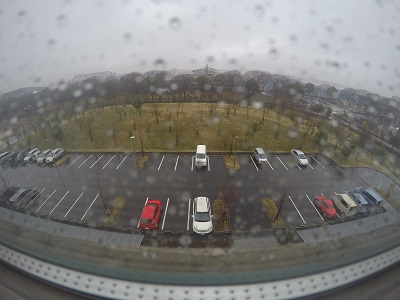}}
	\subfigure[5:30PM, 7th Mar, 2016]{
		\label{fig.park2}
		\includegraphics[width=.4\textwidth]{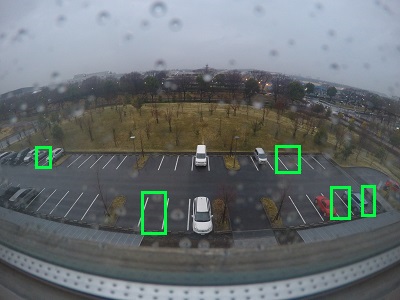}}\\
	\subfigure[Construct samples using sliding windows (red boxes).]{
		\label{fig.construct.samples}
		\includegraphics[width=.4\textwidth]{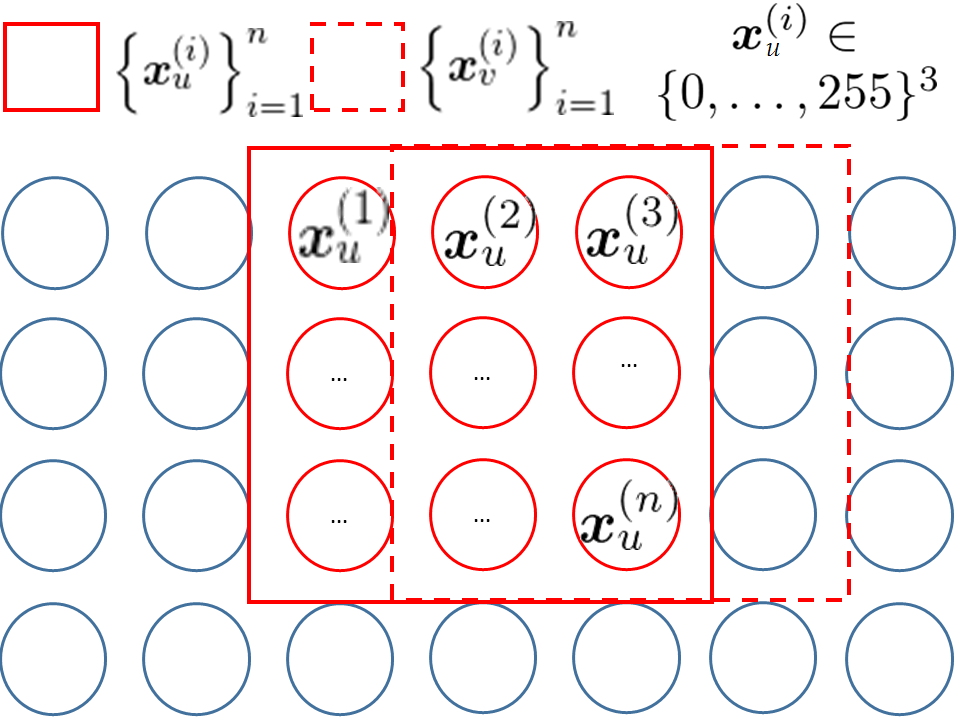}}
	\subfigure[Detected changes, we set $\psi(\boldx_u, \boldx_v) = \exp\left(-\frac{\|\boldx_u - \boldx_v\|^2}{0.5}\right)$]{
		\label{fig.detected}
		\includegraphics[width=.4\textwidth]{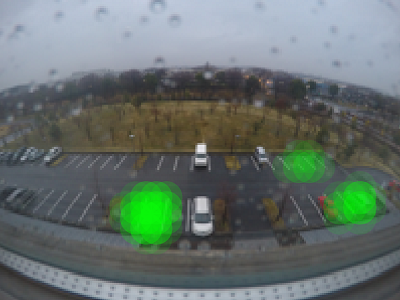}}
	\caption{Detecting changes of parking patterns from two photos.}
\end{figure*}
Two photos were taken in a rainy afternoon using a camera pointing at the parking lot of ISM. In this task, we are interested in learning the changes of the parking patterns marked by green boxes in Figure \ref{fig.park2}. 
As we can see from Figure \ref{fig.img.a2} and \ref{fig.park2}, the light conditions and  positions of raindrops vary in two pictures. 

To construct samples, we use windows of pixels (Figure \ref{fig.construct.samples}). 
Each window is a dimension of a dataset, and the samples are the pixel RGB values within this window. By sliding the window across the entire picture, we may obtain samples of different dimensions. Two sets of data can be obtained by using this sample generating mechanism over two images. 

Assuming an image can be represented by an MN of windows, changes of pixels values within a window may cause changes of ``interactions'' between neighboring windows. In other words, we can discover a change by looking at the change of the dependency of pixel values between a certain window and its neighbours. 
This is more advantageous than simply looking at the pixel values  since changing the brightness of a picture may increase the pixel values in many windows simultaneously, even if the ``contrast'' between two windows does not change by much. 

By applying KLIEP on such two sets of data and highlighting adjacent window pairs that are involved in the changes of pairwise interactions, we may spot changes between two images. 
In our experiment, we use sliding windows of size $16 \times 16$ on a $200\times 150$ image, generating two sets of samples with $m=999$ and $n_p=n_q=256$. We reduce $\lambda$ until $|\hat{S}| > 40$.  The spotted changes were plotted in Figure \ref{fig.detected}. It is can be seen that KLIEP has correctly labeled almost all changed parkings between two images except one missing on the left. 

\section{Proofs}
\subsection{Proof of Lemma \ref{lem.1.maintext}}
\label{sec.proof.lemma.1}
In the appendix, we denote $\ell\left(\boldtheta\right)$ as the \emph{negative} likelihood objective function $\ell_{\mathrm{KLIEP}}(\boldtheta)$.
\begin{proof}
Consider two optimals $\hat{\boldtheta}, \tilde{\boldtheta}$ of \eqref{eq.obj.alter}, $\hat{\boldtheta}$ has the correct sparsity over zero elements. We have the following equality
\begin{align}
\label{eq.lemma1.equality}
\ell(\hat{\boldtheta}) + \lambda_{n_p} \sum_{t\in S \cup S^c}\|\hat{\boldtheta}_{t}\| = \ell(\tilde{\boldtheta}) + \lambda_{n_p} \sum_{t\in {S \cup S^c}}\|\tilde{\boldtheta}_{t}\|.
\end{align}
noting $\hat{\boldz}_{t} \in \nabla \|\boldtheta_t\|, t\in S\cup S^c $,
\begin{align}
\label{eq.lemma1.ineq1}
\lambda_{n_p}\|\tilde{\boldtheta}_{t}\| \ge \lambda_{n_p} \langle \tilde{\boldtheta}_{t}-\hat{\boldtheta}_{t},\hat{\boldz}_{t} \rangle + \lambda_{n_p} \|\hat{\boldtheta}_{t}\|,
\end{align}
substituting \eqref{eq.lemma1.ineq1} into \eqref{eq.lemma1.equality} we have the following inequality:
\begin{align*}
\ell(\hat{\boldtheta}) + \lambda_{n_p} \sum_{t\in S \cup S^c}\|\hat{\boldtheta_t}\| \ge
\ell(\tilde{\boldtheta}) + \lambda_{n_p} \sum_{t\in S \cup S^c}\langle \tilde{\boldtheta}_{t}-\hat{\boldtheta}_{t},\hat{\boldz}_{t} \rangle + \lambda_{n_p} \sum_{t\in S \cup S^c} \|\hat{\boldtheta}_t\|.
\end{align*}
Due to convexity of $\ell(\boldtheta)$ and $\nabla \ell (\hat{\boldtheta}) = -\lambda_{n_p}\hat{\boldz}$ as stated in \eqref{eq.duality}, the leftmost term is further lower bounded as:
\begin{align*}
\ell(\hat{\boldtheta}) + \lambda_{n_p} \sum_{t\in S \cup S^c} \|\hat{\boldtheta}_{t}\| \ge& \ell(\hat{\boldtheta}) + 
\lambda_{n_p} \sum_{t\in S \cup S^c} \langle \tilde{\boldtheta}_{t}-\hat{\boldtheta}_{t},-\hat{\boldz}_{t} \rangle 
+ \lambda_{n_p} \sum_{t\in S \cup S^c} \langle \tilde{\boldtheta}_{t}-\hat{\boldtheta}_{t},\hat{\boldz}_{t} \rangle + \lambda_{n_p} \sum_{t\in S \cup S^c} \|\hat{\boldtheta}_t\|\\
\ge& \ell(\hat{\boldtheta}) + \lambda_{n_p} \sum_{t\in S \cup S^c} \|\hat{\boldtheta}_{t}\|
\end{align*}
The above suggests all the inequality we have used to lower-bound $\ell(\hat{\boldtheta}) + \lambda_{n_p} \sum_{t\in S \cup S^c} \|\hat{\boldtheta}_{t}\|$ should take the exact equality. Therefore \eqref{eq.lemma1.ineq1} should take the exact equality:
\begin{align*}
\lambda_{n_p}\|\tilde{\boldtheta}_{t}\| &= \lambda_{n_p} \langle \tilde{\boldtheta}_{t}-\hat{\boldtheta}_{t},\hat{\boldz}_{t} \rangle + \lambda_{n_p} \|\hat{\boldtheta}_{t}\|\\
&= \lambda_{n_p} \langle \tilde{\boldtheta}_{t},\hat{\boldz}_{t} \rangle,
\end{align*}
which holds only if $\tilde{\boldtheta}_{t} = \boldzero$, given $\|\hat{\boldz}_{t}\| < 1$.

Moreover, if $\mathcal{I}_{SS}$ is invertible, it can be shown that $\hat{\boldtheta}$ is the unique solution. 
\end{proof}

\subsection{Proof of Lemma \ref{lem.2.maintext}}
\label{sec.proof.lemma.2}
Before proving the lemma, we show the boundedness of the deviation between true normalization term $N$ and sample approximated term $\hat{N}$. For conveniences, we denote $A(\boldtheta)$ as $\log N(\boldtheta)$ and $\hat{A}(\boldtheta)$ as $\log \hat{N}(\boldtheta)$.

\begin{prop}
\label{prop.norm.bound}
For any vector $\bolddelta \in \mathbb{R}^{\text{dim}(\boldtheta^*)}$ such that $\|\bolddelta\|\leq \|\boldtheta^*\|$ and some constant $c$,
\begin{align}
\label{eq.prop1}
P\left\{ A(\boldtheta^* + \bolddelta) - A(\boldtheta^*) - \left[\hat{A}(\boldtheta^* + \bolddelta)- \hat{A}(\boldtheta^*) \right]   \geq c \right\} \leq 2\exp \left( -\frac{ c^2 n_q }{200}\right).
\end{align}
\end{prop}
\begin{proof}
Since $\hat{N}(\boldtheta^* + \bolddelta) = \frac{1}{n_q} \sum_{i=1}^{n_q}r\left(\boldx^{(i)}; \boldtheta^*+\bolddelta\right) \cdot N(\boldtheta^*+\bolddelta)$, from Proposition \ref{eq.assum.well.behave.nod} we have
\[
P( \hat{N}(\boldtheta^* + \bolddelta) - N(\boldtheta^* + \bolddelta) \le -\epsilon N(\boldtheta^* + \bolddelta) ) \le 2\exp\left(- n_q \epsilon^2 /40 \right)
\]
A few lines of algebra can show that 
\[
P\left( A(\boldtheta^* + \bolddelta) - \hat{A}(\boldtheta^* + \bolddelta) \geq -\log (1- \epsilon) \right) \leq 2\exp\left(- n_q \epsilon^2 /40 \right).
\]

Similarly, 
\[
P\left( \hat{A}(\boldtheta^*) - A(\boldtheta^*) \geq \log (1+\epsilon) \right) \leq 2\exp\left(- n_q  \epsilon^2 /40 \right).
\]
Applying union bound, we can get
\[
P\left( A(\boldtheta^* + \bolddelta) - A(\boldtheta^*) - \left[\hat{A}(\boldtheta^* + \bolddelta)- \hat{A}(\boldtheta^*)  \right]   \geq \log \frac{1+\epsilon}{1-\epsilon} \right) \leq 4\exp\left(- n_q \epsilon^2 /40\right).
\]
Set $c = \log \frac{1+\epsilon}{1-\epsilon}$, then
\begin{align*}
P\left\{A(\boldtheta^* + \bolddelta) - A(\boldtheta^*) - \left[\hat{A}(\boldtheta^* + \bolddelta)- \hat{A}(\boldtheta^*)  \right]  \geq c \right\} &\leq 4\exp \left( -\frac{n_q}{40}\left(\frac{\exp(c)-1}{\exp(c)+1}\right)^2\right), 
\end{align*}
particularly when $c \le 1$, which is often the case in practice, 
\begin{align}
\label{eq.prop1.simplification}
P\left\{A(\boldtheta^* + \bolddelta) - A(\boldtheta^*) - \left[\hat{A}(\boldtheta^* + \bolddelta)- \hat{A}(\boldtheta^*) \right]  \geq c \right\} &\leq 4\exp \left( -\frac{ c^2 n_q }{200}\right), 
\end{align}
by using the fact that $\left(\frac{\exp(c)-1}{\exp(c)+1}\right)^2 > \frac{c^2}{5}$ when $c \le 1$ (see Figure. \ref{fig.exp_c_c_squared}).
\begin{figure}
	\caption{The plot of $\left(\frac{\exp(c)-1}{\exp(c)+1}\right)^2 / c^2$}
	\includegraphics[width=.7\textwidth]{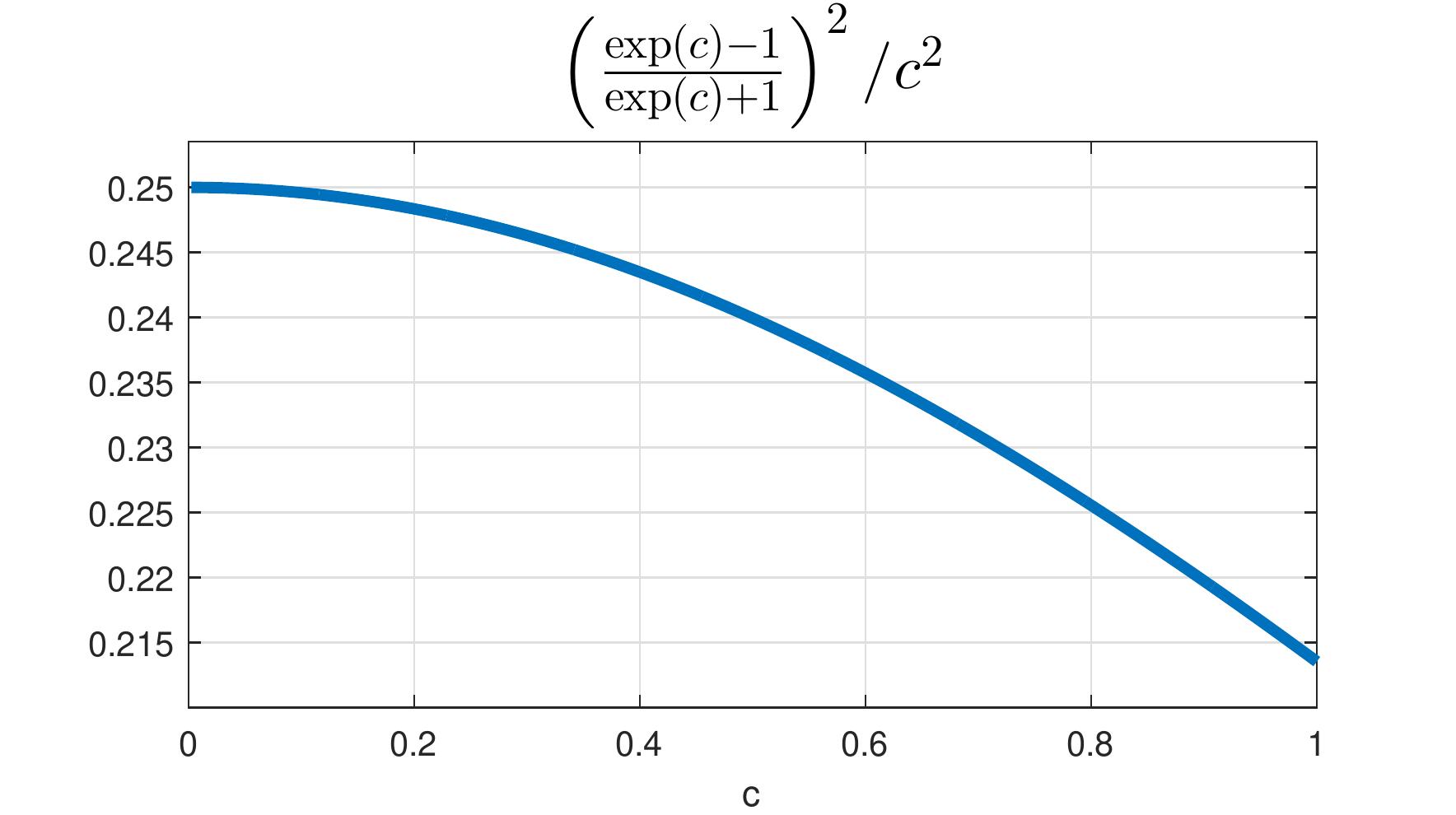}
	\label{fig.exp_c_c_squared}
\end{figure}
\end{proof}

We then introduce the proof of Lemma \ref{lem.2.maintext}:
\begin{proof}
To bound $\|\boldw_t\|,$ where $ \boldw_t = -\nabla_{\boldtheta_t}\ell(\boldtheta^*)$, we may show $\langle \text{sign}(\boldw_t), \boldw_t \rangle = |\boldw|$ is bounded.

The direct boundedness of $|\boldw|$ could be difficult to prove, however we can investigate the bound of the inner product 
$\langle \boldu_t, \boldw\rangle$ where $\boldu_t \in \{-1,1\}^{\text{dim(}\boldw)} $ is a zero padding sign vector with only non-zero elements on sub-vector indexed by $t$. Since $\boldw=\frac{1}{n_p}\sum_{i=1}^{n_p}\boldw^{(i)}$,
\[
\langle \boldu_t, \boldw\rangle = \frac{1}{n_p}\sum_{i=1}^{n_p}\langle \boldu_t, \boldw^{(i)} \rangle = \frac{1}{n_p}\sum_{i=1}^{n_p}\langle \boldu_t, \boldf(\boldx^{(i)}) - \nabla \hat{A}(\boldtheta^*) \rangle,
\]  
bound each summand indexed by $(i)$ is sufficient.

In order to use Chernoff bounding technique, we look into the moment generating function of $\langle \boldu_t, \boldw^{(i)}\rangle$. Since
$
p^*(\boldx) = q^*(\boldx) r(\boldx,\boldtheta^*),
$ 
after a few lines of algebra we have 
\begin{align*}
\mathbb{E}_{p^*}\left[\exp\left(\langle v\boldu_t,\boldf(\boldx^{(i)}) - \nabla \hat{A}(\boldtheta^*)\rangle\right)\right]& =  \mathbb{E}_{q^*}\left[\exp\langle \boldtheta^* + v\boldu_t, \boldf(\boldx^{(i)}) \rangle \right]h(v\boldu_t)
\end{align*}
where $v \le \frac{\|\boldtheta^*\|}{\|\boldu_t\|}$ and $h(v\boldu_t) = \exp\left[ -  A(\boldtheta^*) - \langle v\boldu_t, \nabla \hat{A}(\boldtheta^*) \rangle\right]$. Take logarithm on both sides,
\begin{align*}
\log\mathbb{E}_{q^*}\left[\exp\left(\langle v\boldu_t,\boldf(\boldx^{(i)}) - \nabla \hat{A}(\boldtheta^*)\rangle\right)\right]&= A(\boldtheta^* +
v\boldu_t ) + \log h(v\boldu_t) \notag\\
& =A(\boldtheta^* +
v\boldu_t ) - A(\boldtheta^*) - \langle v\boldu_t,\nabla \hat{A}(\boldtheta^*)\rangle\notag.
\end{align*}

Now, define the event,
\begin{align}
\label{eq.event.xi}
\xi_{\boldu_t} :=\left\{ \left[A(\boldtheta^*+ v\boldu_t) - A(\boldtheta^*) -\left( \hat{A}(\boldtheta^* +  v\boldu_t)-\hat{A}(\boldtheta^*)\right)\right] \leq v^2\right\}.
\end{align}
Note that $P(\xi_{\boldu_t}^c) \le 2 \exp \left( -\frac{ v^4 n_q }{200} \right)$ given $v^2 \le 1$ 
Then the following holds conditioned on event $\xi_{\boldu_t}$:
\begin{align}
\label{eq.lemma2.exp}
\log\mathbb{E}_{q^*}\left[\exp\left(\langle v\boldu_t,\boldf(\boldx^{(i)}) - \nabla \hat{A}(\boldtheta^*)\rangle\right)\right] \leq\hat{A}(\boldtheta^* + v\boldu_t ) - \hat{A}(\boldtheta^*) - \langle v\boldu_t,\nabla \hat{A}(\boldtheta^*)  \rangle + v^2,
\end{align}

Then applying Taylor expansion to \eqref{eq.lemma2.exp},
\begin{align}
\label{eq.emperical.upperbound}
\hat{A}(\boldtheta^* + v\boldu_t ) - \hat{A}(\boldtheta^*) - \langle v\boldu_t,\nabla \hat{A}(\boldtheta^*)  \rangle + v^2 &= \frac{1}{2}v\boldu_t^\top\nabla^2 \hat{A}(\boldtheta^*+ \bar{\boldu}_t)v\boldu_t + v^2, \notag \\
&\leq \frac{1}{2}\|v\boldu_t\|^2\lambda_{\text{max}}(\nabla^2 \hat{A}(\boldtheta^*+ \bar{\boldu}_t)) + v^2 \notag\\
&\leq \frac{\lambda_{\text{max}}bv^2}{2} + v^2
\end{align}
where the $\bar{\boldu}_t$ is a vector between $\boldtheta^*$ and $\boldtheta^* + v\boldu_t$ in coordinate fashion. 

Finally, applying the Chernoff bounding technique:
\begin{align*}
\log P\left( \langle \boldu_t,\boldf(\boldx^{(i)}) - \nabla \hat{A}(\boldtheta^*)\rangle\geq \beta ~|~ \xi_{\boldu_t} \right) &\leq \log \mathbb{E}_p \left[ \exp \left( \langle v\boldu_t,\boldf(\boldx_p^{(i)}) - \nabla \hat{A}(\boldtheta^*)\rangle \right)/\exp(\beta v)\right]\\
& \leq 
\frac{\lambda_{\text{max}}bv^2}{2}+ v^2 - \beta v
\end{align*}
and 
\begin{align*}
&\log P\left(\langle \boldu_t,\underbrace{\frac{1}{n_p}\sum_{i=1}^{n_p}\boldf(\boldx^{(i)}) - \nabla \hat{A}(\boldtheta^*)\rangle}_{-\nabla\ell(\boldtheta^*)}\geq \beta ~|~ \xi_{\boldu_t} \right) \\
\leq& \log \mathbb{E}_p\left[ \prod_i^{n_p} \exp \left( \langle v\boldu_t,\boldf(\boldx_p^{(i)}) - \nabla \hat{A}(\boldtheta^*)\rangle \right)/\exp(\beta v) \right]
\le n_p(\frac{\lambda_{\text{max}}
	bv^2}{2}+v^2 - \beta v)
\end{align*}
Finally, exponentiate both sides and set $v = \frac{\beta}{\lambda_\text{max}b+2} \le 1$, we have
\[
P\left(\langle \boldu_t,\boldw \rangle \ge \beta ~|~ \xi_{\boldu_t} \right) \le \exp\left( - \frac{n_p\beta^2}{ 2(\lambda_\mathrm{max}b + 2)}\right).
\]

Since there exists a sign vector $\tilde{\boldu}$ so that $\langle \tilde{\boldu}, \boldw \rangle = |\boldw_t| \ge \|\boldw_t\| $, therefore, we can conclude:
\[
P\left( \langle \text{sign}(\boldw_t), \boldw_t \rangle \ge \|\boldw_t\|\geq \beta ~|~ \xi_{\boldu_{1, t}}, \xi_{\boldu_{2, t}}, \dots, \xi_{\boldu_{2^b, t}}   \right) \leq \exp\left( - \frac{n_p\beta^2}{ 2(\lambda_\mathrm{max}b + 2)}\right),
\] 

Consider a vector $\boldw$ with $\frac{m^2+m}{2}$ sub-vectors, we can bound $\max_{t\in S \cup S^c} \|\boldw_t\|$ by union bound:
\[
P\left( \max_{t\in S \cup S^c} \|\boldw_t\| \geq \beta ~|~ \forall t,  \xi_{\boldu_{1,t}}, \xi_{\boldu_{2,t}}, \dots, \xi_{\boldu_{2^b,t}}  \right) \leq \exp\left( - \frac{n_p\beta^2}{ 2(\lambda_\mathrm{max}b + 2)} + \log \frac{m^2+m}{2} \right).
\]
When $\beta \ge \sqrt{\frac{4(\lambda_\mathrm{max}b+2)\log \frac{m^2+m}{2}}{n_p}}$ and $\frac{\beta}{\lambda_\mathrm{max}b+2} \le \min(\frac{\|\boldtheta^*\|}{\sqrt{b}},1)$,
\[
P\left( \max_{t\in S \cup S^c} \|\boldw_t\| \geq \beta ~|~ \forall t, \xi_{\boldu_{1,t}}, \xi_{\boldu_{2,t}}, \dots, \xi_{\boldu_{2^b,t}}  \right) \leq \exp\left( -c'n_p \right).
\] 

Let $\beta = \frac{\alpha}{2-\alpha} \frac{\lambda_{n_p}}{4}$, then for
\[
\frac{8 (2-\alpha)}{\alpha} \sqrt\frac{{(\lambda_{\text{max}}b+2)\log \frac{m^2+m}{2}}}{n_p} \le \lambda_{n_p} \le \frac{4 (2-\alpha)}{\alpha} {\left(\lambda_\mathrm{max}b+2\right)}\min\left(\frac{\|\boldtheta^*\|}{\sqrt{b}}, 1\right),
\]
we can have $\max_{t\in S^c \cup S} \|\boldw_t\| \leq \frac{\alpha}{2-\alpha} \frac{\lambda_{n_p}}{4} $ in high probability. 


Moreover, by using the fact that 
\begin{align*}
P(A) &\le P(A|\forall t, \xi_{\boldu_{1,t}}, \xi_{\boldu_{2,t}}, \dots, \xi_{\boldu_{2^b,t}}  ) + \frac{m^2+m}{2}2^b P(\xi_{\boldu_t}^c)\\
&\le P(A|\forall t, \xi_{\boldu_{1,t}}, \xi_{\boldu_{2,t}}, \dots, \xi_{\boldu_{2^b,t}}  ) + (m^2+m)^b P(\xi_{\boldu_t}^c)
\end{align*}
we can obtain
\[
P\left( \max_{t\in S^c \cup S} \|\boldw_t\| \geq \beta  \right) \leq \exp\left( -c'n_p \right) + 4\exp\left(- c''v^4n_q +b \log \left( m^2+m \right) \right).
\]
for constants $c'$ and $c''$. By substituting $v = \frac{\beta}{\lambda_\mathrm{max}b+2}$ and replace $\beta$ with its lower bound, we complete the proof.
\end{proof}

\subsection{Proof of Lemma \ref{lem.3.maintext}}
\label{sec.proof.lemma.3}
The proof Lemma \ref{lem.3.maintext} and Lemma \ref{lem.4.maintext} is in a straightforward fashion following Lemma \ref{lem.2.maintext}.

\begin{proof}
Since we are trying to prove that $\| \hat{\boldtheta}_S - \boldtheta^*_S \| \le B$, where $B>0$, according to \cite{Ravikumar_2010,YangGMGLM}, we may construct the following function:
\begin{align}
\label{eq.G}
G(\bolddelta_S) = \ell(\boldtheta_S^*+\bolddelta_S) - \ell(\boldtheta^*_S) + \lambda_{n_p} \sum_{{t'}\in S} (\|\boldtheta^*_{t'}+\bold\bolddelta_{t'}\|-\|\boldtheta_{t'}^*\|),
\end{align}
where $G$ is a convex function, $G(\boldzero)= 0$, reaches the minimal at $\bolddelta_S^* = \hat{\boldtheta}_S - \boldtheta_S^*$ and $G(\bolddelta_S^*) \le 0$. Simple proof in \cite{Ravikumar_2010} can show that for a $\tilde{\bolddelta}_S, \|\tilde{\bolddelta}_S\|=B$ if $G(\tilde{\bolddelta}_S) > 0$,  $\|\hat{\boldtheta}_S-\boldtheta_S^*\|\le B$.

We first use Taylor expansion on the first two terms of \eqref{eq.G}, 
\begin{align}
\label{eq.g.expand}
G(\bolddelta) = \bolddelta_S^\top\underbrace{\nabla  \ell(\boldtheta_S^*)}_{\boldw_S} + \frac{1}{2}\bolddelta_S^\top\nabla^2  \ell(\boldtheta_S^*+\bar{\bolddelta}_S)\bolddelta_S + \lambda_{n_p} \sum_{t'\in S} (\|\boldtheta^*_{t'}+\bold\bolddelta_{t'}\|-\|\boldtheta_{t'}^*\|).
\end{align}
Since for the first and the last term
\begin{align}
\label{eq.first}
|\langle \boldw_S,\bolddelta_S \rangle| \le \|\boldw_S\| \|\bolddelta_S\| \le \sqrt{d}\|\bolddelta_S\| \max_{t'\in S}\|\boldw_{t'}\| \le  \frac{\sqrt{d}\lambda_{n_p}}{4}  \|\bolddelta_S\| ,
\end{align}
and 
\begin{align}
\label{eq.last}
\lambda_{n_p} \sum_{t'\in S} (\|\boldtheta^*_{t'}+\bolddelta_{t'}\|-\|\boldtheta_{t'}^*\|) \ge -\lambda_{n_p} \sum_{{t'}\in S} \|\bolddelta_{t'}\| \ge - \sqrt{d} \lambda_{n_p} \|\bolddelta_S\|, 
\end{align}
we only need to lower-bound the middle term.

Obviously, we need to lower-bound $\Lambda_{\mathrm{min}}\left(\nabla^2 {\ell (\boldtheta_S^* + \bar{\bolddelta}_S)}\right)$. By applying the Mean-value theorem
\[
\nabla^2 {\ell (\boldtheta_S^* + \bar{\bolddelta}_S)}= \nabla^2 {\ell (\boldtheta_S^*)} + \nabla^3 {\ell (\boldtheta_S^*+ \bar{\bar{\bolddelta}}_S)}\bar{\bolddelta}_S,
\]
Weyl's inequality\citep{Horn1986MatrixAnalysis} implies:
\begin{align*}
\Lambda_{\text{min}}\left(\nabla^2 {\ell (\boldtheta_S^* + \bar{\bolddelta}_S)}\right) &\geq \Lambda_{\mathrm{min}}\left(\nabla^2 {\ell (\boldtheta_S^*)}\right) - \vertiii{\nabla^3 {\ell (\boldtheta_S^*+ \bar{\bar{\bolddelta}}_S)}\bar{\bolddelta}_S } \\
&\geq \Lambda_{\mathrm{min}}\left(\nabla^2 {\ell (\boldtheta_S^*)}\right) - \sqrt{d}\max_{t' \in S}\vertiii{\nabla_{\boldtheta_t}\nabla^2 {\ell (\boldtheta_S^*+ \bar{\bar{\bolddelta}}_S)}}\|\bar{\bolddelta}_S \|,
\end{align*}
Given the setting of $ B \leq \|\boldtheta^*\|$, from Assumption \ref{assum.depen} and \ref{assum.smooth} we can get 
\begin{align}
\label{eq.middle}
\Lambda_{\text{min}}\left(\nabla^2 {\ell (\boldtheta_S^* + \bar{\bolddelta}_S)}\right) \geq \lambda_\mathrm{min} - \sqrt{d}\lambda_{3,\mathrm{max}}\|\bolddelta_S\|.
\end{align}
Combining \eqref{eq.first}, \eqref{eq.middle} and \eqref{eq.last} with \eqref{eq.g.expand}, we can get
\begin{align}
G(\bolddelta_S) \geq -\frac{\lambda_{n_p} \|\bolddelta_S\|\sqrt{d}}{4}  + \frac{\lambda_\mathrm{min} }{4}\|\bolddelta_S\|^2- {\lambda_{n_p}} \sqrt{d}\|\bolddelta_S\|,
\end{align}
by setting 
\begin{align}
\label{eq.lemma3.condition}
\frac{1}{2}\lambda_\mathrm{min} \ge \sqrt{d}\lambda_{3,\mathrm{max}}\|\bolddelta_S\|.
\end{align}

Let $\|\bolddelta_S\| = \sqrt{d}\lambda_{n_p} M$, where $M>0$, we have
\[
G(\bolddelta_S) \geq (\lambda_{n_p})^2 d(-\frac{M}{4}+\lambda_\mathrm{min}\frac{1}{4}M^2-M),
\]
which is strictly positive when $M = \frac{10}{\lambda_\mathrm{min}}$.
Substitute $M = \frac{10}{\lambda_\mathrm{min}}$ to \eqref{eq.lemma3.condition}, we have
\begin{align}
\label{eq.lemma3.condition2}
d\lambda_{n_p} \le \frac{\lambda_\mathrm{min}^2}{20\lambda_{3,\mathrm{max}}}.
\end{align}
Therefore, when $\max_{t\in S \cup S^c} \|\boldw_t\| \le \frac{\lambda_{n_p}}{4}$ and \eqref{eq.lemma3.condition2} holds, we have $\|\boldtheta^* - \hat{\boldtheta}\| \leq \frac{10\lambda_{n_p} \sqrt{d}}{\lambda_\mathrm{min}} $.
\end{proof}
\subsection{Proof of Lemma \ref{lem.4.maintext}}
\label{sec.proof.lemma.4}	
Since
\[
\boldg_t= \left[\nabla_{\boldtheta_t}\nabla \ell(\boldtheta^*) - \nabla_{\boldtheta_t}\nabla \ell(\bar{\boldtheta})\right] \left[\boldtheta^* - \hat{\boldtheta}\right]
\]
by applying the Mean-value theorem, we get
\begin{align*}
\|\boldg_t \|&\le \vertiii{\nabla_{\boldtheta_t}\nabla^2 \ell(\boldtheta^* + \bar{\bar{\boldu}})} \|\boldtheta^* - \bar{\boldtheta}\|\|\boldtheta^* - \hat{\boldtheta}\|\\
&\le   \vertiii{\nabla_{\boldtheta_t}\nabla^2 \ell(\boldtheta^* + \bar{\bar{\boldu}})} \|\boldtheta^* - \hat{\boldtheta}\|^2\\
&\le \frac{100(\lambda_{n_p})^2 d}{\lambda_\mathrm{min}^2} \lambda_{3,\mathrm{max}}.
\end{align*}

Thus, when 
\[
\lambda_{n_p} d \le \frac{\lambda_\mathrm{min}^2}{100\lambda_{3,\mathrm{max}}}\frac{\alpha}{4(2-\alpha)},
\]
we have 
\[
\max_{t\in S \cup S^c} \|\boldg_t\| \le \frac{\alpha \lambda_{n_p}}{4(2-\alpha)}. 
\]
Note that this condition is stronger than the one in Lemma \ref{lem.3.maintext}, so the result of Lemma \ref{lem.3.maintext} holds automatically. 

\subsection{Proof of Lemma \ref{lemma5.maintext}}
\label{sec.proof.lemma.5}
\begin{proof}
We know that
\[
\boldw_t = 
-\left[\frac{1}{n_p}\sum_{i=1}^{n_p} \boldf_t(\boldx^{(i)})\right]+
\left[ \frac{1}{n_q}\sum_{j=1}^{n_q} \hat{r}(\boldx^{(j)};\boldtheta^*) \boldf_t(\boldx^{(j)})\right],
\]
so now we show by using the boundedness of $\|\boldf_t(\boldx)\|$ implied from Assumption \ref{assum.bounded.ratio}, this converges to $0$ in probability.

First we show that $\|\boldw_t\|$ can be upper-bounded by:
\begin{align*}
\|\boldw_t \| &\le \underbrace{\left\| \frac{1}{n_p}\sum_{i=1}^{n_p} \boldf_t(\boldx^{(i)}) -\mathbb{E}_{p}[\boldf_t(\boldx)]\right\|}_{a_{n_p}}
+\underbrace{\left\| \frac{1}{n_q}\sum_{i=1}^{n_q} \hat{r}(\boldx^{(i)};\boldtheta^*) \boldf_t(\boldx^{(i)}) - \frac{1}{n_q}\sum_{j=1}^{n_q} r(\boldx^{(j)};\boldtheta^*) \boldf_t(\boldx^{(j)}) \right\|}_{b_{n_q}}\\
&+\underbrace{\left\| \frac{1}{n_q}\sum_{j=1}^{n_q} r(\boldx^{(j)};\boldtheta^*) \boldf_t(\boldx^{(j)}) - \mathbb{E}_{qr_\boldtheta^*}\left[ \boldf_t(\boldx) \right]\right\|}_{c_{n_q}},
\end{align*}
We now need Hoeffding inequality for norm-bounded vector random variables  which has appeared in previous literatures such as \cite{supportVectorMachines}: For a set of bounded zero-mean vector-valued random variable $ \{\boldy_i\}^n_{i=1}, \|\boldy\|\le c$, we have  
\begin{align*}
P(\left\|\sum_{i=1}^n \boldy_i \right\| \ge n\epsilon)\le 2\exp\left(\frac{-n\epsilon^2}{2c^2}\right),
\end{align*}
for all 
$\epsilon\ge \frac{2c}{\sqrt{n}}.$
Now it is easy to see
\begin{align}
\label{eq.acnq}
P(a_{n_p} \ge \epsilon) \le \exp\left(- \frac{n_p\epsilon^2}{2 C^2_\mathrm{\boldf_t,\mathrm{max}}} \right) ~~~
P(c_{n_q} \ge \epsilon) \le \exp\left(- \frac{n_q\epsilon^2}{2C^2_\mathrm{max} C^2_\mathrm{\boldf_t,\mathrm{max}}} \right),
\end{align}
as long as 
\begin{align}
\label{lem2.hoeffding}
	\epsilon \ge \frac{2C_\mathrm{max}C_\mathrm{\boldf_t,\mathrm{max}}}
	{\sqrt{\min(n_p, n_q)}}.
\end{align}
As to $b_{n_q}$, it can be upper-bounded by 
\begin{align*}
b_{n_q} &= \left\| \frac{1}{n_q}\sum_{i=1}^{n_q} \hat{r}(\boldx^{(i)};\boldtheta^*) \boldf_t(\boldx^{(i)}) - \frac{1}{n_q}\sum_{j=1}^{n_q} r(\boldx^{(j)};\boldtheta^*) \boldf_t(\boldx^{(j)}) \right\|\\
&=\left\|\frac{\hat{N}(\boldtheta^*)}{N(\boldtheta^*)} \frac{1}{n_q}\sum_{i=1}^{n_q} \hat{r}(\boldx^{(i)};\boldtheta^*) \boldf_t(\boldx^{(i)}) - \frac{1}{n_q}\sum_{j=1}^{n_q} \hat{r}(\boldx^{(j)};\boldtheta^*) \boldf_t(\boldx^{(j)}) \right\|\\
&\le\left\|\frac{1}{n_q}\sum_{i=1}^{n_q} \hat{r}(\boldx^{(i)};\boldtheta^*) \boldf_t(\boldx^{(i)})\right\|\cdot \left\|\frac{\hat{N}(\boldtheta^*)}{N(\boldtheta^*)} - 1 \right\|\\
&\le C_\mathrm{ratio} C_\mathrm{\boldf_t,\mathrm{max}} \left|\frac{1}{n_q}\sum_{i=1}^{n_q} r(\boldx^{(i)};\boldtheta^*) - 1\right|,
\end{align*}
and using regular Hoeffding-inequality we may obtain:
\begin{align}
\label{eq.bnq}
P(b_{n_q} > \epsilon) < 2\exp\left(- \frac{2n_q\epsilon^2}{C^2_\mathrm{max}C^2_\mathrm{ratio} C^2_\mathrm{\boldf_t,\mathrm{max}}} \right).
\end{align}
Therefore, combining \eqref{eq.acnq} and \eqref{eq.bnq}:
\begin{align*}
P(\|\boldw_t\| \ge 3\epsilon) \le P( a_{n_p}+b_{n_q}+c_{n_q}\ge 3\epsilon) & \le 4\exp\left(- \frac{\min(n_p,n_q)\epsilon^2}{c'} \right),
\end{align*}
where $c'$ is a constant defined as $c' = \max \left(\frac{1}{2}C^2_\mathrm{max}C^2_\mathrm{ratio} C^2_\mathrm{\boldf_t,\mathrm{max}}, 2C^2_\mathrm{max} C^2_\mathrm{\boldf_t,\mathrm{max}}, 2 C^2_\mathrm{\boldf_t,\mathrm{max}}\right)$.
Applying the union-bound for all $t\in S\cup S^c$,
\begin{align*}
P(\max_{t \in S\cup S^c} \|\boldw\| \ge 3\epsilon) \le  2(m^2+m)\exp\left(- \frac{\min(n_p,n_q)\epsilon^2}{c'} \right),
\end{align*}
\begin{align*}
P\left(\max_{t \in S\cup S^c} \|\boldw\| \ge \frac{\alpha \lambda_{n_p, n_q}}{4(2-\alpha)}\right) \le  2(m^2+m)\exp\left(- \frac{\min(n_p,n_q)}{c'} \left(\frac{\alpha\lambda_{n_p, n_q}}{12(2-\alpha)}\right)^2 \right),
\end{align*}
and when 
$\lambda_{n_p, n_q} \ge \frac{24(2-\alpha)}{\alpha}\sqrt{\frac{c'\log \frac{m^2+m}{2}}{\min(n_p, n_q)}}$,
\begin{align*}
P\left(\max_{t \in S\cup S^c} \|\boldw\| \ge \frac{\alpha \lambda_{n_p, n_q}}{4(2-\alpha)} \right)\le  4\exp\left(-c''\min(n_p,n_q) \right),
\end{align*}
where $c''$ is a constant.
Assume that $\log \frac{m^2+m}{2} > 1$ and we set $\lambda_{n_p, n_q}$ as 
\[
\lambda_{n_p, n_q} \ge \frac{24(2-\alpha)}{\alpha}\sqrt{\frac{(c'+4C^2_\mathrm{max}C^2_\mathrm{\boldf_t,\mathrm{max}})\log \frac{m^2+m}{2}}{\min(n_p, n_q)}},
\]
thus \eqref{lem2.hoeffding}, the condition of using vector Hoeffding-inequality is satisfied.
\end{proof}
\subsection{Proof of Corollaries}
\label{sec.proof.corol}
\paragraph*{Proof of Corollary \ref{corol.relaxed}}  
Consider an event $\xi'_{\boldu_t}$ which is slightly different from \eqref{eq.event.xi}:
\[
\xi'_{\boldu_t} :=\left\{ \left[A(\boldtheta^*+ v\boldu_t) - A(\boldtheta^*) -\left( \hat{A}(\boldtheta^* +  v\boldu_t)-\hat{A}(\boldtheta^*)\right)\right] \leq n_p^\frac{1}{4}v^2\right\}.
\]
Note that under the new Assumption \ref{assum.smooth.ratiomodel}, we will have 
\[
P\left( {\xi^{'c}_{\boldu_t}}\right) \le \exp\left(-\frac{n_p^\frac{1}{2}v^4n_q}{200d}\right)
\]
given $n_p^{\frac{1}{4}}v^2 \le 1$. Note here the extra $d$ is needed on the denominator of RHS. The derivation is the same as what has been used in Proposition \ref{prop.norm.bound}. 

Then the following holds conditioned on $\xi'_{\boldu_t}$:
\begin{align*}
\log\mathbb{E}_{q^*}\left[\exp\left(\langle v\boldu_t,\boldf(\boldx^{(i)}) - \nabla \hat{A}(\boldtheta^*)\rangle\right)\right] \leq\hat{A}(\boldtheta^* + v\boldu_t ) - \hat{A}(\boldtheta^*) - \langle v\boldu_t,\nabla \hat{A}(\boldtheta^*)  \rangle + n_p^{\frac{1}{4}}v^2.
\end{align*} 
Since 
\begin{align*}
\hat{A}(\boldtheta^* + v\boldu_t ) - \hat{A}(\boldtheta^*) - \langle v\boldu_t,\nabla \hat{A}(\boldtheta^*)  \rangle + n_p^{\frac{1}{4}}v^2
&\leq \frac{\lambda_{\text{max}}bv^2}{2} + n_p^{\frac{1}{4}}v^2\\
&\leq  n_p^{\frac{1}{4}}\left(\frac{\lambda_{\text{max}}bv^2}{2} + v^2 \right),
\end{align*}
by following the similar techniques in Lemma \ref{lem.2.maintext}, it can be shown that \[
P\left(\langle \boldu_t,\boldw \rangle \ge \beta ~|~ \xi_{\boldu_t} \right) \le 2\exp\left( - \frac{n_p^{\frac{3}{4}}\beta^2}{ 2(\lambda_\mathrm{max}b + 2)}\right).
\]

Using the same derivation in Lemma \ref{lem.2.maintext}, we will reach the conclusion that 
for
\[
\frac{8 (2-\alpha)}{\alpha} \sqrt\frac{{(\lambda_{\text{max}}b+2)\log \frac{m^2+m}{2}}}{n_p^{\frac{3}{4}}} \le \lambda_{n_p} \le \frac{4 (2-\alpha)}{\alpha} \left(\lambda_\mathrm{max}b+2\right)\min(\frac{\|\boldtheta^*\|}{\sqrt{b}},\frac{1}{n_p^{1/8}}),
\]
we can have $\max_{t\in S \cup S^c} \|\boldw_t\| \leq \frac{\alpha}{2-\alpha} \frac{\lambda_{n_p}}{4} $ with high probability.  

Again, by using the fact that 
\[
P(A) \le P(A|\forall t, \xi'_{\boldu_{1,t}}, \xi'_{\boldu_{2,t}}, \dots, \xi'_{\boldu_{2^b,t}}  ) + (m^2+m)^b P(\xi_{\boldu_t}^{'c}),
\]
we can obtain
\[
P\left( \max_{t\in S \cup S^c} \|\boldw_t\| \geq \beta  \right) \leq \exp\left( -c'n_p \right) + 4\exp\left(- c''n_p^\frac{1}{2}v^4n_q/d +b \log (m^2+m)\right),
\]
for constants $c'$ and $c''$. By substituting $v = \frac{\beta}{\lambda_\mathrm{max}b+2}$ and replace $\beta$ with its lower bound, we get:
\[
P\left( \max_{t\in S \cup S^c} \|\boldw_t\| \geq \beta  \right) \leq \exp\left( -c'n_p \right) + 4\exp\left(- c'''n_q \frac{ \left(\log \frac{m^2+m}{2}\right) ^{2}}{dn_p} +b \log (m^2+m)\right).
\]

\paragraph*{Proof of Corollary \ref{corol.compex.model}}
For basis function $\boldpsi(\boldx)$ where at most $s$ fixed-location elements are non-zero (irrelevant to $\boldx$), 
\[
\langle \boldu'_t, \boldw\rangle =  \frac{1}{n_p}\sum_{i=1}^{n_p}\langle \boldu'_t, \boldf(\boldx^{(i)}) - \frac{1}{n_q}\sum_{j=1}^{n_q}\hat{r}(\boldx^{(j)};\boldtheta^*) \boldf(\boldx^{(j)})\rangle,
\] 
may have at most $s$ non-zero elements. Therefore, consider a sign vector $\boldu'_t$ with $s$ non-zero elements is sufficient, i.e. $\|\boldu'_t\| = \sqrt{s}$. Instead of \eqref{eq.emperical.upperbound}, we have the following inequality:
\begin{align*}
\hat{A}(\boldtheta^* + v\boldu'_t ) - \hat{A}(\boldtheta^*) - \langle v\boldu'_t,\nabla \hat{A}(\boldtheta^*)  \rangle + v^2 
&\leq \frac{\lambda_{\text{max}}sv^2}{2} + v^2\\
&\leq \frac{\lambda_{\text{max}}sv^2}{2} + sv^2.
\end{align*}
Follow the derivation of Lemma \ref{lem.2.maintext}, we may get the boundedness of $\lambda_{n_p}$ as stated in Corollary \ref{corol.compex.model}. 

Furthermore, for a sign vector only allowed to take values on a length $b$ sub-vector, where only $s$ elements can be non-zero, the number of possibilities is $ {b \choose s} 2^s$, and after applying the union bound
\[
P(A) \le P(A|\forall t, \xi_{\boldu'_{1,t}}, \xi_{\boldu'_{2,t}}, \dots, \xi_{\boldu'_{{b \choose s}2^s,t}}) + \frac{m^2+m}{2}{b \choose s}2^s P(\xi_{\boldu'_t}^c),
\]	
we can obtain
\begin{align*}
P\left( \max_{t\in S \cup S^c} \|\boldw_t\| \geq \beta  \right) &\leq \exp\left( -c'n_p \right) + 4\exp\left(- c''dv^4n_q +s \log 2 + \log {b \choose s} +\log \frac{m^2+m}{2} \right)\\
&\leq \exp\left( -c'n_p \right) + 4\exp\left\{- c''dv^4n_q +s \log \left[(m^2+m) {b \choose s}\right]\right\}.
\end{align*}

\subsection{Proof of Proposition \ref{prop.bounded.2nd}}
\label{sec.proof.prop.2nd}
\begin{proof}
Let's define the $n$-input log-sum-exp function as $\mathrm{LSE}(a_1, a_2, \cdots a_n)$, then the sample Fisher information matrix, $\nabla_{\boldtheta}^2 \ell(\boldtheta)$ can be written as 
\[
\nabla_{\boldtheta}^2 \ell(\boldtheta) = \nabla_{\boldtheta}^2 \mathrm{LSE}(\boldtheta^\top \boldf(\boldx^{(1)}), \boldtheta^\top \boldf(\boldx^{(2)}), \cdots \boldtheta^\top \boldf(\boldx^{(n_q)})).
\]
Using the chain-rule and matrix notation, we may write the second-order derivative as 
\[
\nabla_{\boldtheta}^2 \ell(\boldtheta) = F \nabla_{\boldg}^2 \mathrm{LSE}\left(g_1, g_2, \cdots, g_{n_q}\right)F^\top,
\]
where $F$ is a $b(m^2+m)/2 \times n_q$ matrix that each column $F_i = \boldf\left(\boldx_q^{(i)}\right)$, and $g_i = \boldtheta^\top \boldf\left(\boldx_q^{(i)}\right)$. Therefore, 
\begin{align}
	\label{eq.L.d2}
	\normDtwo \le \|FF^\top\| \|\DtwoLSE\|.
\end{align}

The second-order derivative of log-sum-exp function can be found in previous literatures \citep{book:Boyd+Vandenberghe:2004}, and is a positive semi-definite matrix. 
\[
\DtwoLSE = \frac{1}{s^2}\left(s \cdot \mathrm{diag}(\bolde) - \bolde \bolde^\top\right),
\]
where $\bolde$ is a column vector, $e_i = \exp\left(g_i\right), s = \sum_{i=1}^{n_q} e_i$ and the operator ``$\mathrm{diag}$'' of a vector $\bolde$ with length $n_q$ is to set the elements of $\bolde$ as the diagonal of a $n_q \times n_q$ diagnoal-matrix.

It can be observed that the elements in $\DtwoLSE$ are all strictly negative except the diagonal elements, which are the negative sums of the elements in the same column. $\DtwoLSE$ is a \emph{Laplacian matrix} \citep{merris1994laplacian} of a fully connected graph, with strictly positive weights. The upper-bound of maximum eigenvalue of a Laplacian matrix has been studied in great detail, such as Corollary 3.2, \citep{LargestEigenvalueWeightedGraph}. Restate using our notation:
\begin{align}
	\label{eq.laplacian.trick}
	\left\|\DtwoLSE\right\|& \le \max_{i,j, i\neq j} \left(\frac{\sum_{k\neq i} e_k e_i}{s^2} + \frac{\sum_{k\neq j} e_k e_j}{s^2}\right)\notag\\
	&\le \max_{i} \left(\frac{2e_i}{s}\right) = \max_{i} \frac{2\hat{r}\left(\boldx^{(i)};\boldtheta\right)}{n_q}\le  \frac{2C_\mathrm{ratio}}{n_q}.
\end{align}
The last line is due to Assumption \ref{assum.bounded.ratio} given $\boldtheta = \boldtheta^* + \bolddelta$, and $\|\bolddelta\| \le \|\boldtheta^*\|$.
It can be seen immediately that as long as 
$
\frac{1}{n_q}\|FF^\top\|
$
is bounded, then $\normDtwo$ is bounded. 
\end{proof}
\subsection{Proof of Proposition \ref{prop.bounded.3rd}}
\label{sec.proof.prop.3nd}
\begin{proof}
To simplify the notation, we actually prove the boundedness of $\|\nabla_{\theta_k}\nabla^2 \ell(\boldtheta)\|$. Using the exact same set of assumptions, we can prove $\max_{t\in S\cup S^c}\vertiii{\nabla_{\boldtheta_t}\nabla^2 \ell(\boldtheta)}$ is bounded using the same strategy.

It can be shown that $\|\nabla_{\theta_k}\nabla^2 \ell(\boldtheta)\| = \sum_{i=1}^{n_q} F\nabla_{g_i}\nabla_{\boldg}^2 \mathrm{LSE}(g_1, g_2, \cdots, g_{n_q}) F^\top f_k(\boldx^{(i)})$.

From multiplicative rule, we have
\begin{align*}
\nabla_{g_i}&\nabla_{\boldg}^2 \mathrm{LSE}(g_1, g_2, \cdots, g_{n_q})  = \left(\nabla_{g_i} \frac{1}{s^2}\right) \cdot \left(s \cdot \mathrm{diag}(\bolde) - \bolde^\top \bolde\right) + \frac{1}{s^2} \cdot D(i),
\end{align*}
where $\left(\nabla_{g_i} \frac{1}{s^2}\right) = -\frac{2e_i}{s^3}$ and  $D(i)=\nabla_{g_i}\left(s \cdot \mathrm{diag}(\bolde) - \bolde^\top \bolde\right)$ is still a Laplacian matrix whose $j,k$-th elements can be written as 
\begin{align*}
D_{j,k}(i) = \begin{cases}
-e_je_k &j=i,k\neq i, \text{ or }j\neq i, k=i\\
0 & j\neq i, k\neq i, j \neq k\\
- \sum_{j} D_{j \neq k,k}(i) & j = k.
\end{cases}
\end{align*}
Similarly to the bounding techniques used in \eqref{eq.laplacian.trick}, we have
\begin{align}
\left\|  \left(\nabla_{g_i}\frac{1}{s^2}\right) \cdot \left(s \cdot \mathrm{diag}(\bolde) - \bolde^\top \bolde\right)\right\|  &\le \left|\frac{2e_i}{s}\right|\left\| \DtwoLSE \right\| \notag\\
&\le \frac{2\hat{r}(\boldx^{(i)};\boldtheta)}{n_q} \max_{j} \frac{2\hat{r}(\boldx^{(j)};\boldtheta)}{n_q}  \le \frac{4C^2_\mathrm{ratio}}{n^2_q}, \label{eq.bounding.3rdorder2}
\end{align}
given $\boldtheta = \boldtheta^* + \bolddelta$, and $\|\bolddelta\| \le \|\boldtheta^*\|$.
As a result,
\begin{align}
\label{eq.d3.di}
&\left\|\sum_{i=1}^{n_q} F\left[\left(\nabla_{g_i}\frac{1}{s^2}\right) \cdot \left(s \cdot \mathrm{diag}(\bolde) - \bolde^\top \bolde\right)\right] F^\top f_k(\boldx^{(i)})\right\| \notag\\
\le& \left(\frac{C_\mathrm{ratio}}{n_q}\sum_{i=1}^{n_q} \left|f_k(\boldx^{(i)})\right| \right) \left(\frac{4C_\mathrm{ratio}}{n_q}\left\| FF^\top\right\|\right) \le 4C^2_\mathrm{ratio} D_\mathrm{max, 1}D_\mathrm{max,2}.
\end{align}

By the construction of $D$, we can see that,  
\begin{align*}
f_k(\boldx^{(i)}) F D(i) F^\top  &= 
-f_k(\boldx^{(i)}) \sum_{j \neq i}  |D_{i,j}(i)| \left(F_j - F_i\right) \left(F_j - F_i\right)^\top
\end{align*}
Therefore, 
\begin{align}
\left\|\frac{1}{s^2}\sum_{i=1}^{n_q}f_k(\boldx^{(i)}) F D(i) F^\top \right\| &\le \frac{\max_{i,j, j\neq i}|D_{i,j}(i)|}{s^2}  \sum_{i=1}^{n_q} \left\|f_k(\boldx^{(i)})\right\| \left\|\sum_{j=1, j \neq i}^{n_q} \left(F_j - F_i\right) \left(F_j - F_i\right)^\top \right\| \notag\\
&\le \frac{2C^2_\mathrm{ratio}D_{\mathrm{2,max}}}{n_q} \sum_{i=1}^{n_q} \left\|f_k(\boldx^{(i)})\right\| \le 2C^2_\mathrm{ratio}D_{\mathrm{1,max}} D_{\mathrm{2,max}} \label{eq.d3.nodi}.
\end{align}
by using the fact that $\frac{\max_{i,j, i\neq j}|D_{i,j}(i)|}{s^2} \le \left(\frac{\max_i \hat{r}(\boldx^{(i)};\boldtheta^*+\bolddelta)}{n_q}\right)^2\le \frac{C^2_\mathrm{ratio}}{n_q^2}$.

Therefore, by combining \eqref{eq.d3.di} and \eqref{eq.d3.nodi} we can conclude that for $\boldtheta = \boldtheta^* + \bolddelta$ and $\bolddelta \le \|\boldtheta^*\|$,
\[
\left\|\nabla_{\theta_k}\nabla^2 \ell(\boldtheta)\right\| \le
6C^2_{\mathrm{ratio}} D_\mathrm{max,1}D_\mathrm{max,2}.
\]
\end{proof}
\subsection{Proof of Proposition \ref{prop.bounded.min.2nd}}
\label{sec.proof.prop.min.2nd}
\begin{proof}
	Recall \[
	\nabla_{\boldtheta_S}^2 \ell(\boldtheta)  = F_S \left[ \frac{1}{s^2}\left(s \cdot \mathrm{diag}(\bolde) -  \frac{1}{s^2} \bolde^\top \bolde \right) \right] F_S^\top .
	\]
	Since the Laplacian matrix in the middle always has the smallest eigenvalue 0, while $F_SF_S^\top$ is always positive semi-definite, we are guaranteed to have a trivial lower bound of the smallest eigenvalue 0 without utilizing any assumptions. 
	However, by the construction of the Laplacian we can see
	\begin{align*}
	\nabla_{\boldtheta_S}^2 \ell(\boldtheta) &= 
	\sum_{i,j = 1, i\neq j}^{n_q} F_{S,i} F_{S,j}^\top \frac{e_i e_j}{s^2} - \sum_{i=1}^{n_q} F_{S,i} F_{S,i}^\top\sum_{j=1, j\neq i}^{n_q}
	\frac{e_i e_j}{s^2}\\
	&= 
	\frac{1}{2}\sum_{i,j = 1, i\neq j}^{n_q} \left(F_{S,i} F_{S,j}^\top -  F_{S,i} F_{S,i}^\top + 
	F_{S,j}F_{S,i}^\top -  F_{S,j} F_{S,j}^\top \right)
	\frac{e_i e_j}{s^2}\\
	&= -\frac{1}{2} \sum_{i,j=1, i\neq j}^{n_q} (F_{S,i} - F_{S,j}) (F_{S,i} - F_{S,j})^\top \frac{e_i e_j}{s^2}.
	\end{align*}
	Therefore 
	\begin{align*}
	&\Lambda_\mathrm{min}\left(\nabla_{\boldtheta_S}^2 \ell(\boldtheta)\right)=\Lambda_\mathrm{min}\left[\frac{1}{2}\sum_{i,j = 1, i\neq j}^{n_q} (F_{S,i} - F_{S,j}) (F_{S,i} - F_{S,j})^\top \frac{e_i e_j}{s^2}\right]\\
	\end{align*}
	Note that 
	$
	(F_{S,i} - F_{S,j}) (F_{S,i} - F_{S,j})^\top
	$
	is always positive semi-definite.
	By Assumption \ref{assum.bounded.ratio}, $\frac{e_ie_j}{s^2} \ge \left(\frac{\min_i \hat{r}(\boldtheta^*+\bolddelta)}{n_q}\right)^2\ge\left( \frac{1}{C_\mathrm{ratio}n_q} \right)^2$, if $\boldtheta = \boldtheta^* + \bolddelta$, $\|\bolddelta \| \le \|\boldtheta^*\|$. Therefore
	\begin{align*}
	\Lambda_\mathrm{min}\left[\nabla^2_{\boldtheta_S} \ell(\boldtheta)\right]&\ge \frac{1}{C^2_\mathrm{ratio}}\Lambda_{\mathrm{min}}\left[\frac{1}{2n^2_q}\sum_{i,j = 1, i\neq j}^{n_q} (F_{S,i} - F_{S,j}) (F_{S,i} - F_{S,j})^\top \right]\\
	&\ge \frac{1}{C^2_\mathrm{ratio}} \Lambda_\mathrm{min}\left[\widehat{\mathrm{Cov}}_q \left[ \boldf_S(\boldx)\right]\right]\\
	&\ge \frac{D_\mathrm{min,2}}{C^2_\mathrm{ratio}} > 0.
	\end{align*}

	\end{proof}
\subsection{Proof of Proposition \ref{prop.bounded.min}}
\label{sec.proof.bounded.min}
\begin{proof}
	\begin{align*}
	\mathbb{E}_{q}\left[ \inf_{\bolddelta \in \mathbb{R}^{\mathrm{dim}(\boldtheta^*)}: \|\bolddelta\|
		\leq \|\boldtheta^*\|} r(x, \boldtheta^*+\bolddelta)\right] &= \mathbb{E}_{q}\left[ r(x, \boldtheta^*) + \inf_{\bolddelta \in \mathbb{R}^{\mathrm{dim}(\boldtheta^*)}: \|\bolddelta\|
		\leq \|\boldtheta^*\|} r(x, \boldtheta^*+\bolddelta) - r(x, \boldtheta^*)\right] \\
	&= \mathbb{E}_{q}\left[ r(x, \boldtheta^*) \right] +  \mathbb{E}_{q}\left[\inf_{\bolddelta \in \mathbb{R}^{\mathrm{dim}(\boldtheta^*)}: \|\bolddelta\|
		\leq \|\boldtheta^*\|} r(x, \boldtheta^*+\bolddelta) - r(x, \boldtheta^*)\right]\\
	&\ge \mathbb{E}_{q}\left[ r(x, \boldtheta^*) \right] -  \mathbb{E}_{q}\left[ \sup_{\bolddelta \in \mathbb{R}^{\mathrm{dim}(\boldtheta^*)}: \|\bolddelta\|
		\leq \|\boldtheta^*\|} \|\bolddelta\|\|\nabla r(\boldx;\boldtheta^* + \bolddelta)\| \right]\\
	&\ge \mathbb{E}_{q}\left[ r(x, \boldtheta^*) \right] -  \|\boldtheta^*\| \mathbb{E}_{q}\left[ \sup_{\bolddelta \in \mathbb{R}^{\mathrm{dim}(\boldtheta^*)}: \|\bolddelta\|
		\leq \|\boldtheta^*\|} \|\nabla r(\boldx;\boldtheta^* + \bolddelta)\| \right]\\
	&\ge 1 - c,
	\end{align*}
\end{proof}
\subsection{Proof of Proposition \ref{prop.bounded.ratio.model.uniform}}
\label{sec.proof.model.uniform}
\begin{proof}
	Since \[\hat{r}(\boldx;\boldtheta^*+\bolddelta) \le \frac{C_\mathrm{max}}{\frac{1}{n_q}\sum_{i=1}^{n_q}r(\boldx^{(i)};\boldtheta^*+\bolddelta)},\]
	we just need to show 
	\[
	\inf_{\bolddelta \in \mathbb{R}^{\mathrm{dim}(\boldtheta^*)}: \|\bolddelta\|
		\leq \|\boldtheta^*\|} \frac{1}{n_q}\sum_{i=1}^{n_q}r(\boldx^{(i)};\boldtheta^*+\bolddelta) > 1-c-\epsilon, \text{ for all } 0<\epsilon <1-c
	\]
	holds with high probability. So we bound the tail probability of this event
	\begin{align*}
	&P\left(\inf_{\bolddelta \in \mathbb{R}^{\mathrm{dim}(\boldtheta^*)}: \|\bolddelta\|
		\leq \|\boldtheta^*\|} \frac{1}{n_q}\sum_{i=1}^{n_q}r(\boldx^{(i)};\boldtheta^*+\bolddelta) -(1-c) < -\epsilon \right) \\
	\le &P\left(\frac{1}{n_q}\sum_{i=1}^{n_q} \inf_{\bolddelta \in \mathbb{R}^{\mathrm{dim}(\boldtheta^*)}: \|\bolddelta\|
		\leq \|\boldtheta^*\|}  r(\boldx^{(i)};\boldtheta^*+\bolddelta) - \mathbb{E}_q \left[\inf_{\bolddelta \in \mathbb{R}^{\mathrm{dim}(\boldtheta^*)}: \|\bolddelta\|
		\leq \|\boldtheta^*\|} r(\boldx, \boldtheta^* + \bolddelta) \right]< -\epsilon \right)
	\end{align*}
	
	The rest is simply based on the concentration property of the sums of bounded zero-mean random variable $\inf_\bolddelta r(\boldx;\boldtheta^* + \bolddelta) - \mathbb{E}_q \left[\inf_{\bolddelta} r(\boldx, \boldtheta^* + \bolddelta)\right]$. Now we can conclude
	\begin{align*}
	P\left(\inf_{\bolddelta \in \mathbb{R}^{\mathrm{dim}(\boldtheta^*)}: \|\bolddelta\|
		\leq \|\boldtheta^*\|} \frac{1}{n_q}\sum_{i=1}^{n_q}r(\boldx^{(i)};\boldtheta^*+\bolddelta) -(1-c) < -\epsilon \right) 
	\le \exp\left(-\frac{2n_q\epsilon^2}{C_\mathrm{max}^2}\right), 
	\end{align*}
	then we have $\sup_{\bolddelta \in \mathbb{R}^{\mathrm{dim}(\boldtheta^*)}: \|\bolddelta\|
		\leq \|\boldtheta^*\|} \hat{r}(\boldx;\boldtheta^*+\bolddelta) \le \frac{C_\mathrm{max}}{1-c-\epsilon} = C'_\mathrm{ratio}, \text{ for all } 0<\epsilon<1-c$ holds with high probability.
\end{proof}
\subsection{Proof of Proposition \ref{prop.bounded.normalization.ratio}}
\label{sec.prop.normalization.ratio}
	\begin{proof}
		The boundedness of zero-mean random variable $r(\boldx;\boldtheta^*) - 1$ guarantees its exponentially decaying tail behavior via standard Hoeffding inequality, i.e.,
		\begin{align*}
		P\left( \left|\frac{1}{n_q}\sum_{i=1}^{n_q} r(\boldx^{(i)};\boldtheta^*) -1 \right| > \epsilon\right) \le 2\exp \left( -\frac{2n_q \epsilon^2}{C_\mathrm{max}^2}\right).
		\end{align*}
		\end{proof}
\subsection{Proof of Proposition \ref{prop.bounded.min.lambda}}
\label{sec.prop.min.lambda}
	\begin{proof}
		\begin{align*}
		\Lambda_\mathrm{min}\left(\widehat{\mathrm{Cov}}_{qr_{\boldtheta^*}}\left[\boldf_S(\boldx)\right]\right)
		=
		&\Lambda_\mathrm{min}\left(\frac{1}{n_q}\sum_{i=1}^{n_q}\frac{1}{n_q}\sum_{j=1, j\neq i}^{n_q} r(\boldx^{(i)};\boldtheta^*)r(\boldx^{(j)};\boldtheta^*) C(i,j)\right)\\
		=&\Lambda_\mathrm{min}\left(\widehat{\mathrm{Cov}}_{q\hat{r}_{\boldtheta^*}}\left[\boldf_S(\boldx)\right]\right)\left(\frac{\hat{N}(\boldtheta^*)}{N(\boldtheta^*)}\right)^2\\
		\le
		&\Lambda_\mathrm{min}\left(\widehat{\mathrm{Cov}}_{q\hat{r}_{\boldtheta^*}}\left[\boldf_S(\boldx)\right]\right)(1+\epsilon)^2,
		\end{align*}
		where $C(i,j)$ is the short-hand for
		\begin{align*}
		C(i,j)=\left(\boldf_S(\boldx^{(i)})-\boldf_S(\boldx^{(j)})\right)\left(\boldf_S(\boldx^{(i)})-\boldf_S(\boldx^{(j)})\right)^\top,
		\end{align*}
		and last line is due to Proposition \ref{prop.bounded.normalization.ratio} and holds with high probability.	
		
		Since $ \nabla^2_{\boldtheta_S} \ell(\boldtheta^*) = \widehat{\mathrm{Cov}}_{q\hat{r}_{\boldtheta^*}}\left[\boldf_S(\boldx)\right]$, we conclude the proof. 
		\end{proof}
	\section{The Derivatives of Negative KLIEP Log-likelihood}
	 \begin{prop}
	 \label{prop.is.moments}
		 	The second-order derivative $\nabla^2_\boldtheta \ell(\boldtheta)$ is the importance sampled \textbf{covariance} using samples from $Q$. Moreover, the third-order derivative $\nabla_{\boldtheta_t}\nabla^2_\boldtheta \ell(\boldtheta)$ is the importance-sampled \textbf{skewness} using samples from $Q$:
		 	\begin{align*}
		 	\nabla^2_\boldtheta \ell(\boldtheta) =\widehat{\mathrm{Cov}}_{q \hat{r}_\boldtheta} \left[ \boldf(\boldx)\right]&= \hat{\mathbb{E}}_q \left[ \hat{r}(\boldx; \boldtheta) \boldf(\boldx) \boldf(\boldx)^\top\right]\\
		 	&- \left\{\hat{\mathbb{E}}_q \left[ \hat{r}(\boldx; \boldtheta) \boldf(\boldx)\right]	\right\} \left\{\hat{\mathbb{E}}_q \left[ \hat{r}(\boldx; \boldtheta) \boldf(\boldx)\right]	\right\}^\top
		 	\end{align*}
		 	\begin{align*}
		 	\nabla_{\theta_k}\nabla^2_\boldtheta \ell(\boldtheta) &= 2\hat{\mathbb{E}}_q\left[\hat{r}(\boldx; \boldtheta)f_k(\boldx)\boldf(\boldx)\boldf(\boldx)^\top\right]\\
		 	& -\hat{\mathbb{E}}_q\left[\hat{r}(\boldx; \boldtheta)f_k(\boldx)\boldf(\boldx)\right]\hat{\mathbb{E}}_q\left[\hat{r}(\boldx; \boldtheta)\boldf(\boldx)^\top\right]\\
		 	& -\hat{\mathbb{E}}_q\left[\hat{r}(\boldx; \boldtheta)\boldf(\boldx)\right]\hat{\mathbb{E}}_q\left[\hat{r}(\boldx; \boldtheta)f_k(\boldx)\boldf(\boldx)^\top\right]\\
		 	&- \hat{\mathbb{E}}_q \left[ \hat{r}(\boldx; \boldtheta) f_k(\boldx) \widehat{\mathrm{Cov}}_{q \hat{r}_\boldtheta} \left[ \boldf(\boldx)\right]\right],
		 	\end{align*}
		 	where $\theta_k$ is a scalar.
		 	\end{prop} 

\section{Experimental Settings of Success Rate Plots}
\label{sec.exp.settings}
Define two edge sets for MN $P$ and $Q$ as $E_p$ and $E_q$ separately.
In experiments with Gaussian distributions, densities are constructed as  
\begin{align*}
p(\boldx) \propto \exp\left(-\sum_{u=1}^{m} \theta_0 x^2_u - \sum_{(u,v) \in E_p}^{m} \theta_1x_vx_u\right)
\end{align*}
where $\theta_0 = 2$, $\theta_1 = -0.4$. $E_q$ are set to be the same as $E_p$ while for $d$ randomly picked edges, $\theta_1 = 0.4$.

As to the ``8-shaped'' distribution, its pairwise potentials are introduced in Figure \ref{fig.contour}. We set $\theta_0=1,\theta_1=5$. $E_q$ is constructed from $E_p$ with $d$ randomly removed edges.

\section{Experimental Settings of Comparison with Differential Network Leanring}
\label{sec.exp.diff}
We adopt the True Positive (TP) and True Negative (TN) rate as described in \cite{zhao2014direct}:
\begin{align*}
\mathrm{TPR} = \frac{\sum_{t'\in S} \delta(\hat{\boldtheta}_{
		t'} \neq \boldzero)}{\sum_{t'\in S} \delta(\boldtheta^*_{t'} \neq \boldzero)}, 	~~\mathrm{TNR} = \frac{\sum_{t'' \in S^c} \delta(\hat{\boldtheta}_{t''} = \boldzero)}{\sum_{t''\in S^c} \delta(\boldtheta^*_{t''} = \boldzero)},
\end{align*}
where $\delta$ is the indicator function. For KLIEP, we compute a sequence of TPR and TNR from $\hat{\boldtheta}(\lambda)$ by changing the regularization parameter $\lambda$. For differential network learning, we obtain $\boldDelta(\epsilon)$ for different choices of $\epsilon$, and compute TPR and TNR by varying the threshold $\tau$ for each $\boldDelta(\epsilon)$. All TPR and TNR values are averaged over at least 25 random trials. Then the curve corresponds to $\epsilon$ that maximizes the Area Under the Curve (AUC) is plotted.

\section{Bootstrap of Gene Expression Dataset}
\label{sec.bootstrap}
\begin{figure*}
\includegraphics[width=.99\textwidth]{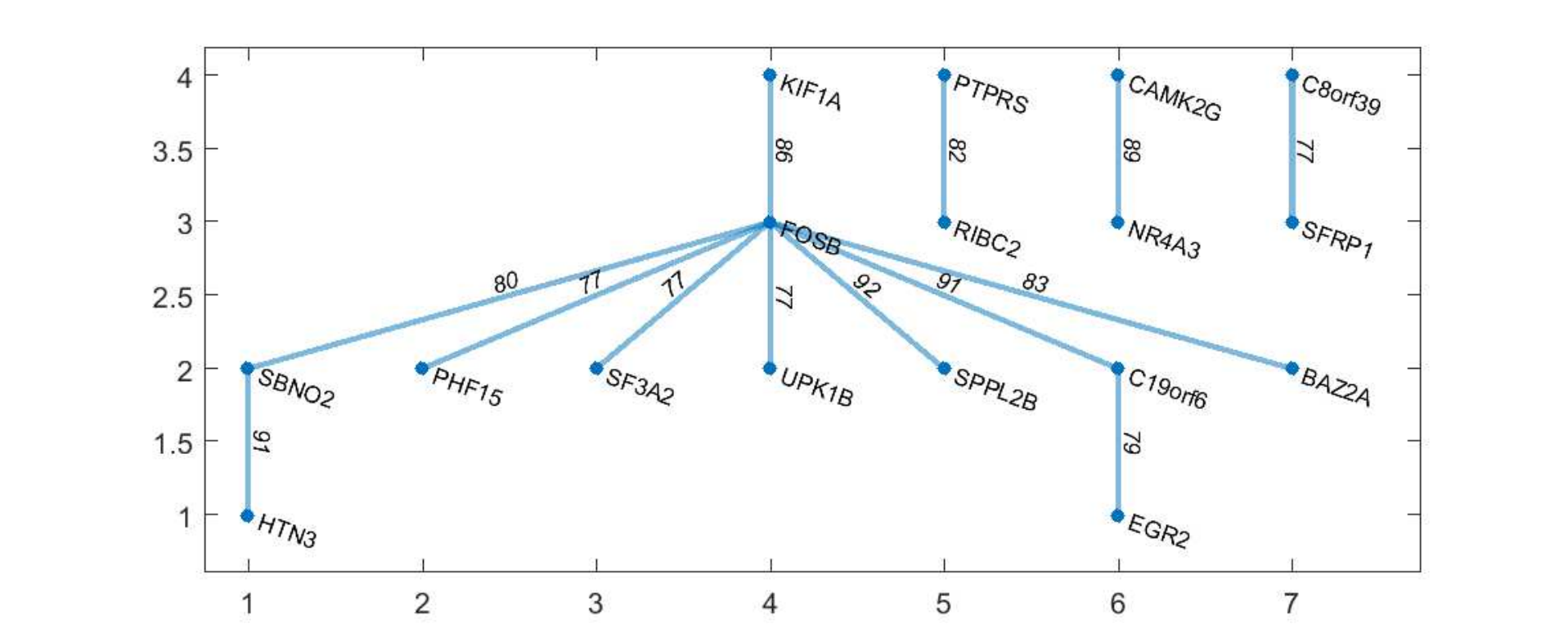}
\caption{Edges appear more than 75 times out of 100 bootstrap trials.}
\label{fig.bootstrap}
\end{figure*}
To assess the reliability of the results obtained in Section \ref{sec.gene}, we have also conducted bootstrap random trials on this dataset.
Datasets are re-sampled from both $P$ and $Q$ and the experiments are repeated for 100 times. After counting the edges appearing in the estimated graph structures we plot the graph  in Figure \ref{fig.bootstrap} whose edges appear at least 75 times out of our 100 trails.
Nodes are names of genes, and the labels on the edges are the numbers of occurrences. 

It can be seen that the obtained differential network is very close to the one we presented in Figure \ref{fig.gene1}, and the main structure (``the hub'' around the FOSB gene) remains a dominant feature in the graph, thus this does not change our conclusion that the FOSB is a regulator of other genes.

\begin{figure}
	\subfigure[One-dimensional Gaussian distributions and their density ratios. Assumptions \ref{assum.bounded.ratio} and \ref{assum.bounded.ratio.3} can be applied to the 0-bounded density ratio (green line). Our Assumptions \ref{assum.smooth.ratiomodel.nod} and \ref{assum.smooth.ratiomodel} can partially cover the totally unbounded density ratio model with some exceptions (such as the green dashed Gaussian density ratio, which grows too violently).]{
		\includegraphics[width=.75\textwidth]{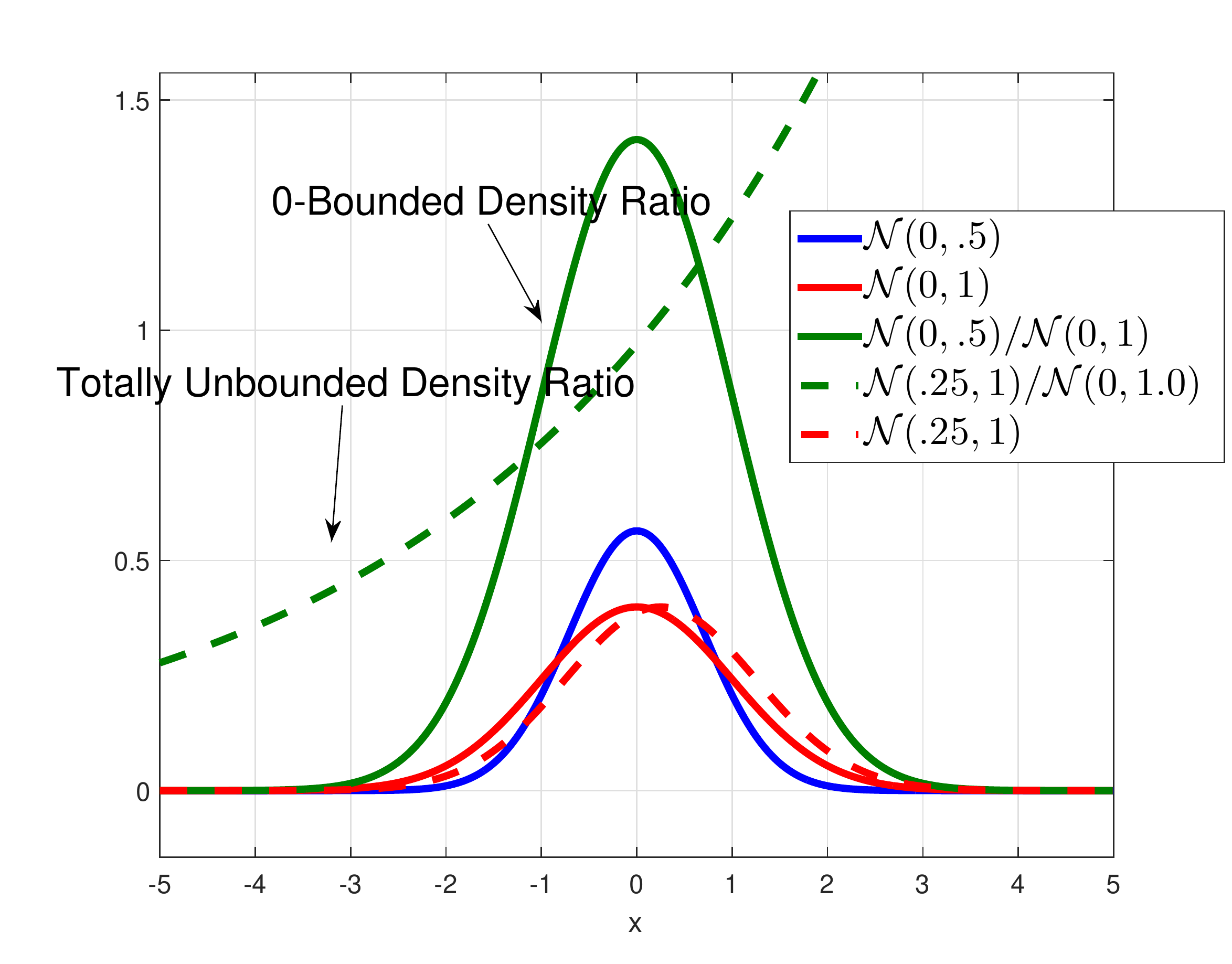}
	}
	\subfigure[Truncated one-dimensional Gaussian distribution ($t\mathcal{N}$, truncated at $-1.5,1.5$) and their density ratios. All our assumptions (Assumption \ref{assum.smooth.ratiomodel.nod}, \ref{assum.smooth.ratiomodel}, \ref{assum.bounded.ratio} and \ref{assum.bounded.ratio.3})  can be applied here.]{
		\includegraphics[width=.75\textwidth]{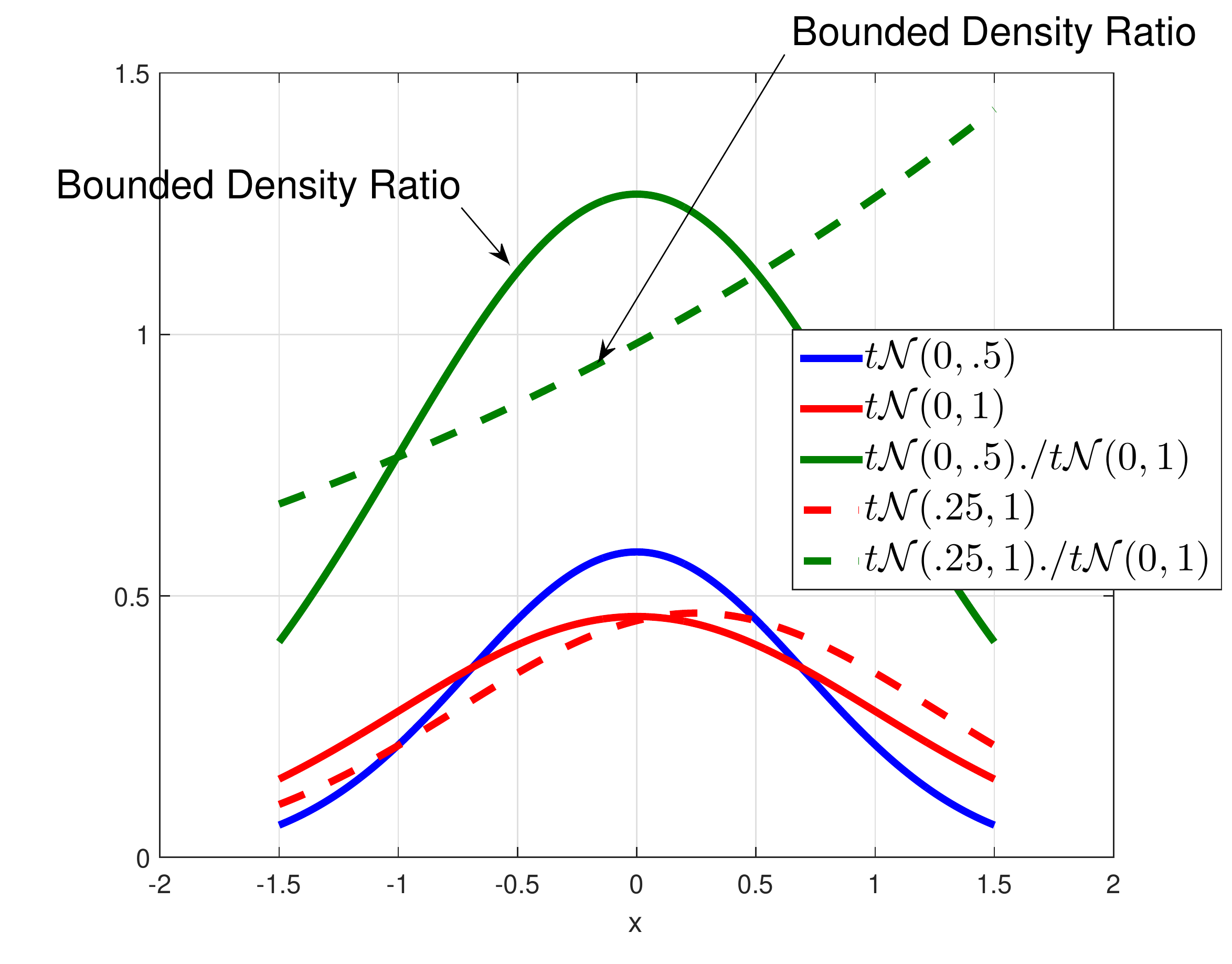}
	}
	\caption{The comparison between different boundedness assumption}
	\label{fig.bounded.comparison}
\end{figure}
\begin{figure}
	\includegraphics[width=.99\textwidth]{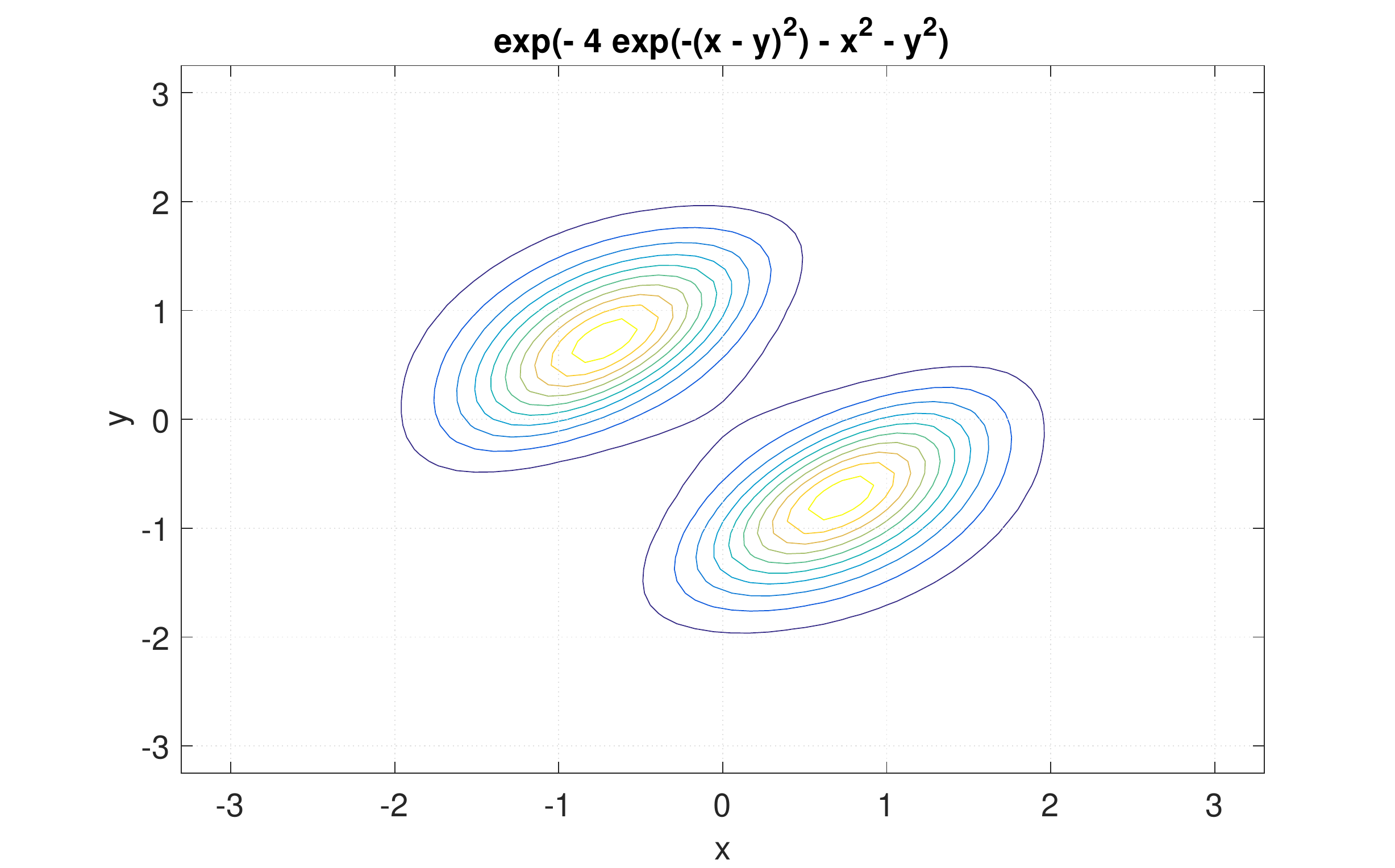}
	\caption{Contour plot for 2D ``8-shaped Distribution'' with correlation between dimension $x$ and $y$. The distribution is truncated over a ball centered at origin with radius of 15.}
	\label{fig.contour}
\end{figure}
\begin{figure}
	\includegraphics[width=.7\textwidth]{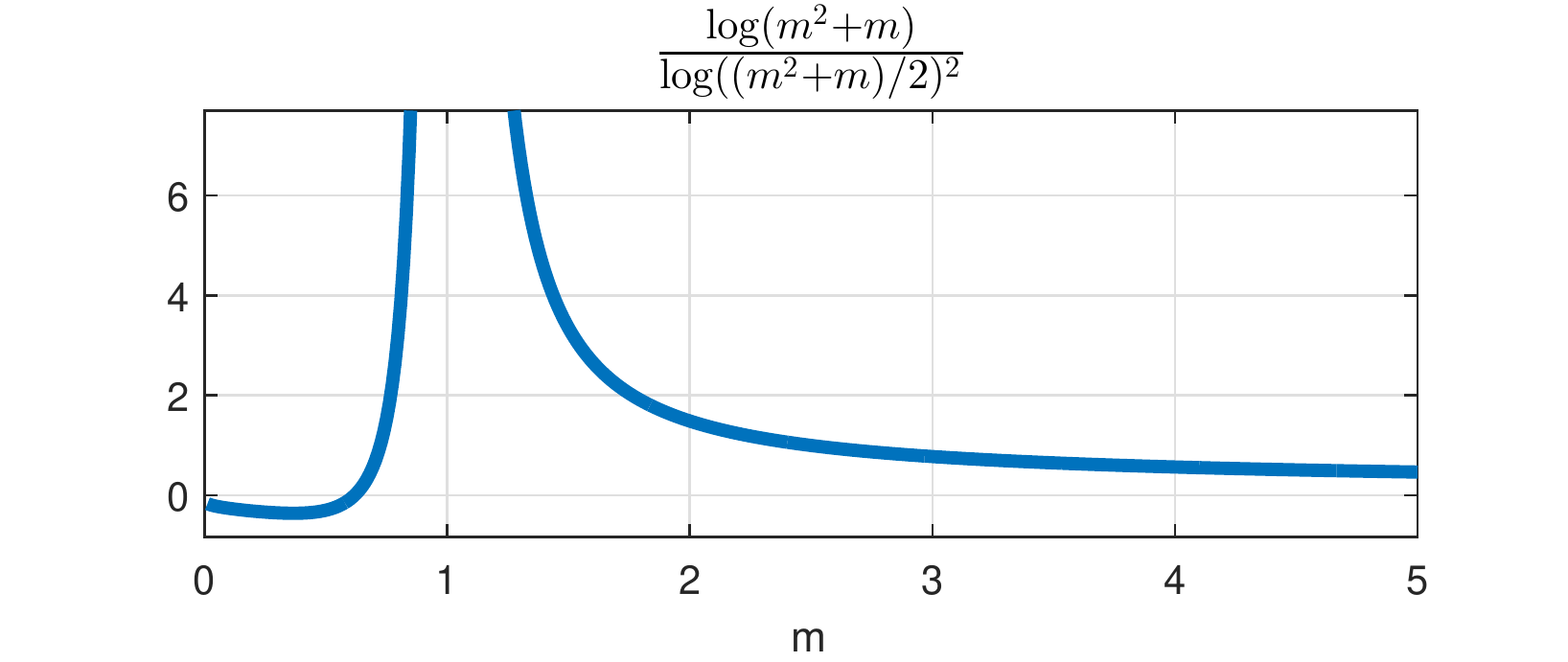}
	\caption{The plot of $g(m)$.}
	\label{fig.gm}
\end{figure}

\section{Illustrations}
\subsection{Applicability of Smoothness Assumptions}
	\label{fig.applicatility.smooth.ass}
	Figure \ref{fig.bounded.comparison} illustrates the applicability of our assumptions using one-dimensional Gaussian distributions.
\subsection{Contours of non-Gaussian Distributions}
	Figure \ref{fig.contour} shows the contours of the 2-D non-Gaussian distributions used in Section \ref{sec.exp}.
\subsection{Plot of $g(m)$}
	See Figure \ref{fig.gm} for the plot of $g(m)$ which is used in Theorem \ref{them.the.main.theorem}.

\end{changemargin}
\end{document}